\RequirePackage{fix-cm}
\documentclass[smallextended]{svjour3}       
\smartqed  
\usepackage{graphicx}
\usepackage{amsfonts, amsmath, amssymb, mathrsfs}
\usepackage{times, helvet, courier}
\usepackage{graphicx, epsfig, epstopdf}
\usepackage{wrapfig, caption, multirow, url}
\usepackage[ruled,linesnumbered,vlined]{algorithm2e}
\usepackage{stmaryrd}
\usepackage{picinpar}
\usepackage{color}
\usepackage[colorlinks=true, citecolor=cyan]{hyperref}

\SetProcNameSty{textsc}
\SetProcArgSty{textsc}

\newcommand{\prob}[6]{\multirow{2}{*}{{\footnotesize{#1}}} & \multirow{2}{*}{#2}  & \multirow{2}{*}{#3}  & \multirow{2}{*}{#4}  & \multirow{2}{*}{#5} & \multirow{2}{*}{#6}}

\newcommand{\pa}{\textsl{pa}}

\DontPrintSemicolon

\newcommand{\setf}[1]{{\bf{#1}}}
\newcommand{\varf}[1]{{\it{#1}}}

\newcommand{\bemph}[1]{\textbf{\emph{#1}}}

\newcommand{\fproj}[2]{f_{#1}{_{|#2}}}

\newcommand{\paral}[1]{\stackrel{#1}{\leftleftarrows} }

\newcommand{\out}{\textsl{out}}
\newcommand{\inp}{\textsl{in}}

\newcommand{\UTIL}{\mbox{\it UTIL}}
\newcommand{\VALUE}{\mbox{\it VALUE}}

\newcommand{\tuple}[1]{\langle #1 \rangle}

\newcommand{\size}[1]{|#1|}

\def\is{\leftarrow}
\def\eqto{\!=\!}
\def\st{\: | \:}

\newcommand{\argmax}{\operatornamewithlimits{argmax}}
\newcommand{\argmin}{\operatornamewithlimits{argmin}}
\DeclareMathOperator*{\bigtimes}{\vartimes}

\newcommand{\scope}[1]{\setf{x}^{#1}}

\newboolean{includeMemo}
\newcommand{\memo}[1]{
  \ifthenelse {\boolean{includeMemo}}{
  \medskip\noindent\fbox{\begin{minipage}[b]{\dimexpr\linewidth-1em}#1\end{minipage}}\medskip\newline}
}

\newcommand{\bitemize}{\begin{list}{$\bullet$}{\topsep=1pt \parsep=0pt \itemsep=1pt \leftmargin=1em }} 
\newcommand{\eitemize}{\end{list}}
\newcommand{\blist}{\begin{list}{$\bullet$}{\topsep=0pt \parsep=0pt \itemsep=1pt \leftmargin=1em \labelwidth=0.5em \labelsep=0.5em}} 
\newcommand{\elist}{\end{list}}
\newcommand{\benumerate}{\hspace{-0.5in} \begin{enumerate}\topsep=0pt \parsep=0pt \itemsep=-3pt } 
\newcommand{\eenumerate}{\end{enumerate}}
\newcommand{\beitemize}{\begin{list}{$\bullet$}{\topsep=1.5pt \parsep=0pt \itemsep=1pt \leftmargin=1.5em }} 
\newcommand{\enitemize}{\end{list}}

\newboolean{bshrink}
\newcommand{\shrink}{\ifthenelse {\boolean{bshrink}}{\vspace{-0.5em}}}

\newcommand{\rev}[1]{#1}

\setboolean{includeMemo}{true} 
\setboolean{bshrink}{false} 

\allowdisplaybreaks
\begin{document}

\title{Accelerating Exact and Approximate Inference for (Distributed) Discrete Optimization with GPUs\footnote{This journal article is an extended version of an earlier conference paper~\cite{fioretto:15f}. It includes (\emph{i})~a parallelized design and implementation of Mini-Bucket Elimination with GPUs on WCSPs; (\emph{ii})~a more detailed description of the GPU operations to ease reproducibility; (\emph{iii})~a significantly more comprehensive empirical evaluation with additional WCSP benchmarks and different GPU devices.}
}

\titlerunning{Accelerating Exact and Approximate Inference Algorithms with GPUs}        

\author{Ferdinando Fioretto \and Enrico Pontelli \and\\ William Yeoh \and Rina Dechter
}

\authorrunning{F. Fioretto \emph{et al.}} 


\institute{Ferdinando Fioretto\at 
           Industrial and Operations Engineering, University of Michigan \\
           Ann Arbor, MI, USA\\
           \email{fioretto@umich.edu}
\and 
	Enrico Pontelli, William Yeoh \at
           Computer Science, New Mexico State University \\
           Las Cruces, NM, USA\\
           \email{\{epontell,wyeoh\}@cs.nmsu.edu}
\and
           Rina Dechter \at
           School of Information and Computer Science, University of California, Irvine \\
           Irvine, CA, USA\\
           \email{dechter@ics.uci.edu}      
}

\date{Received: date / Accepted: date}

\maketitle

\begin{abstract}
\emph{Discrete optimization} is a central problem in artificial intelligence.  
The optimization of the aggregated cost of a network of cost functions arises in a variety of problems including \emph{Weighted Constraint Programs} (WCSPs), \emph{Distributed Constraint Optimization} (DCOP),  \rev{as well as optimization in stochastic variants such as the tasks of finding the \emph{most probable explanation} (MPE) 
in \emph{belief networks}.}
Inference-based algorithms are powerful techniques for solving discrete optimization problems, which can be used independently or in combination with other techniques. However, their applicability is often limited by their compute intensive nature and their space requirements. 
 
This paper proposes the design and implementation of a novel inference-based technique, which exploits modern massively parallel architectures, such as those found in  Graphical Processing Units (GPUs), to speed up the resolution of exact and approximated inference-based algorithms for discrete optimization.
The paper studies the proposed algorithm in both centralized and distributed optimization contexts. 

The paper demonstrates that the use of GPUs provides significant advantages in terms of runtime and scalability, achieving  up to two orders of magnitude in speedups and showing a considerable reduction in execution
time (up to 345 times faster) with respect to a sequential version.
\end{abstract}

\section{Introduction}
\label{sec:intro}

The importance of constraint optimization is outlined by the impact of its application in a 
wide range of 
domains, such as supply-chain management (e.g., \cite{supply1,gaudreault:09}), roster scheduling (e.g., \cite{nurse1,burke:04}), combinatorial auctions (e.g., \cite{sandholm:02}), bioinformatics (e.g., \cite{allouche:14,campeottoPDFP13,fiorettoDP15}),
multi-agent systems (e.g., \cite{dovier1}) and probabilistic reasoning (e.g, \cite{Pearl:88}).

In \emph{Constraint Satisfaction Problems (CSPs),} the goal is to find  a value assignment for a set of variables that satisfies a set of constraints~\cite{apt:03,handbook}. The assignments satisfying the problem constraints are called \emph{solutions}. In \emph{Weighted Constraint Satisfaction Problems (WCSPs)} the goal is that of finding an optimal solution, given a set of preferences expressed by means of cost functions \cite{shapiro:81,schiex:95,Bistarelli:97}. 

\rev{
When the problems involve uncertainty, we recur to the notion of \emph{belief} or \emph{Bayesian} networks (BNs) \cite{Pearl:88}. This framework aims at modeling natural phenomena and exogenous uncertainty through probabilistic reasoning. In BNs the goal is that of answering queries given partial beliefs under conditions uncertainty. 
Common tasks over BNs include finding the \emph{most probable explanation} (MPE), 
also known as \emph{maximum a posteriori hypothesis} (MAP), that is finding the assignment with largest probability to all unobserved variables given some observed variables 
 \cite{Dechter:03}. Belief networks are widely applied to a variety of applications, such as, diagnosis~\cite{lerner2000bayesian} and linkage analysis~\cite{fishelson:02,friedman:00}.
}

When resources are distributed among a set of autonomous agents and communication among the agents is restricted, WCSPs take the form of \emph{Distributed Constraint Optimization Problems (DCOPs)}~\cite{modi:05,petcu:05,yeoh:12}. 
In this context, agents coordinate their value assignments to minimize the overall sum of resulting constraint costs.
DCOPs have been employed to model various distributed optimization problems, such as meeting scheduling~\cite{maheswaran:04a,yeoh:10}, resource allocation~\cite{farinelli:08,zivan:15}, power network management problems~\cite{kumar:09,naps13,fioretto:AAMAS-17b}, and coordination of appliances in smart homes~\cite{fioretto:AAMAS-17a,rust:16}. We will refer to WCSPs and DCOPs as \emph{discrete optimization problems}.

Algorithms to solve discrete optimization problems can be classified as \emph{exact} and \emph{approximated}. 
Exact algorithms are guaranteed to find optimal solutions. However, since solving WCSPs and DCOPs is NP-hard \cite{handbook}, optimally solving these  problems results in prohibitive runtime and/or use of resources, such as memory or network load. 
In contrast, approximated algorithms trade  solution optimality for shorter runtime and a more efficient use of the available resources.

Furthermore, discrete optimization algorithms can adopt two main paradigms: \emph{search} or \emph{inference}. 
\emph{Search-based methods} rely on the use of non-deterministic branching rules to explore different value assignments to variables. These rules are applied recursively until all problem variables are assigned. 
This process defines a search tree (typically traversed in a depth-first fashion), which has the advantage of requiring only polynomial space. However, the practical efficiency of these methods relies on their ability to prune redundant or
sub-optimal subtrees. 
\emph{Inference-based methods} are inspired from dynamic programming (DP) techniques. These methods apply a sequence of transformations  
to reduce the problem size at each step while preserving its semantics. A well known inference-based approach is \emph{Bucket Elimination} (BE) \cite{Dechter:99}. BE 
iterates over  the variables of the problem, reducing the size of the problem at each step, by replacing a variable and its related cost functions with a single new function, derived by optimizing over the possible values of the eliminated variable. The \emph{Dynamic Programming Optimization Protocol (DPOP)}~\cite{petcu:05} is one of the most efficient inference-based DCOP solvers, and it can be seen as a distributed version of BE, where agents exchange  newly introduced cost functions via messages. 

The importance of inference-based approaches arises in several optimization fields including constraint programming~\cite{apt:03,handbook}. For example, several \emph{propagators} adopt DP-based techniques to establish constraint consistency. For instance, {\bf (1)}~the \emph{knapsack} constraint propagator proposed by Trick applies DP techniques to establish arc consistency on the constraint~\cite{trick:03}; {\bf (2)}~the propagator for the \emph{regular} constraint establishes arc consistency using a specific digraph representation of the DFA, which has similarities to dynamic programming~\cite{pesant:04}; {\bf (3)}~the \emph{context free grammar} constraint makes use of a propagator based on the CYK parser that uses DP to enforce generalized arc consistency~\cite{quimper:06}.

The main drawback of inference-based methods, including BE and DPOP, is that each transformation may introduce cost functions with large arities, requiring exponential time and space in a key structural parameter of a problem, called \emph{induced width}. 
While inference-based approaches may not always be appropriate to solve discrete optimization problems, as their time and space requirements may be prohibitive, they may be very effective in problems with particular structures, such as problems where their underlying \rev{primal graphs} have small induced widths or distributed problems where the number of messages is crucial for performance, despite the size of the messages.
Additionally, approximated inference methods can be effectively used to derive lower bounds, which are important components of branch and bound algorithms, as they can be used to prune parts of the search space by detecting \emph{dominated} solutions---i.e., solutions whose cost can provably not be lower than the best cost found so far. \emph{Mini-Bucket Elimination (MBE)} is an approximated variant of BE that can be used for this purpose.

Recent developments on external-memory algorithms have shown that the use of large secondary data storage can be effective to extend the applicability of memory intensive approaches \cite{Edelkamp2004,lim2010scaling,kask2012beem,sturtevant:13}. However, the computational solving runtime remains a bottleneck.  

To contrast this background, we note that the structure exploited by inference-based approaches in constructing solutions makes it suitable to exploit a novel class of massively parallel platforms that are based on the \emph{Single Instruction Multiple Thread} (SIMT) paradigm---where multiple threads may concurrently operate on different data, but are all executing the same instruction at the same time.
The SIMT-based paradigm is widely used in modern \emph{Graphical Processing Units (GPUs)} for general purpose parallel computing. Several libraries and programming environments (e.g., the \emph{Compute Unified Device Architecture} (CUDA)) have been made available to allow programmers to 
exploit the parallel computing power of GPUs.

In this paper, we propose the design and implementation of both 
an exact and an approximated inference-based algorithm that exploits parallel computation using GPUs to solve WCSPs and DCOPs. 
Our proposal aims at employing GPU hardware to speed up the inference process, thus providing 
 an alternative way to enhance the performance of inference-based discrete optimization approaches. 

This paper makes the following contributions: 
{\bf (1)}~We propose a novel design and implementation of a centralized and a distributed exact inference-based algorithm, inspired by BE and DPOP, to optimally solve WCSPs and DCOPs, which harnesses the computational power offered by parallel platforms based on GPUs; 
{\bf (2)}~We introduce an approximated version of the GPU-based inference-based algorithm, inspired by MBE;
{\bf (3)}~We report an extensive empirical analysis that shows significant improvements in performance with respect to the 
sequential CPU-based algorithms, reporting an average speedup of two order of magnitude; and
{\bf (4)}~We show the generality of our approach through empirical evaluations on three different GPU architectures, all providing significant speedups.
\rev{While the description of the techniques proposed in this paper focus on discrete optimization tasks, they also applies to other key problems in probabilistic graphical models, such as, MPE.}

\section{Background: Weighted Constraint Satisfaction Problems}
\label{sec:cop}\label{sec:be}

A \emph{weighted constraint satisfaction problem} (WCSP)~\cite{larrosa:02,shapiro:81} is a tuple $\langle \setf{X}, \setf{D}, \setf{C} \rangle$, where 
$\setf{X} = \{x_1, \ldots, x_{n}\}$ is a finite set of variables, 
$\setf{D} = \{D_{x_1}, \ldots, D_{x_n}\}$ is a set of finite domains for the variables in $\setf{X}$, with $D_{x_i}$ being the set of possible values for the variable $x_i$, and
$\setf{C}$ is a set of \emph{weighted constraints} (or \emph{cost functions}).
A weighted constraint $f_i \in \setf{C}$ is a function that maps tuples defined on the set of variables relevant to $f_i$ into \rev{$\mathbb{R}_+ \cup \{\infty\}$}, where $\infty$ is a special value denoting that a given combination of values is not allowed. 
The set of variables relevant to $f_i$ is referred to as the \emph{scope} of $f_i$, and denoted as $\scope{i} \subseteq \setf{X}$. 
Formally, \rev{$f_i : \bigtimes_{x_j \in \scope{i}} D_{x_j} \to \mathbb{R}_+ \cup \{\infty\}$}.\footnote{For simplicity, we assume that tuples of variables are built according to a predefined ordering.}
%
%
A \emph{solution} is a value assignment for a subset $\rho$ of variables from $\setf{X}$ that is consistent with their respective domains; i.e., it is a partial function $\theta: \setf{X} \rightarrow \bigcup_{i=1}^n D_{x_i}$ such that, for each $x_j \in \setf{X}$, if $\theta(x_j)$ is defined (i.e., $x_j\in\rho$), then $\theta(x_j) \in D_{x_j}$. 
The \bemph{cost} of an assignment $\rho$ is the sum of the evaluation of the constraints involving all the variables in $\rho$. 
A solution is \emph{complete} if it assigns a value to each variable in $\setf{X}$ and has finite cost (i.e., different
from $\infty$).
We will use the notation $\sigma$ to denote a complete solution, and, for a set of variables $\setf{V} = \{x_{i_1}, \ldots, x_{i_h}\} \subseteq \setf{X}$,  $\sigma_\setf{V} = \tuple{\sigma(x_{i_1}), \ldots, \sigma(x_{i_h}) }$  is the projection of
 $\sigma$ to the  variables in $\setf{V}$,  where $i_1 < \cdots < i_h$. 
The goal of a WCSP is to find a complete solution $\sigma^*$ with minimal cost, i.e.,
\begin{equation}	
\sigma^* \eqto \argmin_{\sigma \in \Sigma} 
		         \sum_{f_i \in \setf{C}} f_i( \sigma_{\scope{i}} ), 
\end{equation}
where $\Sigma$ is the \emph{state space}, defined as the set of all possible complete solutions. 

\rev{
Given a WCSP $P$,  $G_P \eqto (\setf{X},E_{\setf{C}})$ is the \bemph{primal graph} of $P$, where $\{x,y\} \in E_{\setf{C}}$ iff $\exists f_i \in \setf{C}$ such that $\{x,y\} \subseteq \scope{i}$. 
}
Given an ordering $o$ on $\setf{X}$, we say that a variable $x_i$ has lower \bemph{priority} w.r.t.~a variable $x_j$, denoted $x_i \prec_o x_j$,  if $x_i$ precedes $x_j$ in $o$. 

\begin{definition}[Induced Graph, Induced Width \cite{Dechter:03}] 
Given the \rev{primal} graph $G_P$ and an ordering $o$ on its nodes, the \emph{induced graph} $G_P^*$ on $o$ is the graph obtained from $G_P$ by connecting nodes, processed in descending order of priority, to all their preceding neighbors. 
\rev{
Processing a node $x_i$ results in the addition of 
edges connecting pairs of preceding neighbors of $x_i$. 
%
}
Given a graph and an ordering of its nodes, the \emph{width} of a node is the number of edges connecting it to its preceding nodes in the ordering. The induced width $w_o^*$ of $G_P$ is maximum width over all nodes of $G_P^*$ along the ordering $o$.
\end{definition}

\begin{figure}[!tbh]
  \centering
  \includegraphics[width=0.95\textwidth]{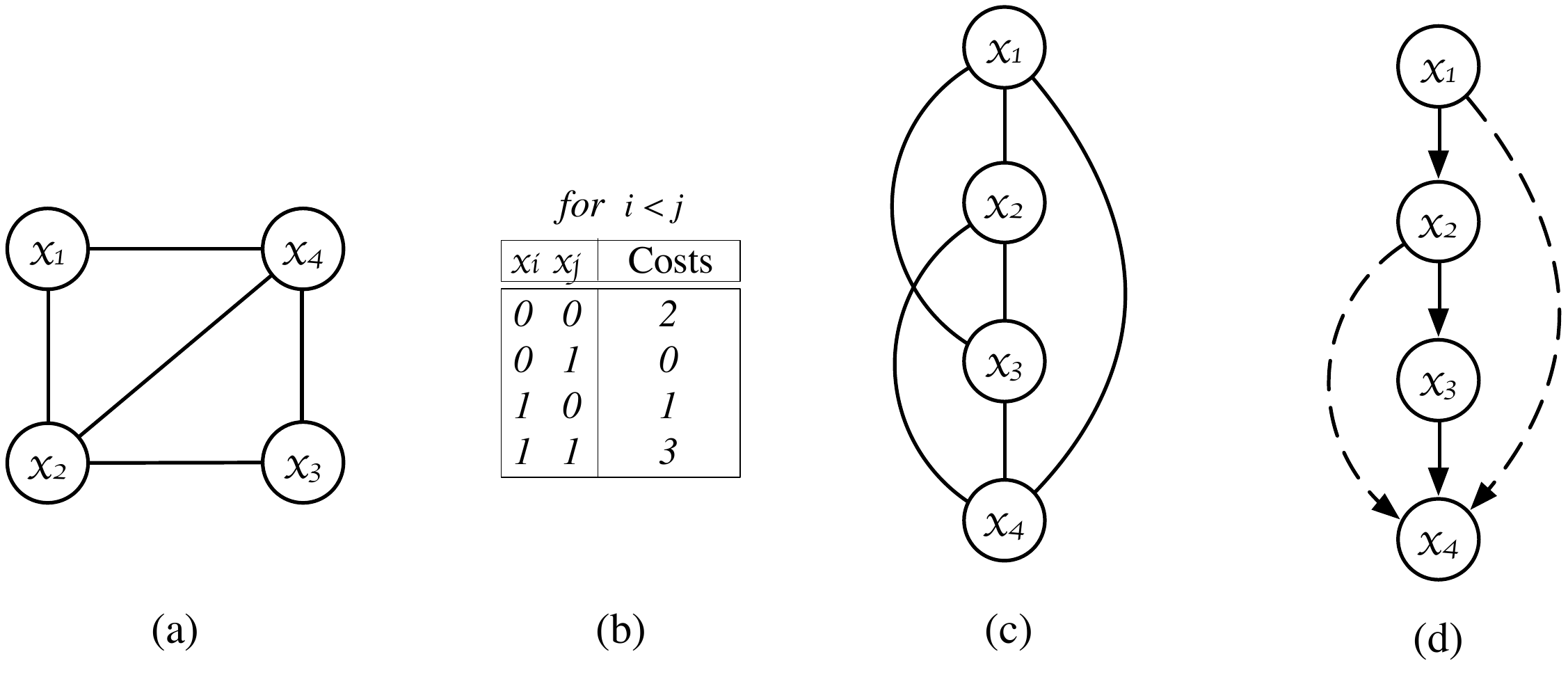}
  \caption{Example of a WCSP :  
  \textbf{(a):} Primal graph. 
  \textbf{(b):} Cost functions.
  \textbf{(c):} A possible induced graph. 
  \textbf{(d):} A possible pseudo-tree. 
  \label{fig:wcsp}}  
\end{figure}

\begin{example}
\label{ex:wcsp}
Fig.~\ref{fig:wcsp}(a) illustrates the \rev{primal graph} of a simple WCSP instance with 4 binary  variables, $x_1$, $x_2$, $ x_3$, and $x_4$, and $5$ 
\rev{constraints, $f(x_1, x_2)$, $f(x_1, x_4)$, $f(x_2, x_3)$, $f(x_2, x_4)$, $f(x_3, x_4)$.}
Fig.~\ref{fig:wcsp}(b) illustrates the constraints costs of the WCSP, which associate a cost value for each combination of values for the variables in the scope of the constraints. 
Fig.~\ref{fig:wcsp}(c) shows the induced graph $G_P^*$ obtained along the ordering $o = \langle x_1, x_2, x_3, x_4 \rangle$.
Its induced width is $3$.
\end{example}

\begin{definition}[Pseudo-tree \cite{Dechter:13}]
Given a \rev{primal graph} $G_P$ \rev{and an ordering $o$ on its nodes}, a \emph{DFS} \emph{pseudo-tree} arrangement for $G_P$ is a \rev{rooted directed} tree $T \eqto \langle \setf{X}, E_T\rangle$ of $G_P$ such that if $f_i \!\in\! \setf{C}$ and $\{x,y\} \subseteq \scope{i}$, then $x$ and $y$ appear in the same branch of $T$. 
\rev{The root of $T$ is the node associated to the variable with lower priority in $o$.} 
Edges of $G_P$ that are \emph{in} (resp. \emph{out} of) $E_T$ are called \emph{tree edges} (resp. \emph{backedges}). The tree edges connect parent-child nodes, while backedges connect a node with its \emph{pseudo-parents} and its \emph{pseudo-children}. 
\end{definition}

\begin{example}
\label{ex:wcsp-pt}
Fig.~\ref{fig:wcsp}(d) shows one possible pseudo-tree $T \eqto \langle \setf{X}, E_T\rangle$ associated to the \rev{primal graph} shown in Fig.~\ref{fig:wcsp}(a), with 
\rev{$E_T = \{f(x_1, x_2), f(x_2, x_3), f(x_3, x_4)\}$, and order $o = \langle x_1, x_2, x_3, x_4 \rangle$.}
\rev{The node labeled $x_1$ is the root node; it has a pseudo-child, node $x_4$. The node labeled $x_4$ has two pseudo-parents nodes: $x_2$ and $x_1$.}
The solid lines describe tree edges, while the dotted lines represent backedges.
\end{example}

\begin{definition}[Projection]
The \emph{projection} of a cost function $f_i$ on a set of variables $\setf{V} \subseteq \scope{i}$ is a new cost function 
\rev{$ \fproj{i}{\setf{V}} : \setf{V} \to \mathbb{R}_+ \cup \{\infty\}$}, such that for each possible assignment 
$\theta \in \bigtimes_{x_j \in \setf{V}} D_{x_j}$, 
$\fproj{i}{\setf{V}} (\theta) = \displaystyle \min_{\sigma \in \Sigma, \sigma_\setf{V}=\theta} f_i(\sigma_{\scope{i}})$.\footnote{For simplicity, we also use $\theta$ to represent the tuple
	$\langle \theta(x_{i_1}),\dots, \theta(x_{i_h})\rangle$ where $\{x_{i_1},\dots, x_{i_h}\}$ is the domain of
	$\theta$.}
\end{definition}

In other words, $\fproj{i}{\setf{V}}$ is constructed from the tuples of $f_i$, removing the values of the variable that do not appear in $\setf{V}$ and removing duplicate values by keeping the minimum cost of the original tuples in $f_i$. 
\begin{definition}[Concatenation]
Let us consider two assignments $\theta'$, defined for variables $V$, and $\theta''$, defined for
variables $W$, such that for each $x\in V\cap W$ we have that $\theta'(x) = \theta''(x)$.
Their \emph{concatenation} is an assignment $\theta' \cdot \theta''$ defined for 
$V\cup W$, such as for each $x\in V$ (resp. $x\in W$) we have that
$\theta'\cdot\theta''(x) = \theta'(x)$ (resp. $\theta'\cdot\theta''(x) = \theta''(x)$).
\end{definition}
We define two operations on cost functions:
\bitemize
	\item The \bemph{aggregation} of two functions $f_i$ and $f_j$, is a function \rev{$f_i + f_j : \scope{i} \cup \scope{j} \to \mathbb{R}_+ \cup \{\infty\}$}, such that $\forall \theta' \in \bigtimes_{x_k \in \scope{i}} D_{x_k}$ and $\forall \theta'' \in \bigtimes_{x_k \in \scope{j}} D_{x_k}$, if $\theta'\cdot\theta''$ is defined, then we have that 
	$$(f_i + f_j) (\theta' \cdot \theta'') \eqto f_i(\theta') + f_j(\theta'').$$

	\item The \bemph{elimination} of a variable $x_j \in \scope{i}$ from a function $f_i$, denoted as $\pi_{-x_j}(f_i)$, produces a new function with scope $\scope{i}\setminus\{x_j\}$, and defined as the projection of $f_i$ on $\scope{i}\setminus\{x_j\}$, i.e., $$\pi_{-x_j}(f_i) \eqto \fproj{i}{\scope{i}\smallsetminus\{x_j\}}.$$
\eitemize

\rev{While the aggregation and elimination operators are defined on summation and minimization, respectively, for discrete optimization problems, several tasks in belief networks can be solved by using variants of the aggregation and elimination operators~\cite{Dechter:03}.}


\begin{algorithm}[!htbp]
	\tcc{Variable Elimination Phase}
	\For{$i\leftarrow n$ \textnormal{\textbf{downto}} $1$}{
		$B_i \is \{ f_j \in \setf{C} \st x_i \in \scope{j} \land i = \min \{k \st x_k \in \scope{j}\} \}$\;
		$\hat{f}_i \is \pi_{-x_i}\Big( \sum_{f_j \in B_i} f_j \Big)$\;
		$\setf{X} \is \setf{X} \setminus \{x_i\}$\;
		$\setf{C} \is (\setf{C} \setminus B_i) \cup \{\hat{f}_i\}$\;
	}
	\tcc{Value Assignment Phase}
	\For{$i\leftarrow 1$ \textnormal{\textbf{to}} $n$}{
		$x_i \is d_i$ s.t.~$d_i \in D_{x_i}$ and $d_i$ is the best extension of $x_1, \ldots, x_{i-1}$ w.r.t.~$B_i$\;
		}
	\Return{$\hat{f}_1$}\;
	\caption{\textsc{Bucket Elimination}}
	\label{alg:be}
\end{algorithm}

\subsection{Bucket Elimination}
\label{sec:be}

\emph{Bucket Elimination} (BE) \cite{Dechter:99,Dechter:03} is a complete inference algorithm 
that can be used to find the optimal solutions of a WCSP. Algorithm~\ref{alg:be} illustrates its pseudocode. 
BE operates in the following two phases: 
\bitemize
\item \bemph{Variable Elimination Phase}. BE operates from the highest to lowest priority variable. When operating on variable $x_i$, it creates a bucket $B_i$, which is the set of all cost functions that involve $x_i$ as the highest priority variable in their scope (line~2). The algorithm then computes a new cost function $\hat{f}_i$ by aggregating the functions in $B_i$ and eliminating $x_i$ (line~3). Thus, $x_i$ can be removed from the set of variables $\setf{X}$ to be processed (line~4) and the new function $\hat{f}_i$ replaces in $\setf{C}$ all the cost functions that appear in  $B_i$ (line~5). 
\rev{Thus, a bucket $B_i$ contains both the original WCSP functions as well as the functions placed in it during the variable elimination process.} 
In this work, we refer to the $\hat{f}_i$ functions as the \bemph{bucket functions}.

\item \bemph{Value Assignment Phase}. Once the variable with the lowest priority has been processed, the algorithm considers variables in increasing order of priority. For each variable $x_i$, it generates an optimal assignment by selecting a value $d_i \in D_{x_i}$ that minimizes the cost of the functions in $B_i$ given the assignments of all the other variables appearing in the scope of the functions in $B_i$. 
\eitemize
As a byproduct, and without additional overhead, BE can compute the number of \rev{consistent} solutions of the problem (see \cite{Dechter:03}, for details).
The time and space complexity of BE is exponential on the induced width of the underlying \rev{primal graph}, which captures the maximum arity of the $\hat{f}_i$ functions (line~3).

\begin{example}
In our WCSP example of Fig.~\ref{fig:wcsp}, during the \emph{Variable Elimination Phase}, BE operates, in order, on the variables $x_4$, $x_3$, $x_2$, and $x_1$. 
When $x_4$ is processed, \rev{the bucket $B_4 \eqto \{f(x_1, x_4), f(x_2, x_4), f(x_3, x_4)\}$} is generated, and highlighted in Fig.~\ref{fig:be}(a)(top) by red edges. The resulting bucket function $\hat{f}_4$ is shown in Fig.~\ref{fig:be}(a)(bottom), where the rightmost column shows the values for $x_4$ after its elimination.
BE, hence, updates the sets $\setf{X} \eqto \{x_1,x_2,x_3\}$ and \rev{$\setf{C} \eqto \{f(x_1, x_2), f(x_2,x_3), \hat{f}_4\}$}, as shown in the \rev{primal graph} of Fig.~\ref{fig:be}(b)(top), where the function $\hat{f}_4$ is displayed as a dotted line. 
When $x_3$ is processed, \rev{$B_3 \eqto \{f(x_2, x_3), \hat{f}_4\}$}, 
and $\hat{f}_3$ is shown in Fig.~\ref{fig:be}(b)(bottom). 
Thus, $\setf{X} \eqto \{x_1,x_2\}$ and \rev{$\setf{C} \eqto \{f(x_1, x_2), \hat{f}_3\}$}, as shown in Fig.~\ref{fig:be}(c)(top). 
Next, $x_2$ is processed, and \rev{$B_2 \eqto \{f(x_1, x_2), \hat{f}_3\}$}, and $\hat{f}_2$ is illustrated in Fig.~\ref{fig:be}(c)(bottom). 
Thus, $\setf{X} \eqto \{x_1\}$ and $\setf{C} \eqto \{\hat{f}_2\}$, as shown in Fig.~\ref{fig:be}(d)(top).
Lastly, the algorithm processes $x_1$, sets $B_1 \eqto \{\hat{f}_2\}$, and $\hat{f}_1$ is minimized when 
$x_1 \eqto 0$, as shown in Fig.~\ref{fig:be}(d)(bottom).
\begin{figure}[!tb]
  \centering
  \includegraphics[width=1.00\textwidth]{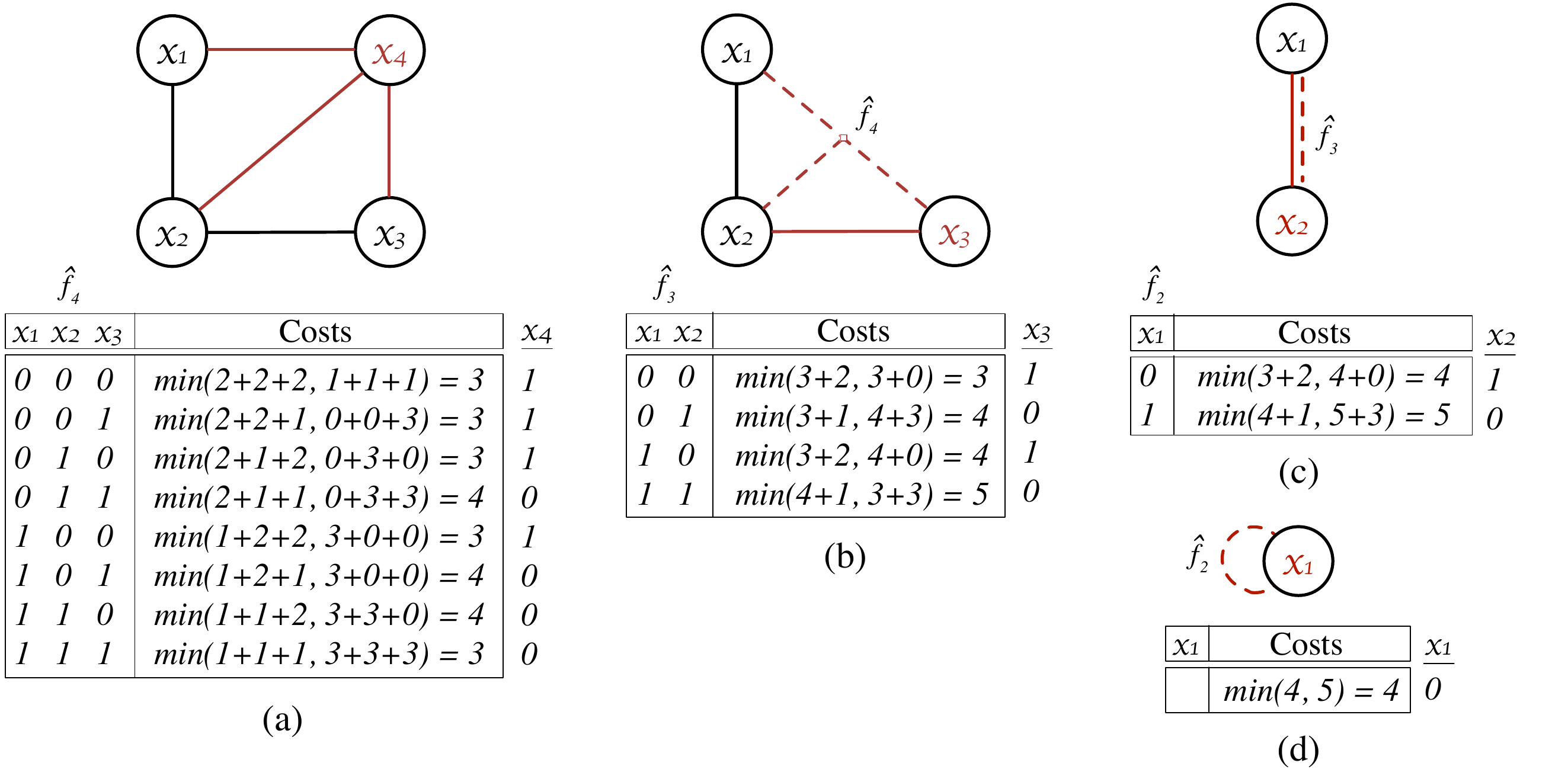}
  \caption{Bucket Elimination steps for the WCSP of Fig.~\ref{fig:wcsp} . 
  \label{fig:be}}  
\end{figure}
Next, BE starts the \emph{Value Assignment Phase}, which operates, in order, on the variables $x_1$, $x_2$, $x_3$, and $x_4$. First, it selects the value that minimizes $\hat{f}_1$, ($x_1 \eqto 0$). Thus, it processes $x_2$, and selects the value $x_2 \eqto 1$, as it minimizes $\hat{f}_2$ when $x_1 \eqto 0$, as illustrated in Fig.~\ref{fig:be}(c)(bottom).  
Similarly, when BE processes $x_3$, it selects the value $x_3 \eqto 0$, as it minimizes $\hat{f}_3$ when $x_1 \eqto 0$ and $x_2 \eqto 1$, illustrated in Fig.~\ref{fig:be}(b)(bottom). 
Finally, BE processes the last variable $x_4$ and assigns it the value $1$, since it minimizes $\hat{f}_4$ when $x_1\eqto 0, x_2\eqto 1$, and $x_3\eqto 0$, illustrated in Fig.~\ref{fig:be}(a)(bottom). Thus,  $\sigma^* \eqto \tuple{0,1,0,1}$ is an optimal solution to the problem, with cost 4.
\end{example}

\begin{algorithm}[!t]
	\tcc{Variable Elimination Phase}
	\For{$i\leftarrow n$ \textnormal{\textbf{downto}} $1$}{
		$B_i \is \{ f_j \in \setf{C} \st x_i \in \scope{j} \land i = \min \{k \st x_k \in \scope{j}\} \}$\;
		Let  $\{ B_{i_1}, \ldots, B_{i_m}\}$ be a partition of $B_i$ s.t. 
		${\left| \bigcup_{f_j \in B_{i_k}} \hspace{-0pt} \scope{j} \right| \leq z}$, for each $k = 1, \ldots, m$\;
		\ForEach{ $k \in \{1, \ldots, m\}$}	{		
			$\hat{f}_{i_k} \is \pi_{-x_i}\Big( \sum_{f_j \in B_{i_k}} f_j \Big)$\;
			$\setf{C} \is (\setf{C} \setminus B_{i_k}) \cup \{\hat{f}_{i_k}\}$\;
		}
		$\setf{X} \is \setf{X} \setminus \{x_i\}$\;
	}
	\tcc{Value Assignment Phase}
	\For{$i\leftarrow 1$ \textnormal{\textbf{to}} $n$} {
		$x_i \is d_i$ s.t.~$d_i \in D_{x_i}$ and $d_i$ is the best extension of $x_1, \ldots, x_{i-1}$ w.r.t.~$B_i$\;
	}
	\Return{$\hat{f}_1$}\;
	\caption{\textsc{Mini-Bucket Elimination}($z$)}
	\label{alg:mbe}
\end{algorithm}

\subsection{Mini-Buckets}
\label{sec:mbe}

The memory complexity and time complexity of BE depend on the arity of the functions $\hat{f}$ produced during the variable elimination step. Such requirements can quickly become infeasible for problems with large induced widths. 
To overcome this limitation, Dechter and Rish proposed an incomplete version of the Bucket Elimination~\cite{dechter:03b}. The \emph{Mini-Bucket Elimination} (MBE) is an approximated version of the BE that allows one to bound the arity of the functions $\hat{f}_i$ generated during the Variable Elimination Phase.
Its pseudocode is illustrated in Algorithm~\ref{alg:mbe}.
Similarly to BE, MBE operates in the following two phases:
\bitemize
\item \bemph{Variable Elimination Phase}. As in BE, during the variable elimination phase, 
MBE operates on the problem variables in decreasing order of priority. However, rather than creating a single bucket function $\hat{f}_i$ whose scope is the union of the scope of each function in the bucket $B_i$, it partitions $B_i$ in a set of $m$ ``mini"-buckets $\{B_{i_1}, \ldots, B_{i_m}\}$, such that the size of the scope of the bucket function $\hat{f}_{i_k}$, obtained by aggregating the functions in $B_{i_k}$, is bounded by a parameter $z$, for each $k \in \{1,\ldots,m\}$ (line~3).
Thus, MBE considers each mini-bucket independently, and computes $m$ new bucket functions $\hat{f}_{i_k}$, by aggregating the functions in $B_{i_k}$ and eliminating $x_i$ (line~5). These functions replace in \setf{C} all the functions that appear in $B_{i_k}$ (line~6), and the set of variables is updated as in BE (line~7). 
\item \bemph{Value Assignment Phase}. This phase 
is analogous to that of BE (lines~8--9).
\eitemize
Consider the elimination step for a variable $x_i \in \setf{X}$. Since:
$$
	\sum_{k=1}^m \left[ \pi_{-x_i}\Big( \sum_{f_j \in B_{i_k}} f_j \Big) \right]
	\leq
	\pi_{-x_i}\Big( \sum_{f_j \in B_{i}} f_j \Big) 
$$
eliminating $x_i$ using mini-buckets produces a lower bound on the optimal cost for the bucket $B_i$. 
Thus, MBE produces a lower bound on the optimal solution cost. Running the Value Assignment Phase might hence return a sub-optimal solution, whose evaluation will be an upper bound on the optimal solution cost.

The time and space complexity of MBE is exponential on the maximal arity of the aggregated functions in the mini-buckets (line~13), and thus it is bounded by the parameter $z$.

\begin{example}
Consider the WCSP of Fig.~\ref{fig:wcsp} solved with MBE using $z=1$.
As in BE, during the \emph{Variable Elimination Phase} MBE operates, in order, on the variables $x_4, x_3, x_2$, and $x_1$. 
When $x_4$ is processed, the bucket \rev{$B_4 \eqto \{f(x_1, x_4), f(x_2, x_4), f(x_3, x_4)\}$}---illustrated by the red edges in Fig.~\ref{fig:be}(a) top---would result in aggregated bucket function whose arity is $3$, and thus exceeds the maximal arity allowed. 
Thus, MBE creates a partition $\{ B_{4_1}, B_{4_2}, B_{4_3} \}$ for $B_4$, whose sets consists of the functions, respectively, \rev{$f(x_1, x_4), f(x_2, x_4)$, and $f(x_3, x_4)$}. 
The resulting functions $\hat{f}_{4_1}$, $\hat{f}_{4_2}$, and $\hat{f}_{4_3}$ have arity $1$, as illustrated in Fig.~\ref{fig:be}(a) bottom.
Then, MBE updates the sets $\setf{X}$ to $\{x_1,x_2,x_3\}$ and $\setf{C}$ to 
\rev{$\{f(x_1, x_2), f(x_2, x_3), \hat{f}_{4_1}, \hat{f}_{4_2}, \hat{f}_{4_3}\}$}, as shown in the \rev{primal graph} of Fig.~\ref{fig:be}(b) top. 
When $x_3$ is processed, \rev{$B_3 \eqto \{f(x_2, x_3), \hat{f}_{4_3}\}$}, marked red in Fig.~\ref{fig:be}(b) top, and the mini-bucket $B_{3_1} \eqto B_3$. The resulting bucket function $\hat{f}_{3_1}$ is shown in Fig.~\ref{fig:be}(b) bottom. 
Thus, $\setf{X} \eqto \{x_1,x_2\}$ and \rev{$\setf{C} \eqto \{f(x_1, x_2), \hat{f}_{4_1}, \hat{f}_{4_2}\hat{f}_{3_1}\}$}. 
Next, $x_2$ is processed; \rev{$B_2 \eqto B_{2_1} \eqto \{f(x_1, x_2), \hat{f}_{4_2}, \hat{f}_{3_1}\}$}, 
and $\hat{f}_{2_1}$ is illustrated in Fig.~\ref{fig:be}(c) bottom. 
Thus, $\setf{X} \eqto \{x_1\}$ and $\setf{C} \eqto \{\hat{f}_{4_1}, \hat{f}_{2_1}\}$. 
Lastly, the algorithm processes $x_1$, sets $B_1 \eqto B_{1_1} \eqto \{\hat{f}_{4_1}, \hat{f}_{2_1}\}$, and $\hat{f}_{1_1}$ is minimized when 
$x_1 \eqto 1$, as shown in Fig.~\ref{fig:be}(d) bottom.
\begin{figure}[!t]
  \centering
  \includegraphics[width=1.00\textwidth]{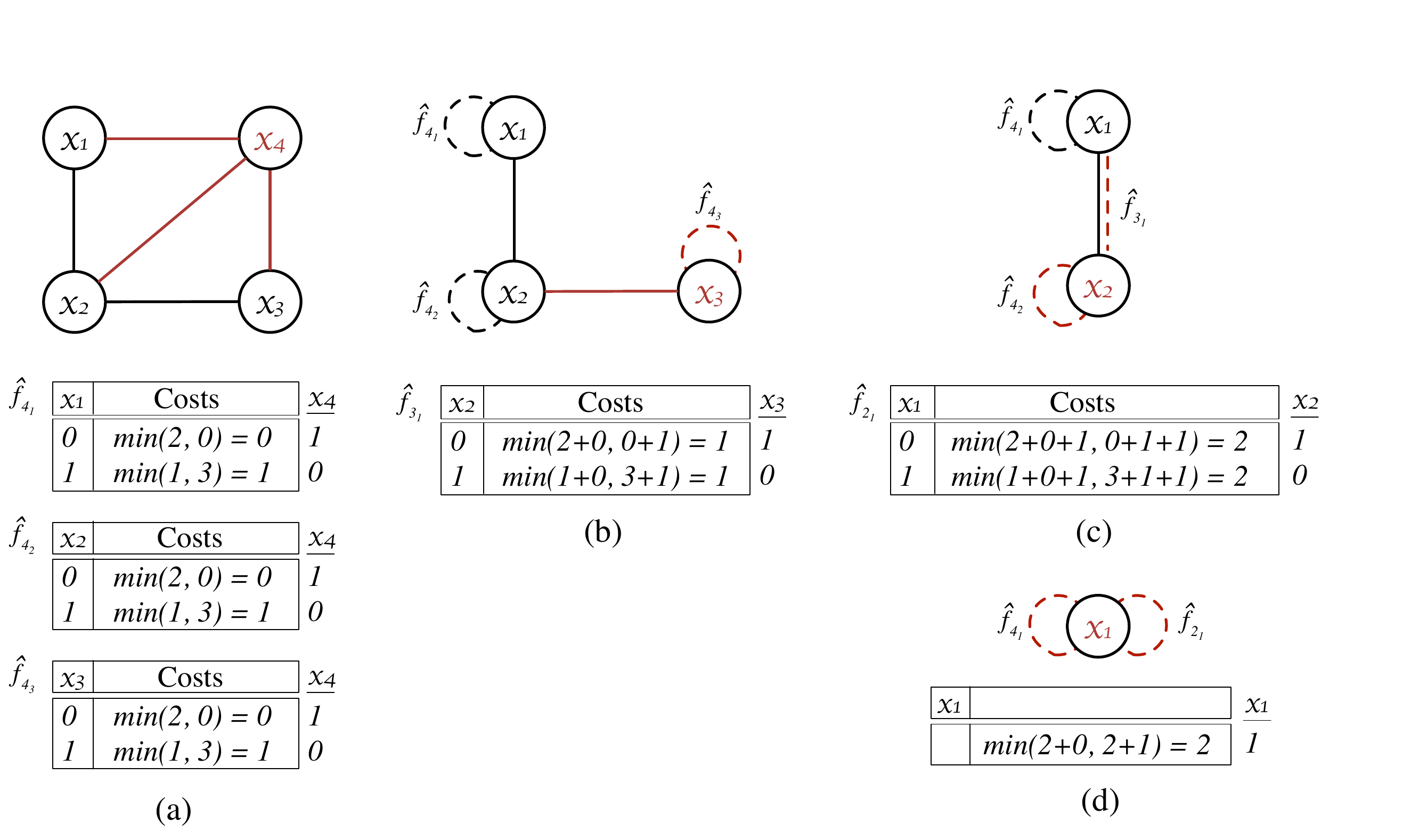}
  \caption{Mini-Bucket Elimination steps for the WCSP of Fig.~\ref{fig:wcsp} :  
  \label{fig:minibe}}  
\end{figure}
%
The \emph{Value Assignment Phase} is analogous to the process carried by BE, except that when processing variable $x_4$ MBE assigns it the value $1$, since it minimizes $\hat{f}_{4_1} + \hat{f}_{4_2} + \hat{f}_{4_3}$ when $x_1\eqto 0, x_2\eqto 1$, and $x_3\eqto 0$ (Fig.~\ref{fig:be}(a) bottom). Thus, $\sigma^* \eqto \tuple{0,1,0,1}$ is the reported solution to the problem, with a lower bound cost of 2.
\end{example}

\section{\rev{Background: Belief Networks and Most Probable Explanation}}
\label{sec:bn}

A \bemph{belief network} (BN) \cite{Pearl:88} is a tuple $\langle \setf{X}, \setf{D}, \setf{P} \rangle$, where 
$\setf{X} = \{x_1, \ldots, x_n\}$ is a set of ordered variables defined over finite domains 
$\setf{D} = \{ D_{x_1}, \ldots, D_{x_n} \}$, with $o$ an ordering of the variables in $\setf{X}$, and 
$\setf{P}$ is a set of \emph{conditional probability tables (CPTs)}. 
A CPT $f_i = \{ \Pr(x_i | \pa_i) \}$ of $\setf{P}$ denotes the join probability of $x_i$ with respect to the variables in $\pa_i$, and $\pa_i \subseteq \{ x_j \in \setf{X} | x_i \prec_o x_j \}$ is the set of variables with higher priority of $x_i$ in the ordering $o$, also called \emph{parent} variables of $x_i$. 
 %
A BN $B$ represents the probability distribution over the variables in $\setf{X}$:
$$
  \Pr\!_{B} (\sigma) = \prod_{i=1}^n \Pr(x_i | \textsl{pa}_i),
$$ 
where $\sigma$ is a complete assignment for the variables in $\setf{X}$. 
The \emph{scope} of a CPT $f_i \in \setf{P}$ is the set $\scope{i} = \{x_i\} \cup \textsl{pa}_i$. An \bemph{evidence set} $\sigma_E$ is an assigned subset of variables $E \subseteq \setf{X}$. 

A BN $B$ is represented through a directed acyclic graph $G_B = (\setf{X}, E_P)$, where $(y, x_i) \in E_P$ iff $\exists f_i \in \setf{P}$ such that $y \in \pa_i$. In other words, $E_P$ is the set of all directed arcs from each parent variable of $x_i$ to $x_i$, for every $x_i \in \setf{X}$. 
The primal graph of a BN is called \bemph{moral graph} and it connects any two variables appearing in the same CPT. 

\begin{figure}[!thb]
  \centering
  \includegraphics[width=0.75\textwidth]{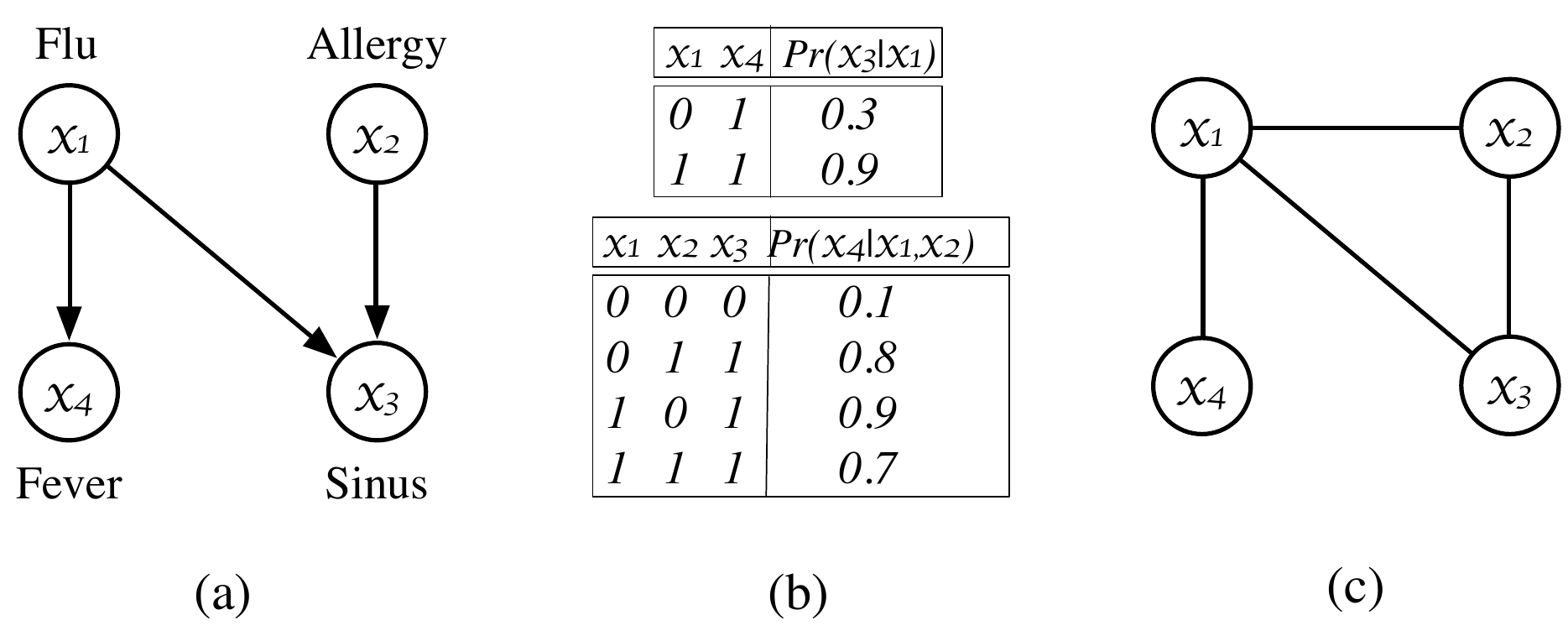}
  \caption{Example of a Belief Network :  
  \textbf{(a):} Belief Network: $Pr(x_4, x_3, x_2, x_1) = Pr(x_4|x_1) Pr(x_3| x_1, x_2) Pr(x_2) Pr(x_1)$. 
  \textbf{(b):} conditional probability tables.
  \textbf{(c):} Its moral graph. 
  \label{fig:bn}}  
\end{figure}

\begin{example}
Fig.~\ref{fig:bn}(a) illustrates a simple belief network with $4$ binary variables: $x_1, x_2, x_3$ and $x_4$, representing respectively the observations for a patient to have \emph{flu}, \emph{allergy}, \emph{fever}, and \emph{sinus infection}, and $4$ CPTs, each associated to a node and its parent nodes. For example, the CPT table illustrated in Fig.~\ref{fig:bn}(b) (top), describes the probability the patient has fever given that she/he does has have \emph{flu}. The conditional probability for $x_4=0$ is implied since the probabilities need to sum up to 1. Similarly, the CPT table of Fig.~\ref{fig:bn}(b) (bottom) describes the probabilities the patient has sinus infection for each combination of outcomes for flu and allergy.
The BN represents the joint probability distribution:
\begin{align*}
\forall x_1, x_2, x_3, x_4, \quad & Pr(x_4, x_3, x_2, x_1) 
	= Pr(x_4|x_1) Pr(x_3| x_1, x_2) Pr(x_2) Pr(x_1).
\end{align*}
Its moral graph is illustrated in Fig.~\ref{fig:bn}(c).
\end{example}

\smallskip

One of the main tasks posed over belief networks is that of finding the \bemph{maximum probably explanation} (MPE). 
%
Given a BN $B$, and an evidence set $E$, finding the MPE correspond to finding a complete assignment for the variables of $B$ that has the maximal probability given the evidence $E$.  More formally, the goal of the MPE is that of finding a complete assignment $\sigma^*$ such that:
\begin{equation}
  \sigma^* = \argmax_{\sigma \in \Sigma} 
    \prod_{f_i \in \setf{P}} \Pr(x_i | \sigma_{\pa_i}, E)
\end{equation}
where $\sigma_{\pa_i}$ is the projection of $\sigma$ to the variables in $\pa_i$.

This problems can be solved with small variations of the Bucket Elimination algorithm presented in section \ref{sec:be}. Bucket Elimination can be adapted to solve MPE tasks on belief networks where the \emph{$\min$ operator} in the projection within the elimination operation is substituted by \emph{$\max$ operator} and the \emph{summation} in the aggregation operator is substituted by the \emph{product} (for more details we refer the reader to~\cite{Dechter:03}).

\section{Background: Distributed Constraint Optimization Problems (DCOPs)}
\label{sec:dcop}

In a \emph{Distributed Constraint Optimization Problem} (DCOP)~\cite{modi:05,petcu:05,yeoh:12}, the variables, domains, and cost functions of a WCSP are distributed among a collection of \emph{agents}. A DCOP is defined as $ \langle \setf{X}, \setf{D}, \setf{C}, \setf{A}, \alpha \rangle$, where $\setf{X}, \setf{D}$, and $\setf{C}$ are defined as in a WCSP,  $\setf{A} \eqto \{a_1, \ldots, a_p\}$ is a set of \emph{agents}, and $\alpha: \setf{X} \rightarrow \setf{A}$ maps each variable to one agent. Following common conventions, we 
assume that $\alpha$ is a bijection: Each agent controls exactly one variable. Thus, we will use the terms ``variable'' and ``agent'' interchangeably and assume that $\alpha(x_i) \eqto a_i$. 
In DCOPs, solutions are defined as for WCSPs, and many solution approaches emulate those proposed in the WCSPs literature. For example, ADOPT~\cite{modi:05} is a distributed version of \emph{Iterative Deepening Depth First Search}, and DPOP~\cite{petcu:05} is a distributed version of BE. The main difference is in the way the information is shared among agents. Typically, a DCOP agent knows exclusively its domain and the functions involving its variable. It can communicate exclusively with its neighbors  (i.e., agents directly connected to it in the \rev{primal graph}\footnote{The \emph{\rev{primal graph}} of a DCOP is equivalent to that of the corresponding WCSP.}), and the exchange of information takes the form of messages.
%
Given a DCOP $P$, and a DFS pseudo-tree $T$ for the \rev{primal graph} $G_P$, we use $N(a_i) \eqto \{ a_j \!\in\! \setf{A} \st \{x_i,x_j\} \!\in\! E_{\setf{C}} \}$ to denote the \bemph{neighbors} of agent $a_i$; 
and $sep(a_i)$ to denote the \bemph{separator} of agent $a_i$, which is the set of ancestor agents that are constrained (i.e.,~they are linked in $G_P$) with agent $a_i$ or with one of its descendant agents in the pseudo-tree $T$.


\begin{example}
\label{ex:dcop}
Fig.~\ref{fig:wcsp}(a--b) illustrate an example of a DCOP instance with $4$ agents, $a_i$  $(i\in \{1\ldots,4\})$,  each controlling one variable, $x_i$. The problem variables, domains and constraints are analogous to those of the WCSP of Example \ref{ex:wcsp}. 
\rev{Fig.~\ref{fig:wcsp}(d)} shows one possible pseudo-tree for the DCOP instance, where the agents $a_1$ and $a_2$ have one pseudo-child: $a_4$. The dotted lines represent backedges.
\end{example} 

\subsection{Dynamic Programming Optimization Protocol (DPOP)}
\label{sec:dpop}

DPOP~\cite{petcu:05} is a dynamic programming based DCOP algorithm that is composed of three phases:
\beitemize
\item \bemph{Pseudo-tree Generation Phase}. In this phase the agents coordinate the construction of a pseudo-tree, realized through existing distributed pseudo-tree construction algorithms~\cite{hamadi:98}. 

\item \bemph{\UTIL\ Propagation Phase}. Each agent, starting from the leaves of the pseudo-tree, computes the optimal sum of costs in its subtree for each value combination of variables in its separator set. The agent does so by aggregating the costs of its functions with the variables in its separator and the costs in the \UTIL\ messages received from its child agents, and then eliminating its own variable. 

\item \bemph{\VALUE\ Propagation Phase}: Each agent, starting from the root of the pseudo-tree, determines the optimal value for its variable. The root agent does so by choosing the value of its variable from its \UTIL\ computations---selecting the value with the minimal cost. It sends the selected value to its children in a \VALUE\ message. Each agent, upon receiving a \VALUE\ message, determines the value for its variable that results in the minimum cost given the variable assignments (of the agents in its separator) indicated in the \VALUE\ message. Once determined, such assignment is further propagated to the children via \VALUE\ messages.
\eitemize


\begin{example}
In our example problem, after coordinating to construct the pseudo-tree (Fig.~\ref{fig:wcsp}(d)), agent $a_4$, being the leaf of the pseudo-tree, starts the \UTIL\ propagation phase, by computing the optimal cost for each value combination of variables $x_1$, $x_2$, and $x_3$ (Fig.~\ref{fig:be}(a)(bottom)), and sending the costs to its parent agent $a_3$ in a \UTIL\ message. Upon receiving the \UTIL\ messages from each of its children,  agents $a_3$ and $a_2$ follow an analogous process.
When the root agent $a_1$ receives the \UTIL\ message from each of its children, it computes the minimum cost of the entire problem, and starts the \VALUE\ propagation phase. It selects the value for its variable that minimizes the problem cost (Fig.~\ref{fig:be}(d)(bottom)) and sends this value down to the pseudo-tree to its child, $a_3$, in a \VALUE\ message. Upon receiving a \VALUE\ message from its parent, each agents follows the same process.
\end{example}

\rev{The time and the space complexities of DPOP are dominated by the \UTIL\ \emph{Propagation Phase}, which is exponential in the size of the largest separator set $sep(a_i)$ for all $a_i \!\in\! \setf{A}$.} 
The other two phases require a polynomial number of linear sized messages (in the number of variables of the problem), and the complexity of the local operations is at most linear in the size of the domain \cite{petcu:05}.

Observe that the \UTIL\ \emph{Propagation Phase} of DPOP emulates the \emph{Variable Elimination Phase} of BE in a distributed context~\cite{brito:10}. 
Given a  pseudo-tree and its ordering $o$, the \UTIL\ message generated by each DPOP agent $a_i$ is equivalent to the aggregated and projected function $\hat{f}_i$ in BE when $x_i$ is processed according to the ordering $o$.


\subsection{Approximate Distributed Pseudotree Optimization} 
\label{sec:ADPOP}

Analogously to how DPOP emulates BE in the distributed context, the \emph{Approximate Distributed Pseudotree Optimization} (ADPOP) algorithm emulates MBE to solve DCOPs \cite{petcu:05c}. ADPOP has the same three phases as DPOP, and given a pseudo-tree and its ordering $o$, the content of the \UTIL\ messages generated by each ADPOP agent $a_i$ is equivalent to the bucket functions $\hat{f}_{i_j}$ ($j \in \{ 1, \ldots, i_m\}$) in MBE when $x_i$ is processed according to the ordering $o$. 

The complexity of ADPOP is exponential in the input parameter $z$, while its \VALUE\ \emph{Propagation Phase} has the same order complexity of the \VALUE\ \emph{Propagation Phase} in DPOP.

\section{Background: Graphical Processing Units (GPUs)}
\label{sec:gpu}

Modern \emph{Graphics Processing Units} (GPUs) are massive parallel architectures, offering thousands of computing cores and a rich memory hierarchy to support graphical processing (e.g., DirectX and OpenGL APIs). NVIDIA's \emph{Compute Unified Device Architecture} (CUDA)~\cite{CUDAbook} aims at enabling the use of the multiple cores of a graphic card to accelerate \emph{general purpose} (non-graphical) applications by providing programming models and APIs that enable the full programmability of the GPU. 
The computational model supported by CUDA is \emph{Single-Instruction Multiple-Threads} (SIMT),  where multiple threads perform the same operation on multiple data points simultaneously.

 \begin{figure}[t]
   \centering
   \includegraphics[width=0.45\textwidth]{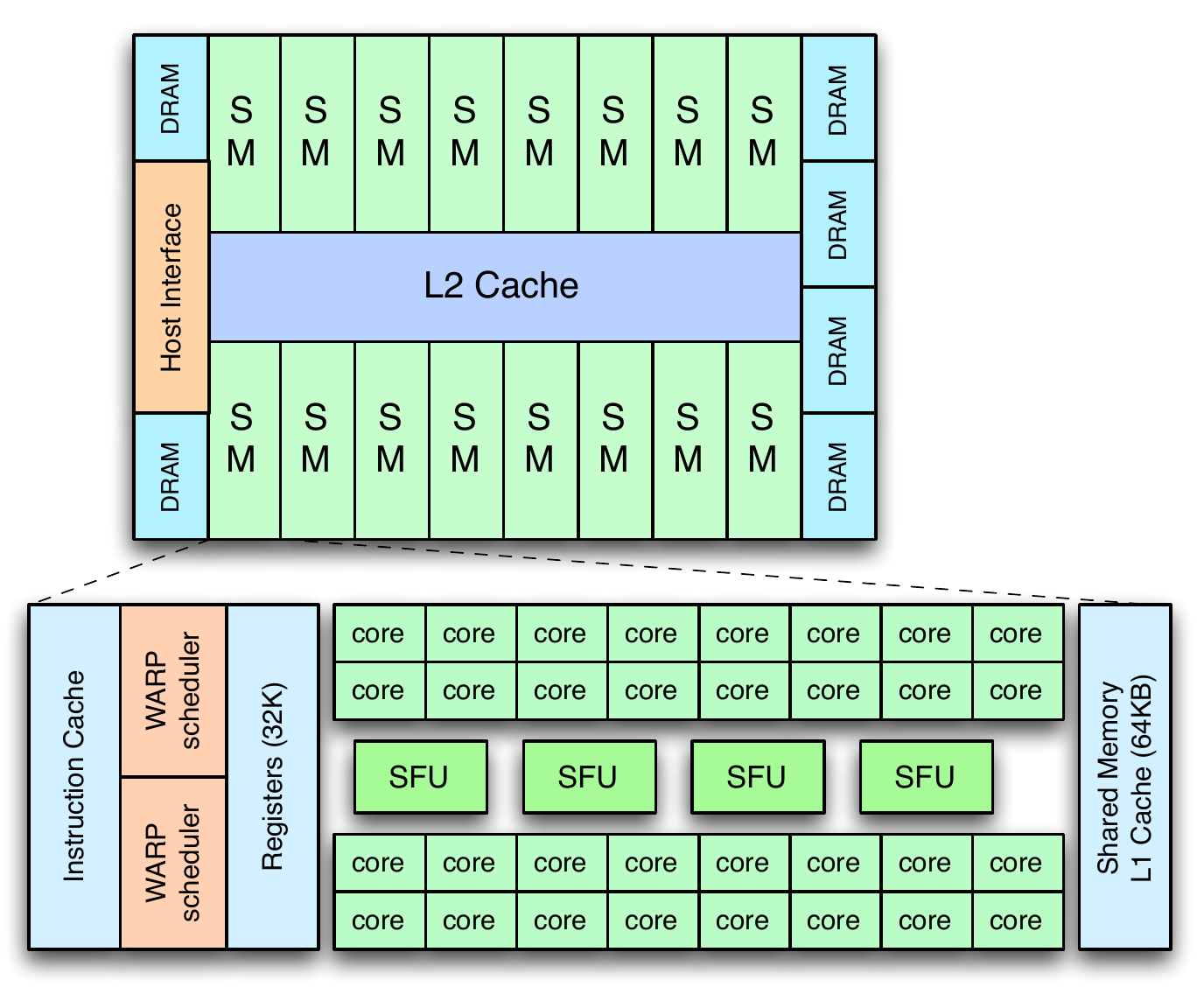}
   \includegraphics[width=0.45\textwidth]{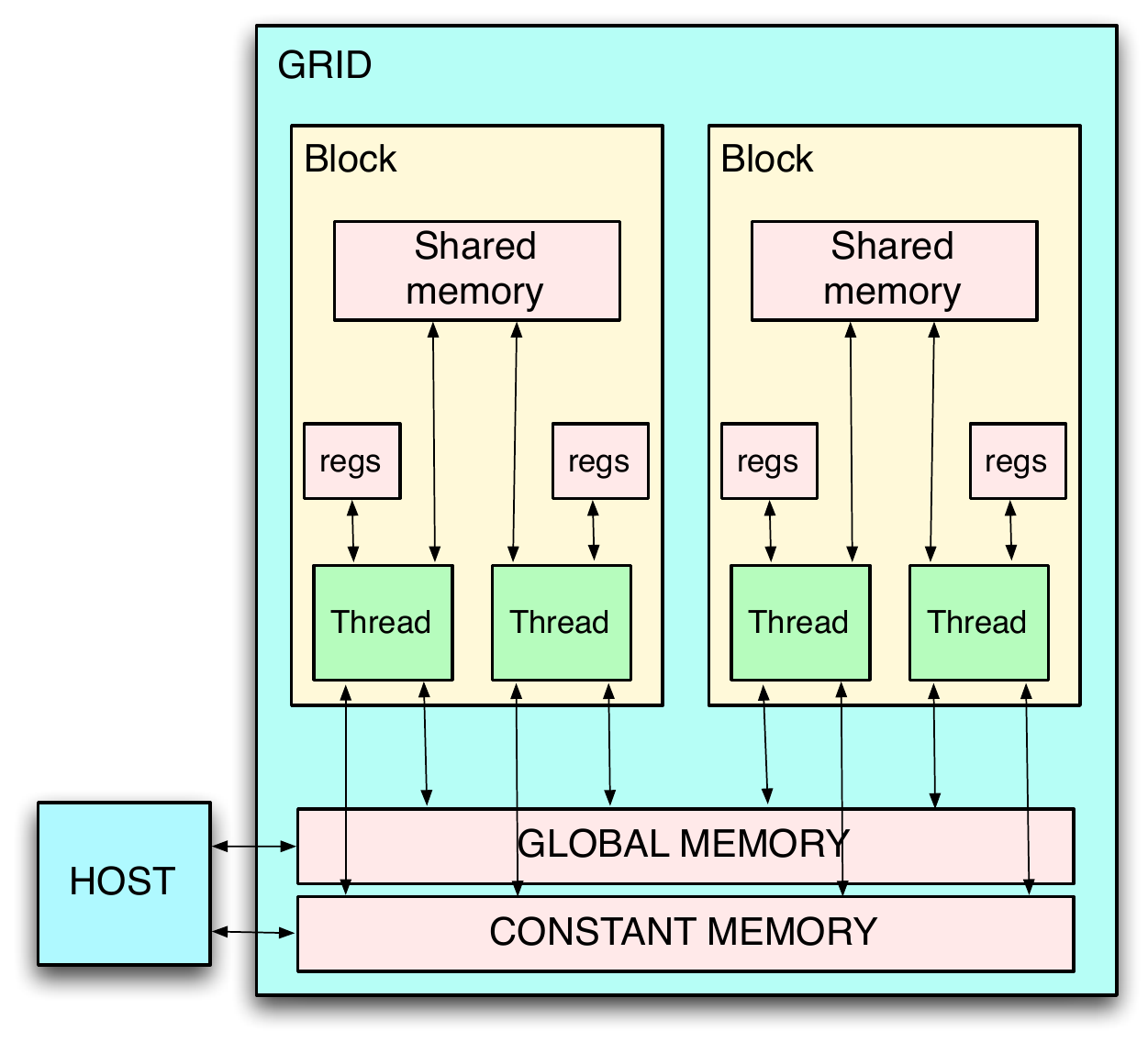}
   \caption{Fermi Hardware Architecture (left) and CUDA Logical Architecture (right)}
   \label{fig:cuda}
 \end{figure}

A GPU is constituted by a series of \emph{Streaming MultiProcessors} (SMs), whose number depends on the specific characteristics  of each class of GPU. For example, the Fermi architecture provides 16 SMs, as illustrated in Fig.~\ref{fig:cuda}(left). Each SM contains a number of computing cores, each of which incorporate an ALU and a floating-point processing unit. 
%
%
Fig.~\ref{fig:cuda}(right) shows a typical CUDA logical architecture. A CUDA program is a C/C++ program that includes parts meant for execution on the CPU (referred to as the \bemph{host}) and parts meant for parallel execution on the GPU (referred as the \bemph{device}). A parallel computation is described by a collection of \bemph{GPU kernels}, where each kernel is a function to be executed by several \bemph{threads}. 
When mapping a kernel to a specific GPU, 
CUDA schedules groups of threads (\bemph{blocks}) on the SMs. In turn, each SM partitions the threads within a block in \bemph{warps}\footnote{A warp is typically composed of 32 threads.} for execution, which represents the smallest work unit on the device. 
%
%
Each thread instantiated by a kernel can be identified by a unique, sequential, identifier ($T_{id}$), which allows to differentiate both the data read by each thread and code to be executed.

\subsection{Memory Organization}

GPU and CPU are, in general, separate hardware units with physically distinct memory types connected by a system bus. Thus, in order for the device to execute some computation invoked by the host and to return the results back to the caller, a data flow needs to be enforced from the host memory to the device memory and vice versa. 
The device memory architecture is quite different from that of the host, in that it is organized in several levels differing to each other for both physical and logical characteristics.

Each thread can utilize a small number of \emph{registers},\footnote{In modern devices, each SM allots 64KB for registers space.} which have  thread lifetime and visibility. Threads in a block can communicate by reading and writing a common area of memory, called \bemph{shared memory}. The total amount of shared memory per block is typically 48KB.
Communication between blocks and communication between the blocks and the host is realized through a large \bemph{global memory}.
The data stored in the global memory has global visibility and lifetime. Thus, it is visible to all threads within the application (including the host), and lasts for the duration of the host allocation. 

Apart from lifetime and visibility, different memory types have also different dimensions, bandwidths, and access times. 
Registers have the fastest access memory, typically consuming zero clock cycles per instruction, while the global memory is the slowest but largest memory accessible by the device, with access times ranging from 300 to 600 clock cycles. 
The shared memory is partitioned into 32 logical banks, each serving exactly one request per cycle.
Shared memory has an extremely small access latency, provided that multiple thread memory accesses are mapped to different memory banks.


\subsection{Bottlenecks and Common Optimization Practices}
 \begin{figure}[t]
   \centering
   \includegraphics[width=0.99\textwidth]{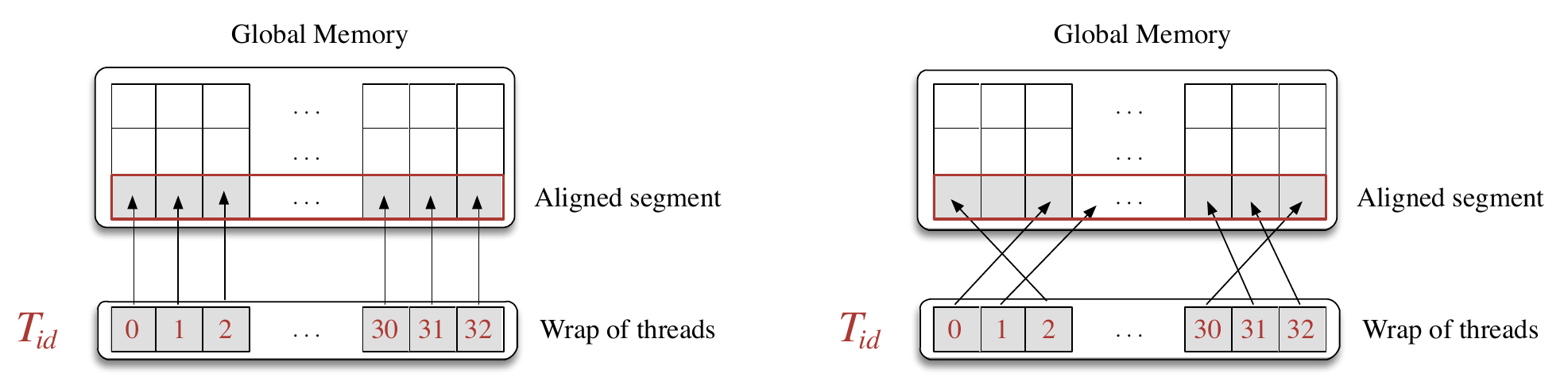}
   \caption{Coalesced (left) and scattered (right) data access patterns.}
   \label{fig:coaleshed}
 \end{figure}

While it is relatively simple to develop correct GPU programs (e.g., by incrementally modifying an existing sequential program), it is nevertheless challenging to design an efficient solution. Several factors are critical in gaining performance. In this section, we discuss a few common practice that are important for the design of efficient CUDA programs.

Memory bandwidth is widely considered to be an important bottleneck for the performance of GPU applications. 
%
%
Accessing global memory is relatively slow compared to accessing shared memory in a CUDA kernel. However, even if not cached, global accesses covering a contiguous 128 Bytes data are fetched at once. Thus, most of the global memory access latency can be hidden if the GPU kernel employs a \emph{coalesced} memory access pattern.
Fig.~\ref{fig:coaleshed}(left) illustrates an example of coalesced memory access pattern, in which aligned threads in a warp accesses aligned entries in a memory \emph{segment}, 
which results in a single transaction. Thus, coalesced memory accesses allow the device to reduce the number of fetches to global memory for every thread in a warp. In contrast, when threads adopt a \emph{scattered} data accesses (Fig.~\ref{fig:coaleshed}(right)), the device serializes the memory transaction, drastically increasing its access latency.

Data transfers between the host and device memory is performed through a system bus, which translates to slow transactions. Thus, in general, it is  convenient to store the data onto the device memory. Additionally, batching small memory transfers into a large one will reduce most of the per-transfer processing overhead~\cite{CUDAbook}.


The organization of the data in data structures and data access patterns play a fundamental role in the efficiency of the GPU computations.
Due to the computational model employed by the GPU, it is important that each thread in a warp executes the same branch of execution. When this condition is not satisfied (e.g., two threads execute different branches of a conditional construct), the degree of concurrency typically decreases, as the execution of threads performing separate control flows can be serialized. This is referred to as \emph{branch divergence}, a phenomenon that has been intensely analyzed within the \emph{High Performance Computing} (HPC) community~\cite{HanA11,ChakrounMMB13,DiamosAMKWY11}.

\section{GPU-based (Distributed) Bucket Elimination (GPU-(D)BE)}
\label{sec:alg}

Our \emph{GPU-based (Distributed) Bucket Elimination} (GpuBE) framework, extends BE and MBE (DPOP and ADPOP, respectively) by exploiting GPU parallelism within the \emph{aggregation} and \emph{elimination} operations.
These operations are responsible for the creation of the functions $\hat{f}_i$ in BE and $\hat{f}_{i_k}$ in MBE (lines~3 and 5 of Algorithms~1 and 2, respectively) and the \UTIL\ tables in DPOP and ADPOP (\UTIL\ \emph{Propagation Phase}), and they dominate the complexity of the algorithms. 
Thus, we focus on the details of the design and the implementation relevant to such operations.
The key observation that allows us to parallelize these operations is that the computation of the cost for each value combination in a bucket function is independent of the computation in the other combinations. The use of a GPU architecture allows us to exploit such independence, by concurrently exploring several value combinations of the bucket function, computed by the aggregation operator, as well as concurrently eliminating out variables.

Due to the equivalence of BE (resp. MBE) and DPOP (resp. ADPOP), we will refer to the \emph{bucket} functions $\hat{f}$ and \emph{\UTIL\ tables} resulted by the aggregation and elimination operations of Algorithms \ref{alg:be} and \ref{alg:mbe}, as well as variables and agents, interchangeably.

\subsection{GPU Data Structures}

In order to fully capitalize on the parallel computational power of GPUs, the data structures need to be designed in such a way to limit the amount of information exchanged between the CPU host and the GPU device, minimizing the accesses to the (slow) device global memory, while ensuring that the data access pattern enforced is coalesced.
To do so, we store into the device global memory exclusively the minimal information required to compute the bucket functions, which are communicated to the GPU once at the beginning of the computation of each bucket or mini-bucket. 
This allows the GPU kernels to communicate with the CPU host exclusively to exchange the results of the aggregation and elimination processes.

We introduce the following concept:
\begin{definition}[Bucket-table] 
A \emph{bucket-table} is a $4$-tuple, $T =\langle \setf{S},  \setf{R}, \chi, \prec \rangle$, where:
\bitemize
\item $\setf{S} \subseteq \setf{X}$, is a list of variables denoting the \emph{scope} of $T$.
\item $\setf{R}$ is a list of tuples of values, each tuple having length $|S|$. Each element in this list (called \emph{row} of $T$) specifies an assignment of values for the variables in $\setf{S}$ that is consistent with their domains. We denote with $\setf{R}[i]$ the tuple of values corresponding to the $i$-th row in $\setf{R}$, for $i = \{1, \ldots, |\setf{R}|\}$.
\item $\chi$ is a list of length $|R|$ of cost values corresponding to the costs of the assignments in $\setf{R}$. 
In particular, the element $\chi[i]$ represents the cost of the assignment $\setf{R}[i]$ for the variables in $\setf{S}$, with $i = \{1, \ldots, |\setf{R}|\}$.
\item $\prec$ denotes an ordering relation used to sort  the variables in the list $\setf{S}$. In turn, the  value assignments, and cost values, in each row of $\setf{R}$ and $\chi$, respectively, obey to the same ordering.
\eitemize
\end{definition}
As a technical note, a bucket table $T$ is mapped onto the GPU device to store exclusively the cost values $\chi$, not the associated variables values. We assume that the rows of $\setf{R}$ are sorted in lexicographic order---thus,
the $i$-th entry $\chi[i]$ is associated with the $i$-th permutation $\setf{R}[i]$ of the variable values in $\setf{S}$, in lexicographic order. This strategy allows us to employ a simple perfect hashing to efficiently associate row numbers with variables' values. We will elaborate on this topic in Section \ref{sec:aggregation}.
Additionally, all the data stored on the GPU global memory is organized in mono-dimensional arrays, so as to facilitate \emph{coalesced memory accesses}.

\begin{algorithm}[htbp]
	\small{\caption{\textsc{GpuBE}(z)}\label{alg:gpu-dbe}}
	\SetNlSty{textbf}{}{~~}
	\tcc{Variable Ordering Phase (Pseudo-Tree Construction)}
	$\setf{\bar{X}} \gets $ Sort $\setf{X}$ w.r.t.~$\prec_{T}$ ordering\;
	
	\tcc{Variable Elimination Phase}
	\For{$i\leftarrow n$ \textnormal{\textbf{downto}} $1$, with $x_i \in \setf{\bar{X}}$} {
		$B_i \gets$ \textsc{Cpu{::}ConstructBucket($\setf{C}, x_i, z$)}\;
		Let  $\{ B_{i_1}, \ldots, B_{i_m}\}$ be a partition of $B_i$ s.t. 
		${\left| \bigcup_{f_j \in B_{i_k}} \hspace{-0pt} \scope{j} \right| \leq z}$, for each $k = 1, \ldots, m$\;
		\ForEach{ $k \in \{1, \ldots, m\}$}	{		
			$T_{i_k} = \langle B_{i_k}, \setf{R}_{i_k}, \chi_{i_k}, \prec_T \rangle \paral{} \textsc{Gpu{::}Reserve}( |\setf{R}_{i_k}| )$\;
			\SetNlSty{textbf}{}{~~}
			\ForEach{$f_j \in B_{i_k}$} {			
				\SetNlSty{textbf}{}{*}
				$T_j = \langle \scope{f_j}, \setf{R}_{j}, \chi_{j}, \prec_T \rangle \paral{D \gets H} \textsc{Gpu{::}Reserve}( |\setf{R}_j| )$\;
				\SetNlSty{textbf}{}{*}
				$T_{i_k} \paral{} \textsc{Gpu{::}Aggregate}( T_{i_k}, T_j )$\;
			}
			$\hat{f}_{i_k} \paral{H \gets D} \textsc{Gpu{::}Eliminate}( T_{i_k}, x_i )$\;
		}
	}
	\tcc{Variable Assignment Phase}
	\For{$i \leftarrow 1$ \textnormal{\textbf{to}} $n$, with $x_i \in \setf{\bar{X}}$} {
		$x_i \gets$ \textsc{Cpu{::}FindBestAssignment($x_1, \ldots, x_{i-1}$)}
	}
	\Return{$\hat{f}_1$}\;
\end{algorithm}

\subsection{Algorithm Overview}
\label{sec:algorithm}

Algorithm~\ref{alg:gpu-dbe} illustrates the pseudocode of GpuBE, where $z$ is an input parameter denoting the maximal mini-bucket size to be processed. 
We use the following notations:
\begin{itemize} 
\item Starred line numbers denote those instructions that are executed concurrently by both the CPU and the GPU. 
\item The symbols $\gets$ and $\paral{}$ denote sequential and parallel (i.e., multiple GPU threads) operations, respectively. 
\item If a parallel operation requires a copy from host (device) to device (host), we write $\!\paral{D \gets H}$ ($\!\paral{H \gets D}$).
Host to device (device to host) memory transfers are performed immediately before (after) the execution of the GPU kernel. 
\end{itemize}

GpuBE is composed of three phases: 
\textbf{(1)} \emph{Variable Ordering}, 
\textbf{(2)} \emph{Variable Elimination}, and 
\textbf{(3)} \emph{Variable Assignment}. 
Let us consider $N(x_i) \eqto \{ x_j \!\in\! \setf{X} \st \{x_i,x_j\} \!\in\! E_{\setf{C}} \}$, 
 defined analogously as for the agents' case.
During the first phase (line~1), the problem variables are sorted
according to a pseudo-tree ordering relation; in particular, we apply the following
heuristics in the construction of the pseudo-tree:
$x_i \prec_T x_j$  \textit{iff}  $|N(x_i)|  < |N(x_j)|$, for every $x_i, x_j \in \setf{X}$.
For the distributed case, this phase is identical to that of (A)DPOP, where the agents coordinate the construction of a pseudo-tree, using an off-the-shelf message-passing algorithm \cite{hamadi:98}.

In the second phase, the algorithm processes each variable, in descending order, according to the relation $\prec_T$, and proceeds as in (M)BE:
\bitemize
	\item The function $\textsc{Cpu::ConstructBucket}$ constructs the bucket $B_i$ as illustrated in Algorithm 1, line~2. The algorithm proceeds in creating a partition of this bucket, if 
	required (i.e., if $z < w^*$).
	This phase differs slightly in the distributed case, where each agent, upon receiving a new bucket function from its descendant agents, inserts it into its bucket set $B_i$.
	
	\item For each mini-bucket $B_{i_k}$ ($k=1,\ldots,m$), GpuBE determines and reserves the amount of global memory to be assigned to each associated bucket-table $T_{i_k}$ (line~6). 
	After the \textsc{Gpu::Reserve} function is invoked, a space sufficient to store the 
	bucket-table is allocated, and its cost values $\chi_{i_k}$ are initialized to $0$.

	\item Thus, GpuBE aggregates the bucket-table $T_j$ associated to each function $f_j$ in the mini-bucket with the bucket-table $T_{i_k}$ (lines~7--9). To do so, it first creates a bucket-table $T_j$ that encodes the cost values of the bucket function $f_j$, reordering them, if necessary, according to the order on its scope specified by the pseudo-tree relation $\prec_T$ (line~8). This procedure requires a memory transfer from the CPU host to the GPU device global memory. Then, it adds the values $\chi_j$ of the aggregating bucket-table $T_j$ into the corresponding entries of the bucket-table $T_{i_k}$ (line~9).
	We will further discuss the details of this function, as well as the other kernel functions, in the next sections. 
	
	\item Finally, the algorithm invokes a GPU call to eliminate the variable $x_i$ from the bucket-table ${T}_{i_k}$, thereby constructing the bucket function $\hat{f}_{i_k}$, which is, finally, copied back to  the CPU host memory (line~10).

\eitemize
In the distributed case, each agent processes lines~3--6 in parallel without prior coordination. 
Starting from the leaves of the pseudo-tree, the agents build their \UTIL\ messages containing the bucket functions (lines~5--10), and send them to their parents. 
Thus, each agent waits to receive the \UTIL\ messages from all of its children before performing the aggregation and elimination operations (lines~7--9 and line~10, respectively) for each mini-bucket. 
By the end of this phase (line~10), the root agent knows the overall cost for each values of its variable $x_i$. Thus, it chooses the value that results in the minimum cost, and it starts the third phase by sending to each child agent the value of its variable $x_i$. 

In the centralized case, when space is not a concern, there is no need of copying the bucket tables back to the host, after the variable elimination step (line~10). Thus, two memory transfer transactions are avoided for each variable being processed.

In the third phase, the algorithm proceeds analogously to as done in (M)BE. For the distributed case, the agents select the values for their variables that minimize their bucket functions costs, given the assignments of their ancestor agents, and send them in \VALUE\ messages to their children. These operations are repeated by every agent receiving a \VALUE\ message until the leaf agents are reached. 

While we described the case in which the underlying problem \rev{primal graph} is connected, our implementation allows us to handle disconnected graphs. This is done by solving the sub-problems in each connected subgraph independently from other subproblems,   
and retrieving the problem cost by aggregating the costs stored in the root of each pseudo-tree associated to the connected graphs.



\subsection{GPU-based Constraint Aggregation}
\label{sec:aggregation}

We now describe the implementation of the constraint aggregation GPU kernel. 
This operation, takes as input two bucket-tables: $T_{i_k}$ and $T_j$, and aggregates the cost values in $\chi_j$ to those of $\chi_{i_k}$ for all the  corresponding assignments of the shared variables in the scope of the two bucket-tables. We refer to $T_{i_k}$ and $T_j$ as to the \emph{output} and \emph{input} bucket-tables, respectively.

Consider the example in Fig.~\ref{fig:aggregate}, the cost values $\chi_j$ of the input bucket-table $T_j$ (right) are aggregated to the cost values $\chi_{i_k}$ of the output bucket-table $T_{i_k}$ (left)---which where initialized to $0$. The rows of the two tables with identical value assignments for the shared variables $x_2$ and $x_3$ are shaded with the same color.

\begin{figure}[htbp]
  \centering
  \includegraphics[width=0.60\textwidth]{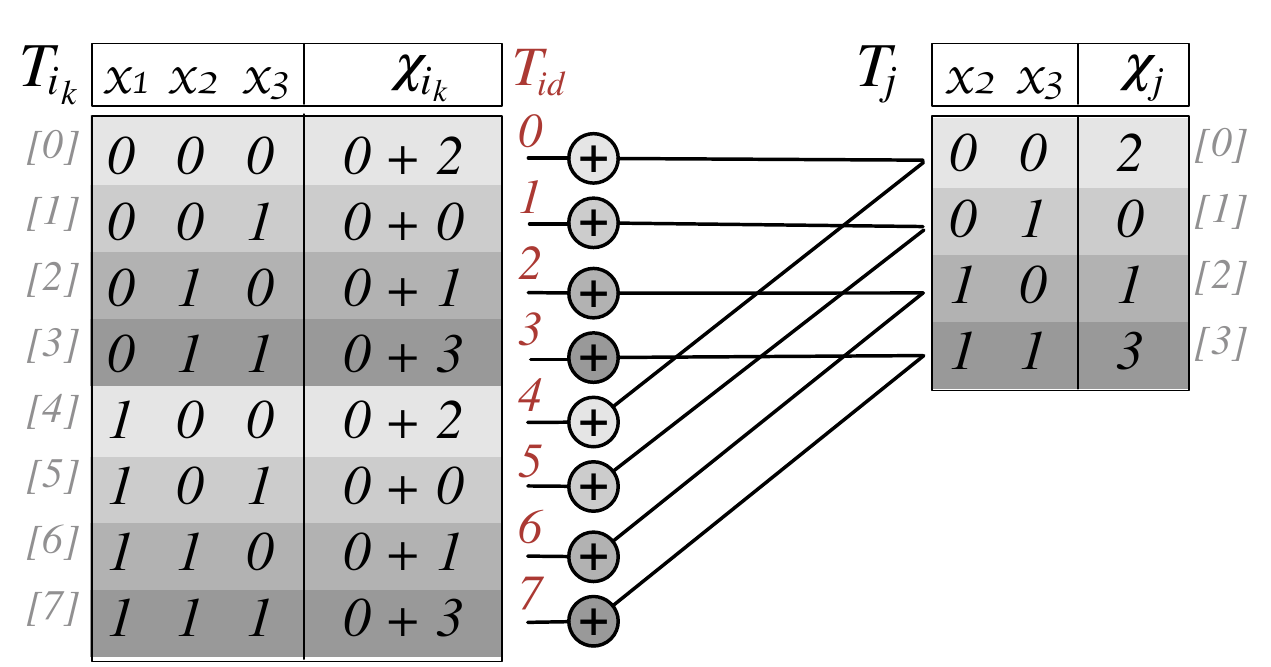}
  \caption{Example of aggregation of two tables on GPU.
  \label{fig:aggregate}}  
\end{figure}

To optimize performance of the GPU operations and to avoid unnecessary data transfer to/from the GPU global memory, we only transfer the list of cost values $\chi$ for each bucket-table that need to be aggregated, and employ a simple perfect hashing
function  to efficiently associate row numbers with variables' values. 
This allows us to compute the indices of the cost vector of the input bucket-table relying exclusively on the information of the thread ID and, thus, avoiding accessing the scope $\setf{S}$ and assignment vectors $\setf{R}$ of the input and output bucket-tables.

We now discuss how this process can be efficiently handled on the GPU kernels.  
Let $T^\out = \langle \setf{S}^\out, \setf{R}^\out, \chi^\out, \prec^\out \rangle$ be the output bucket-table, whose scope is 
$\setf{S}^\out = \{x_1^\out, \ldots, x_m^\out\}$.
Let $T^\inp = \langle \setf{S}^\inp, \setf{R}^\inp, \chi^\inp, \prec^\inp \rangle$ be the input bucket-table, whose scope is $\setf{S}^\inp = \{ x_1^\inp, \ldots, x_s^\inp \}$, and such that $\setf{S}^\inp \sqsubseteq \setf{S}^\out$, 
where $A \sqsubseteq B$ denotes that A is a subsequence of B, and with $s \leq m$. Additionally, let $x_m^\out = x_s^\inp$, i.e., the last variable of the input and output bucket-table scopes coincides. The latter is the variable to be eliminated; We will explain this design choice in the next section, where we will discuss the variable elimination process on a GPU. 
Finally, let $\phi_\out: \mathbb{N} \to \mathbb{N}$ be a mapping from input bucket-table scope variables indexes to output bucket-table scope variable indexes, such that 
$\phi_\out(i) = j$ \textit{iff} $x_i^\inp = x_j^\out$. 
For instance, in our example of Fig.~\ref{fig:aggregate}, $\phi_\out(0) = 1$, as the variable $\setf{S}_j[0] = \setf{S}_{i_k}[1] = x_2$. 
Hence, given a row index $r_{\out}$ for the output bucket-table $\chi^\out$, the corresponding row index $r_\inp$ associated to the input bucket-table cost array $\chi^\inp$ is given by:
\begin{equation}
\small{
r_{\inp} = 
\sum_{k=1}^{s-1} \left[ 
	\underbrace{
		\left(\prod_{j = k+1}^{s} |D_{x_j^\inp}| \right)
	}_{\varf{mul}[k]}
	\!\cdot\!
	 \left( \lfloor
	 	\frac{r_\out}
	        {\displaystyle\underbrace{\prod_{j = \phi_\out(k) + 1}^{m} |D_{x_j^\out}|}_{\varf{div}[k]} }
	         \rfloor 
	 \!\!\!\!\mod\!
	 \underbrace{|D_{x_k^\inp}|}_{\varf{mod}[k]} 
	 \right)   
\right]
	 +
	 r_\out\!\!\!\! \mod \! \underbrace{|D_{x_s^\inp}|}_{\varf{mod}[s]}
}
\label{eq:hash}
\end{equation}
Each term in the summation of Equation~\eqref{eq:hash} represents the contribution of the $k$-th variable's value in $\setf{R}^\out[r_\out]$, as an offset to the index $r_\inp$ in the array $\setf{R}^\inp$. 

The vectors $\varf{mul}$, $\varf{div}$, and $\varf{mod}$ are data structures employed to compute efficiently the $r_\inp$ indices on the GPU. The values $\varf{mul}[k]$, $\varf{div}[k]$, and $\varf{mod}[k]$ (and $\varf{mod}[s]$) can be efficiently computed in $O(s)$, $O(n)$, and $O(1)$, respectively, for each $k=\{1,\ldots,s-1\}$,  and copied onto the GPU global memory with one copy transaction---we allocate them as a single mono-dimensional array.

In order to exploit the highest degree of parallelism offered by the GPU device,  we 
\textbf{(1)} map one GPU thread $T_{id}$ to one element of the output bucket-table $r_\out$ and 
\textbf{(2)} adopt the ordering relation $\prec_T$ for each input and output bucket-table processed. 
Adopting such techniques allows each thread to be responsible of performing exactly two reads and one write from/to the GPU global memory. Additionally, the ordering relation enforced on the bucket-tables allows us to exploit the locality of data and to encourage coalesced data accesses. 
As illustrated in Fig.~\ref{fig:aggregate}, this paradigm allows threads (whose IDs are identified in red by their $T_{id}$'s) to operate on contiguous chunks of data and, thus, minimizes the number of actual read (from the input bucket-table, on the right) and write (onto the output bucket-table, on the left) operations from/to the global memory performed by a group of threads with a single data transaction.\footnote{Accesses to the GPU global memory are cached into cache lines of $128$ Bytes, and can be fetched by all requiring threads in a warp.}

\begin{procedure}[!t]
	  $r_{i_k} \gets $ the thread's entry ID ($T_{id}$)\; 
	  $r_{j} \gets 0$ \tcc{holds the value of the index entry of $\chi_j$}
	  $s \gets |\setf{S}_j|$\;
	  $\tuple{ \varf{mul}, \varf{div}, \varf{mod}} \gets \textsc{CopyToSharedMemory}()$\;
	  \For{$\ell \gets (1 \ldots s \!-\! 1)$} {
	  	$r_{j} \gets r_j + \varf{mul}[\ell] \cdot 
		\big( \lfloor \frac{r_{i_k}}{\varf{div}[\ell]} \rfloor) \% \varf{mod}[\ell]$ \big)\;
	  }
	  $r_{j} \gets r_{j} + \big( r_{i_k} \% \varf{mod}[s] \big)$\;  
	  $\chi_{i_k}[r_{i_k}] \gets \chi_{i_k}[r_{i_k}] + \chi_j[r_j]$\;
	{\caption{{Gpu}{::}Aggregate($T_{i_k}, T_j$)} \label{alg:aggregation}}
\end{procedure}
The constraint aggregation GPU kernel is described in Procedure \ref{alg:aggregation}, which is computed in parallel by a number of threads equal to the number of rows of the output bucket-table. 
Each thread identifies its row index $r_{i_k}$ within the output bucket-table cost values array $\chi_{i_r}$ based on its thread ID (line~1), 
and it initializes a variable that will contain the input bucket-table row index to 0 (line~2). It then copies into the shared memory the static entities $\varf{mul}, \varf{div}$, and  $\varf{mod}$ associated to the aggregation of the the bucket-tables being processed (line~4).
A further inspection to the \ref{alg:aggregation} procedure reveals how it makes use of the auxiliary data structures above to efficiently implement the \emph{hash function} of equation~\eqref{eq:hash}, and retrieve the entry index of the input bucket-table associated to the variables value permutation of the output bucket-table $\setf{R}_{i_k}[r_{i_k}]$ (lines~5--7).
Finally,  the instruction in line~8 aggregates the corresponding input bucket-table value to the output bucket-table $\chi_{i_k}[r_{i_k}]$.

Note that this algorithm highly fits the SIMT paradigm adopted by GPUs; the thread ID and the auxiliary $\varf{mul}, \varf{div}$, and  $\varf{mod}$ arrays are used to retrieve and update all the data necessary to compute the output bucket-table.
Additionally, the accesses to the global memory are minimized, as the auxiliary arrays are copied into the shared memory. 

We illustrate the above process in the following example.
\begin{example}
Consider the operation of aggregating the input bucket-table $T_{j}$ with the bucket-table $T_{i_k}$ of Fig.~\ref{fig:aggregate} corresponding, respectively, to the bucket-table representing the constraint $f_{23}$ and the bucket-table $\hat{f}_3$ (before eliminating the variable $x_3$) in Fig.~\ref{fig:be}(b). 
With the Equation~\eqref{eq:hash} notation, $s = 2$, $m=3$ 
and, thus, the index $k$ of the summation ranges from 1 to $s-1= 1$. Therefore: 
\begin{align*} 
\varf{mul}[0] &= \textstyle \prod_{j=2}^2 |D_{x_j}| = 2 
& \varf{div}[0]  &= \textstyle \prod_{j=\phi_\out(1)+1= 2}^2 |D_{x_j}| = 2 \\
\varf{mod}[0]  &= \textstyle |D_{x_2}| = 2 
& \varf{mod}[1]  &= \textstyle |D_{x_3}| = 2 
\end{align*}
Therefore, the mapping from the thread IDs (or, equivalently, the output bucket-table row indices $r_{i_k}$) to the input bucket-table row indices $r_j$ is: 
\begin{align*} 
T_{id} &= 0  \quad\Rightarrow \quad r_j = 2 \cdot (\textstyle \lfloor \frac{0}{2} \rfloor ) + 0 \!\!\!\!\mod\! 2 = 0 \\
T_{id} &= 1 \quad\Rightarrow \quad r_j = 2 \cdot (\textstyle \lfloor \frac{1}{2} \rfloor ) + 1 \!\!\!\!\mod\! 2 = 1 \\
T_{id} &= 2 \quad\Rightarrow \quad r_j = 2 \cdot (\textstyle \lfloor \frac{2}{2} \rfloor ) + 2 \!\!\!\!\mod\! 2 = 2 \\
T_{id} &= 3 \quad\Rightarrow \quad r_j = 2 \cdot (\textstyle \lfloor \frac{3}{2} \rfloor ) + 3 \!\!\!\!\mod\! 2 = 3 \\
T_{id} &= 4 \quad\Rightarrow \quad r_j = 2 \cdot (\textstyle \lfloor \frac{4}{2} \rfloor ) + 4 \!\!\!\!\mod\! 2 = 0 \\
T_{id} &= 5 \quad\Rightarrow \quad r_j = 2 \cdot (\textstyle \lfloor \frac{5}{2} \rfloor ) + 5 \!\!\!\!\mod\! 2 = 1 \\
T_{id} &= 6 \quad\Rightarrow \quad r_j = 2 \cdot (\textstyle \lfloor \frac{6}{2} \rfloor ) + 6 \!\!\!\!\mod\! 2 = 2 \\
T_{id} &= 7 \quad\Rightarrow \quad r_j = 2 \cdot (\textstyle \lfloor \frac{7}{2} \rfloor ) + 7 \!\!\!\!\mod\! 2 = 3
\end{align*} 
\end{example}

As a technical detail, the bucket-tables are created and processed so that the variables in their scope are sorted according to the order $\prec_T$.
This means that the variables with the highest priority appear first in the scope list, while the variable to be eliminated always appear last. 
We will see, in the next section, that such detail allows us to efficiently encode the elimination operation on the GPU.

\begin{figure}[!h]
  \centering
  \includegraphics[width=0.75\textwidth]{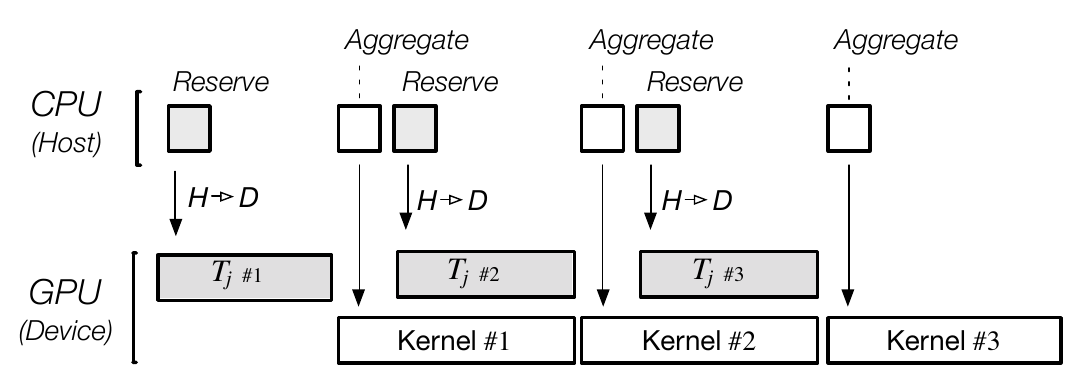}
  \caption{Concurrent computation between host and device.\label{fig:hostdev}}  
\end{figure}

To fully capitalize on the use of the GPU, we exploit an additional level of parallelism, achieved by running GPU kernels and CPU computations concurrently (lines~8--9 of Algorithm \ref{alg:gpu-dbe}). This is possible when the $T_j$ bucket-tables can be partitioned in multiple chunks.  
Fig.~\ref{fig:hostdev} illustrates this operation. After transferring the first bucket-table chunk ($T_j$ ${}_{\#1}$) into the device memory, the process starts the execution of the \ref{alg:aggregation}() kernel, which operates on this portion of the bucket table (called \emph{Kernel} ${}_{\#1}$ in Fig.~\ref{fig:hostdev}). 
Thus, the control immediately returns to the CPU host, which enforces the next data transfer onto the device memory, through a call to a \textsc{Gpu{::}Reserve}($T_j$ ${}_{\#2}$).
A host-device synchronization point is imposed after each memory transfer (except the first one), to ensure that no overlapping \ref{alg:aggregation}() GPU kernels are enforced.


\subsection{GPU-based Variable Elimination}
\label{sec:elimination}
We now describe the implementation of the variable elimination GPU kernel. This operation takes as input a bucket-table $T_{i_k}$ and a variable $x_i \in \setf{S}_{i_k}$ and removes this variable from the bucket-table's scope, optimizing over its cost rows. As a result, the output bucket-table rows list the unique assignments for the value combinations of $\setf{S}_{i_k} \setminus \{x_{i}\}$ in the input bucket-table $\setf{R}_{i_k}$ which minimizes the costs values for each $d \in D_{x_i}$.

\begin{figure}[!t]
  \centering
  \includegraphics[width=0.70\textwidth]{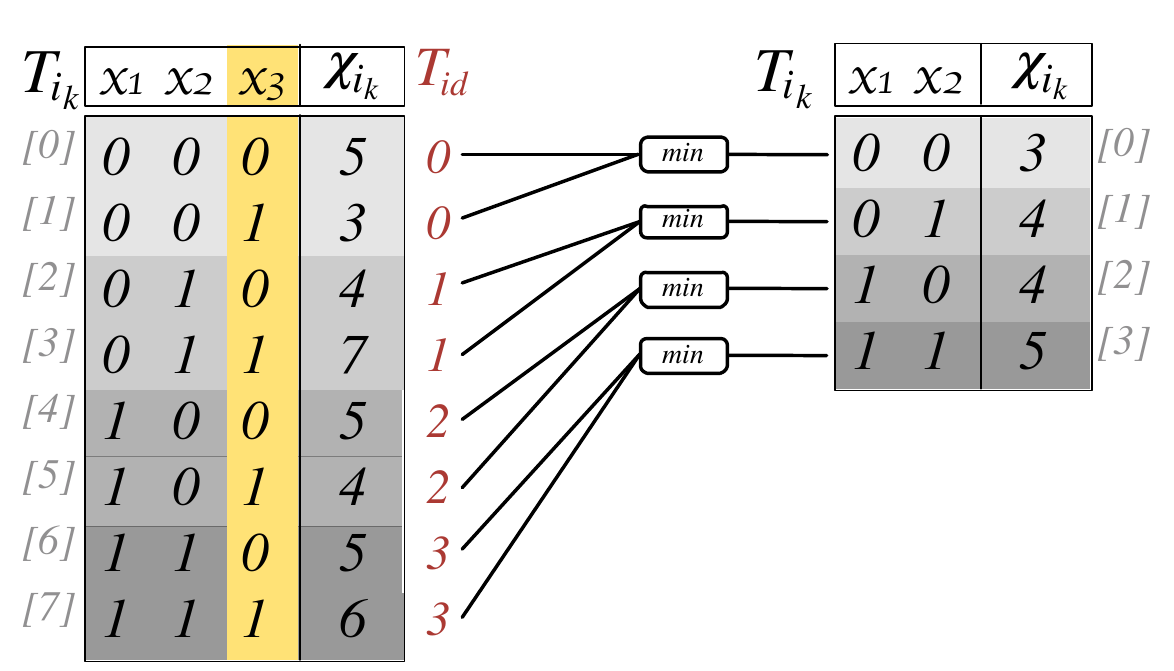}
  \caption{Example of aggregation of two tables on GPU.
  \label{fig:eliminate}}  
\end{figure}
Fig.~\ref{fig:eliminate} illustrates this process, where the variable $x_3$ is eliminated from the bucket-table $T_{i_k}$. The column being eliminated  is highlighted yellow in the input bucket-table. The different row colors identify the unique assignments for the remaining variables $x_1, x_2$, and exposes the high degree of parallelization that is associated to such operation. 
To exploit this level of parallelization, we adopt a paradigm similar to that employed in the aggregation operation on GPU, where each thread is responsible of the computation of a single output element. 

\begin{procedure}[!thb]
	\small{\caption{Gpu{::}Eliminate(${T}_{i_k}, x_i$) \label{alg:eliminate}} }	
	  $r_{i_k} \gets $ the thread's entry ID ($T_{id}$)\; 
	  $r_{j} \gets r_{i_k} \cdot |D_{x_i}|$ \tcc{holds the value of the index entry of $\chi_{i_k}$}
	  $c^* \gets \chi_{i_k}[r_j]$\;

	  \For{$\ell \gets (1 \ldots | D_{x_i} | \!-\! 1)$} {
		  $c^* \gets \min\{c^*, \chi_{i_k}[r_j + \ell] \}$ \;
	  }
	  $\chi_{i_k}[r_{i_k}] \gets c^*$\;

\end{procedure}
The variable elimination GPU kernel is described in Procedure \ref{alg:eliminate}, which is computed in parallel by a number of threads equal to the number of rows of the output bucket-table. Each thread identifies its row index $r_{i_k}$ within the output bucket-table cost values $\chi_{i_k}$ (line~1), given its thread ID. It hence sets an input row index $r_j$ to the value of the first $\chi_{i_k}$ input bucket-table row to analyze (line~1), and it stores in $c^*$ its associated cost. 
Note that, as the variable to eliminate is listed last in the scope of the bucket-table, it is possible to retrieve each unique assignment for the projected output bucket table, simply by offsetting $r_{i_k}$ by the size of $D_{x_i}$. Additionally, all elements listed in $\chi_{i_k}[r_j], \ldots, \chi_{i_k}[r_j + |D_{x_i}|]$ differ exclusively on the value assignment to the variable $x_i$ (see Fig.~\ref{fig:eliminate}).
Thus, the GPU kernel evaluates the input bucket-table cost values associated to each element in the domain of $x_i$, by incrementing the row index $r_j$, $|D_{x_i}| -1$ times, and chooses the minimum cost value (lines~4--5). 
At last, it saves to the associated output row the best cost found (line~6).

Note that each thread reads $|D_{x_i}|$ adjacent values of the vector $\chi_{i_k}$, and writes one value in the same vector. Thus, this algorithm \textbf{(1)} perfectly fits the SIMT paradigm, 
\textbf{(2)} minimizes the accesses to the global memory as it encourages a coalesced data access pattern, and 
\textbf{(3)} uses a relatively small amount of global memory, as it recycles the memory area allocated for the input bucket-table, to output the cost values for the output bucket-table. 

The ordering $\prec_T$ adopted by the bucket-tables makes this procedure effective, 
by forcing the variables to be eliminated to be always listed as last.
Additionally, we note that reordering the bucket-tables scope may be necessary exclusively when constructing the bucket-table associated to the constraints in $\setf{C}$. Indeed, the bucket-tables constructed by the algorithm preserve this ordering over their scope, since all the problem variables are processed according to the same ordering relation $\prec_T$, guaranteeing that the variables being eliminated are those with lower priority with respect to $\prec_T$. Therefore, no reordering will be required in the bucket functions during the process.

%
\smallskip
Finally, to reduce the memory transfer time, in addition to the technique described in the previous section, we unrolled 
the for-loop in lines~7--9 of Algorithm 3. 
Doing so allows us to process all the bucket-tables within a mini-bucket in a single GPU kernel and to copy them to the device using a single transaction.

\section{Theoretical Analysis} 
\label{sec:theo}

We report below a theoretical analysis on the runtime and memory complexity of our GpuBE($z$) algorithms. 
For the distributed case, we report results on the network load and messages size 
complexity provided by the proposed algorithms. 
The \emph{network load} and \emph{messages size} are defined, respectively, as the total number of messages exchanged by the agents and as the size of the largest message exchanged by the agents during problem resolution.
Since our algorithms rely on an inference-based procedure, the agent's complexity (i.e., the maximal number of operations performed by the agents while solving the problem) is equivalent to the size of the largest message exchanged. In turn, the latter corresponds to the memory complexity of the algorithm.
We use \emph{GpuBE}($w^*$) and \emph{GpuDBE}($w^*$) to denote our GPU versions of BE and DPOP, respectively, 
and \emph{GpuBE}($z$) and \emph{GpuDBE}($z$) to denote our GPU versions of MBE and ADPOP, respectively, with mini-bucket size $z$.

\begin{theorem}
For a problem $P$, given an ordering $\prec_T$ on the \rev{primal graph} $G_P$, 
the (mini-)bucket tables (resp.~\UTIL\ messages) constructed by  GpuBE($z$) (resp.~the GpuDBE($z$) agents) are identical to those constructed by (M)BE (resp.~the (A)DPOP agents), for $z \leq w^*$.
\end{theorem}

\begin{proof}
The proof follows from the observation that GpuBE($z$) and (M)BE are executed on the same induced graph $G_P^*$. 
Thus, the problem variables are processed in the same order by both versions of the algorithms---lines 1 and 6 of Algorithm~\ref{alg:be} (1 and 8 of Algorithm~\ref{alg:mbe}) for (M)BE, and lines 2 and 11 of Algorithm~\ref{alg:gpu-dbe} for GpuBE($z$). 
Analogously, in GpuDBE($z$) and (A)DPOP, agents operate on the same pseudo-tree ordering. 

For the centralized case, 
during the \emph{Variable Elimination Phase}, 
the bucket construction and mini-bucket partitioning operations of  GpuBE($z$) (lines 3--4 of Algorithm~\ref{alg:gpu-dbe}) are identical to those of MBE (lines 2--3 of Algorithm~\ref{alg:mbe}). 
For each mini-bucket $B_{i_j}$ in MBE, the operations to create the bucket function 
$\hat{f}_{i_k}$ are identical in both algorithms:  
the effect of invoking the $\textnormal{Gpu{::}Aggregate}( T_{i_k}, T_j )$ routine, in GpuBE, for each bucket-table $T_{i_k}$, corresponding to the bucket function $f_{i_k}$ (lines 7--9 of Algorithm~\ref{alg:gpu-dbe}), is analogous to the 
aggregation operations performed in MBE: 
$F=\sum_{f_j \in B_{i_k}} f_j $ (line 5 of Algorithm~\ref{alg:mbe}, in parenthesis),
and the effect of the $\textnormal{Gpu{::}Eliminate}( T_{i_k}, x_i )$ routine, which projects the variable $x_i$ onto the scope of $T_{i_j}$, produces the bucket function $\hat{f}_{i_k}$, which in turn correspond to the elimination operation performed by MBE:  
$\pi_{-x_i}( F )$ (line 5 of Algorithm~\ref{alg:mbe}).
For the distributed cases, 
both ADPOP and GpuDBE($z$) agents perform the same operations described above---during the \emph{\UTIL\ Propagation Phase}---and populate the \UTIL\ messages they send to their parent.
The equivalence between the Variable Elimination and \emph{\UTIL\ Propagation Phases} of BE and DPOP, with the respective phases in GpuBE($w^*$) and GpuDBE($w^*$), respectively, follows from the process described above differing exclusively in that partitioning $B_i$ produces a single bucket with the same functions as those listed in $B_i$.

The operations performed during the \emph{Variable Assignment Phases} for (M)BE and GpuBE($z$) (lines 5-7, Algorithm~\ref{alg:be}, for BE , lines 8--9, Algorithm~\ref{alg:mbe}, for MBE, and lines 11--12, Algorithm~\ref{alg:gpu-dbe}, for GpuBE($z$)) are identical. Additionally, the variables are processed in the same order in both algorithms. Thus, the solution assignment for the problem variables returned by (M)BE and GpuBE are identical.
Similarly, for the distributed case, (A)DPOP and GpuDBE($z$) agents perform the same \VALUE\ Propagation phase.\hfill$\Box$

\end{proof}


\begin{corollary}
\rev{
For a given $z \leq w^*$, the time and memory (message size) requirements of Gpu(D)BE($z$) are, in the worst case, in O($d^{z+1}$), and O($d^{z}$), respectively,  
where $d = \max_{x_i \in \setf{X}} D_{x_i}$.}
\end{corollary}

\begin{proof}
This result follows from the equivalence of the \emph{Variable Elimination Phases} of (M)BE and GpuBE($z$), and of the 
\emph{\UTIL\ Propagation Phases} of (A)DPOP and GpuDBE($z$).
During these phases, the construction of the (mini)-buckets requires to save, in the worst case, all possible combinations for the value assignments of the bucket-function with bounded arity $z$. Thus, they require $O(d^z)$ space. 
Similarly, for the distributed case, due to the equivalence of (A)DPOP and GpuDBE($z$), the largest message exchanged by the agents  has size $O(d^z)$.

Additionally, the total amount of operations (or, equivalently, bucket-tables rows) that can be processed in parallel during the GPU-based Constraint Aggregation and GPU-based Variable Elimination steps, 
is bounded by a constant value which depends on the GPU card characteristic. Thus, the time complexity of GpuDBE($z$) is in O$(d^{z+1}$). \hfill$\Box$
\end{proof}


\begin{corollary}
The network load required for GpuDBE($z$) is equivalent to the network load required by (A)DPOP.
\end{corollary}

\begin{proof}
This result follow from the equivalence of DPOP with GpuDBE($w^*$) and ADPOP($z$) with GpuBE($z$) (Theorem 1).
Since (A)DPOP requires each agent to send one \UTIL\ message to its parent and one \VALUE\ message to each of its children, there are a total of $n-1$ \UTIL/\VALUE\ messages exchanged---one through each tree-edge of the pseudo-tree $T_P$.  Thus, the network load required by (A)DPOP and GpuDBE($z$) is in $O(n)$. \hfill$\Box$
\end{proof}

\begin{corollary}
Gpu(D)BE is correct and complete.
\end{corollary}

\begin{proof}
The correctness and completeness of GPU-(D)BE($w^*$) follow from the  correctness and completeness of BE \cite{Dechter:99} and DPOP~\cite{petcu:05}, and Theorem 1.\hfill$\Box$
\end{proof}

\section{Experimental Results} 
\label{sec:results}

In this section, we evaluate our GPU implementations of BE and MBE (GpuBE) as well as our GPU implementations of DPOP and ADPOP (GpuDBE) and compare them with their CPU counterparts.\footnote{Our source code is available at \url{https://github.com/nandofioretto/GpuBE}, and \url{https://github.com/nandofioretto/GpuDBE}} 

%
Experiments for GpuDBE and (A)DPOP are conducted using a multi-agent DCOP simulator that simulates the concurrent activities of multiple agents, whose actions are activated upon receipt of a message. 
All algorithms use the same variable ordering in the centralized case and pseudo-tree in the distributed case. 
Performance of the centralized algorithms are evaluated using the algorithms' wallclock runtime, while the performance of distributed algorithms are evaluated using the \emph{simulated runtime} metric~\cite{sultanik:07}.  The main focus of the evaluation is on runtime and speedup achieved by the GPU implementations with respect to their CPU counterparts. 
Additionally, to compare the quality of the solution bounds reported by the incomplete algorithms, we also report the best solution quality found within the given time limits by \emph{toulbar2} \cite{allouchetoulbar2}, an optimized, exact centralized solver for WCSPs. Toulbar2 is a state-of-the-art solver that  uses a depth-first branch-and-bound process to identify a minimum cost assignment and employs the notion of \emph{soft local consistency} to prune the search space using the problem lower bound. 


Our experiments are conducted on an \emph{AMD Opteron 6276} with a 2.3GHz CPU and is equipped with a GPU device \emph{GeForce GTX TITAN} with $14$ multiprocessors, 2688 cores, with a clock rate of 837MHz, and 6GB of global memory.  


We performed our experiments on both randomly generated instances on different networks topologies and on standard WCSP benchmarks.\footnote{Downloadable from \url{http://costfunction.org/en/benchmark/} and \url{http://graphmod.ics.uci.edu/group/Repository}}
We first analyze the runtimes of the CPU and GPU versions of BE and DPOP on randomly generated instances, where we report the runtimes and lower bounds of the GPU and CPU versions of MBE and ADPOP at varying of the bucket size $z$.
Then, to ensure that the speedups are not due to a specific GPU device configuration, we compare the CPU and GPU speedups achieved on $3$ distinct GPU architectures, characterized by different clock rates, number of SMs, and memory sizes. 
Finally, we report the solving time and lower bounds of our GpuBE on an extensive set of WCSP benchmarks to verify the generality of the speedups across different domains.
Each solver has 1-hour timeout of wallclock time in the centralized case and a 1-hour timeout of simulated time in the distributed case. Additionally, they have a memory limit of 32GB to solve each problem instance. 
Results are averaged over all instances. 
If a solver fails to solve an instance is due to either memory limits (labeled \emph{oom}) or timeout (labeled \emph{oot}). 

\subsection{Binary Random Networks}
The instances for each binary network topology are generated as follows:
\bitemize
\item {\bf Random:} We create an $n$-node network, whose density $p_1$ produces  $\lfloor n\,(n-1)\,p_1 \rfloor$ edges in total. We do not bound the tree-width, which is  based on the underlying graph and randomly generated. 

\item {\bf Scale-free:} We create an $n$-node network based on the \emph{Barabasi-Albert model}~\cite{barabasi:99}. Starting from a connected $2$-node network, we repeatedly add a new node,  randomly connecting it to two existing nodes. In turn, these two nodes are selected with probabilities that are proportional to the numbers of their connected edges. The total number of edges is $2\,(n-2)+1$.

\item {\bf Grid:} We create an $n$-node network arranged as a rectangular grid, where each internal  node is connected to four neighboring nodes, while nodes on the grid perimeter are connected to three neighboring nodes unless they are at the corner of the grid, in which case they are connected to two neighboring nodes.
\eitemize
We generate $50$ instances for each topology, ensuring that the underlying graph is connected. The cost functions are generated using random integer costs in $[0, 100]$, and the constraint tightness (i.e., ratio of entries in the cost table that have a cost of $\infty$) $p_2$ is set to $0.5$ for all experiments. 
We set the following as default parameters: For the random and scale-free topology, $n \!=\! 10$, $d \!=\! \max_{D_i \in \setf{D}} \size{D_i} \!=\! 10$, and $p_1 \!=\! 0.3$, and for the grid topology, $\sqrt{n} \!=\! 10$.

%
%
%
\begin{table}
	\centering
	\resizebox{0.9\textwidth}{!}
	{
      \begin{tabular}{|*{4}{c} || *{3}{r}| *{3}{r} | }
        \hline
        \multicolumn{4}{|c||}{Problem} & \multicolumn{3}{| c |}{BE} & \multicolumn{3}{| c |}{DPOP}\\  
	$n$ & $d$ & $p_1$ & $w^*$ & CPU & GPU & speedup & CPU & GPU & speedup\\
	\hline\hline
        10  &  10 & 0.3 & 2.9 & 0.019  &  0.002  &  10.5  &  0.007  &  0.001  &  7.20\\
        11  &  10 & 0.3 & 3.2 & 0.031  &  0.002  &  13.6  &  0.013  &  0.001  &  13.0\\
        12  &  10 & 0.3 & 3.6 & 0.069  &  0.003  &  25.7  &  0.028  &  0.001  &  28.5\\
        13  &  10 & 0.3 & 4.3 & 0.413  &  0.005  &  79.4  &  0.210  &  0.002  &  116\\
        14  &  10 & 0.3 & 4.4 & 0.631  &  0.006  &  98.6  &  0.214  &  0.002  &  134\\
        15  &  10 & 0.3 & 5.3 & 4.190  &  0.026  &  158  &  1.609  &  0.009  &  187\\
        16  &  10 & 0.3 & 5.8 & 32.29  &  0.189  &  171  &  9.848  &  0.049  &  202\\
        17  &  10 & 0.3 & 6.4 & 65.41  &  0.328  &  200  &  28.14  &  0.138  &  204\\
        18  &  10 & 0.3 & 7.5 & 206.1  &  0.944  &  218  &  103.0  &  0.483  &  213\\
        19  &  10 & 0.3 & 8.0 & 602.1  &  2.541  &  237  &  470.2  &  2.019  &  233\\
        20  &  10 & 0.3 & 8.5 & 675.3  &  3.145  &  215  &  508.9  &  2.160  &  236\\
	\hline
        10  &  5  &  0.3 & 3.0 & 0.001  &  0.002  &  0.56  &  0.001  &  0.001  &   0.80\\
        10  &  10 & 0.3 & 2.9 & 0.019  &  0.002  &  10.5  &  0.007  &  0.001  &   7.20\\
        10  &  25  &  0.3 & 2.8 & 0.227  &  0.004  &  55.4  &  0.092  &  0.001  &  92.3\\
        10  &  50  &  0.3  & 2.9 & 24.81  &  0.095  &  262  &  13.99  &  0.048  &  291\\
        10  &  100 &  0.3 & 2.9 & 67.59  &  0.220  &  308  &  35.22  &  0.118  &  299\\
	\hline
        10 & 10  &  0.2  & 2.0 & 0.001  &  0.001  &  0.62  &  0.001  &  0.001  &  0.94\\
        10  &  10 & 0.3 & 2.9 & 0.019  &  0.002  &  10.5  &  0.007  &  0.001  &  7.20\\
        10 & 10  &  0.4 & 3.8 & 0.094  &  0.002  &  40.7  &  0.042  &  0.001  &  42.5\\
        10 & 10  &  0.5  &4.5 & 0.525  &  0.005  &  105  &  0.234  &  0.002  &  130\\
        10 & 10  &  0.6  & 5.4& 3.378  &  0.019  &  176  &  1.941  &  0.011  &  176\\
        10 & 10  &  0.7  & 5.9& 14.86  &  0.072  &  205  &  10.00  &  0.053  &  189\\
        10 & 10  &  0.8  & 6.7& 56.23  &  0.246  &  228  &  31.29  &  0.147  &  213\\
        10 & 10  &  0.9  &7.6 & 72.32  &  0.312  &  232  &  42.47  &  0.201  &  211\\
    \hline
  \end{tabular}
  }
  \caption{Random networks. \label{tab:be_rand}}
\end{table}

\begin{table}
	\centering
  	\resizebox{0.9\textwidth}{!}
	{
      \begin{tabular}{|*{3}{c} || *{3}{r}| *{3}{r} |}
        \hline
        \multicolumn{3}{|c||}{Problem} & \multicolumn{3}{| c |}{BE} & \multicolumn{3}{| c |}{DPOP}\\ 
	$n$ & $d$ & $w^*$ & CPU & GPU & speedup & CPU & GPU & speedup \\
	\hline\hline
        10  &  10 &  6.3 &  22.99  &  0.111  &  207  &  13.78  &  0.064  &  215\\
        11  &  10 & 6.0  &  25.57  &  0.120  &  212  &  13.21  &  0.057  &  231\\
        12  &  10 & 6.0 & 27.96  &  0.132  &  212  &  14.60  &  0.072  &  203\\
        13  &  10 & 5.9 & 80.14  &  0.370  &  217  &  36.21  &  0.174  &  208\\
        14  &  10 & 6.9 & 78.36  &  0.339  &  231  &  32.50  &  0.145  &  223\\
        15  &  10 & 8.2 & 189.4  &  0.887  &  213  &  66.86  &  0.340  &  197\\
        16  &  10 & 9.2 & oom & oom & - & oom & oom & - \\
        17  &  10 & 9.5 & oom & oom & - & oom & oom & - \\
        18  &  10 & 10 & oom & oom & - & oom & oom & - \\
        19  &  10 & 11 & oom & oom & - & oom & oom & - \\
        20  &  10 & 12 & oom & oom & - & oom & oom & - \\
        	\hline
        10  &  5  & 6.8 &  0.322  &  0.004  &   74.8  &  0.175  &  0.001  &  145\\
        10  &  10 &  6.3 &  22.99  &  0.111  &  207  &  13.78  &  0.064  &  215\\
        10  &  25 &  6.6 &  242.5  &  0.888  &  273  &  127.9  &  0.593  &  216\\
        10  &  50 & 6.4 &  oom & oom & - & oom & oom & -  \\
        10  &  100 & 6.4 &  oom & oom & - & oom & oom & -  \\
    \hline
  \end{tabular}
  }
    \caption{Scale-free networks. \label{tab:be_scalefree}}
\medskip

  	\resizebox{0.9\textwidth}{!}
	{
      \begin{tabular}{|*{3}{c} || *{3}{r}| *{3}{r} |}
        \hline
        \multicolumn{3}{|c||}{Problem} & \multicolumn{3}{| c |}{BE} & \multicolumn{3}{| c |}{D-BE}\\ 
	$\sqrt{n}$ & $d$ & $w^*$ & CPU & GPU & speedup & CPU & GPU & speedup \\
	\hline\hline
	5 & 10   & 3.3 & 0.259 & 0.005 &  44.7 & 0.022 & 0.001 & 21.8\\
	6 & 10   & 3.7 & 0.267 & 0.008 &  33.4 & 0.022 & 0.001 & 22.6\\
	7 & 10   & 3.7 & 0.515 & 0.012 &  42.9 & 0.037 & 0.001 & 30.8\\
	8 & 10   & 3.6 & 0.848 & 0.018 &  47.1 & 0.041 & 0.001 & 29.6\\
	9 & 10   & 3.9 & 1.460 & 0.028 &  52.0 & 0.049 & 0.001 & 33.7\\
	10 & 10 & 4.0 & 1.881 & 0.035 &  34.4 & 0.054 & 0.002 & 31.7\\
	11 & 10 & 3.7& 1.934 & 0.040 &  48.3 & 0.073 & 0.002 & 38.4\\
	12 & 10 & 3.8 & 2.174 & 0.042 &  48.4  & 0.089 & 0.002 & 38.7\\
	13 & 10 & 3.9 & 2.430 & 0.045 &  42.2  & 0.102 & 0.003 & 37.8\\
	14 & 10 & 4.0 & 2.996 & 0.055 &  54.5  & 0.127 & 0.003 & 39.7\\
	15 & 10 & 4.0 & 3.785 & 0.071 &  53.3  & 0.151 & 0.004 & 38.7\\
	\hline
	10 &  5  & 3.7 & 0.043 & 0.020    & 2.21 &  0.001 & 0.001 & 1.00\\
	10 & 10 & 4.0 & 1.881 & 0.035 &  34.4 & 0.054 & 0.002 & 31.7\\
	10 & 25 & 3.9 & 97.29 & 0.388 &  251 & 2.930 & 0.011 & 266\\
	10 & 50 & 4.0 & oom & oom	 & - & oom & oom & - \\
	10 & 100 	& 4.0 & oom & oom	 & - & oom & oom & - \\
    \hline
      \end{tabular}
      }
    \caption{Grid networks. \label{tab:be_grid}}
\end{table}
%
Tables \ref{tab:be_rand}--\ref{tab:be_grid} report the runtime, in seconds, for random, scale-free, and grid topologies, respectively, varying the number of variables (resp.~agents) for the centralized (resp.~distributed) algorithms, the size of the variables domains, and the constraint tightness of the \rev{primal graph}. 
The first four (three) columns of Table \ref{tab:be_rand}, (\ref{tab:be_scalefree} and \ref{tab:be_grid}) describe the problem setting adopted for each experiment. The induced width $w^*$ is averaged across all instances. All other columns report the average runtime and GPU vs.~CPU speedup in parenthesis. 
We make the following observations:
\bitemize
\item
The GPU-based inference-algorithms are consistently faster that their CPU counterparts, with speedups of up to 307x. 
Only two exceptions arise for the random networks, where in the small instances with $n=10$, $d=5$, $p_1=0.3$, and $n=10$, $d=10$, $p_1=0.2$, the GPU versions of the algorithms are slower than their CPU counterparts.
\item
The speedup increases with the problem size. In particular, the speedup increases with increasing induced width and with increasing domain size of the problem variables. Both these factors influence the size of the bucket-tables to be processed.\footnote{Recall that BE needs to process bucket-tables whose number of rows is in $O(d^{w^*})$.}  
This observation corroborates the effectiveness of the GPU parallelism exploited in the construction of these tables. 
\item
As expected, the inference-based algorithms are unable to process instances characterized by large induced widths 
or large domain sizes, 
as the size of the bucket-tables become intractable with the memory limitations.
 This is evident in the scale-free and grid networks, where the solvers run out of memory for instances with $n \geq 16$ and $d \geq 50$, respectively.
%
%

\item 
The simulated runtimes of the DCOP algorithms are consistently smaller than the wallclock runtimes of the WCSPs ones.
This is due to the fact that agents in different branches of the pseudo-tree can compute their bucket-tables independently from each other. 
\item 
Finally, the speedup trends of the distributed algorithms are similar to those of the centralized algorithms. 
\eitemize

\begin{figure}[!t]
  \centering
  \includegraphics[width=0.7\textwidth]{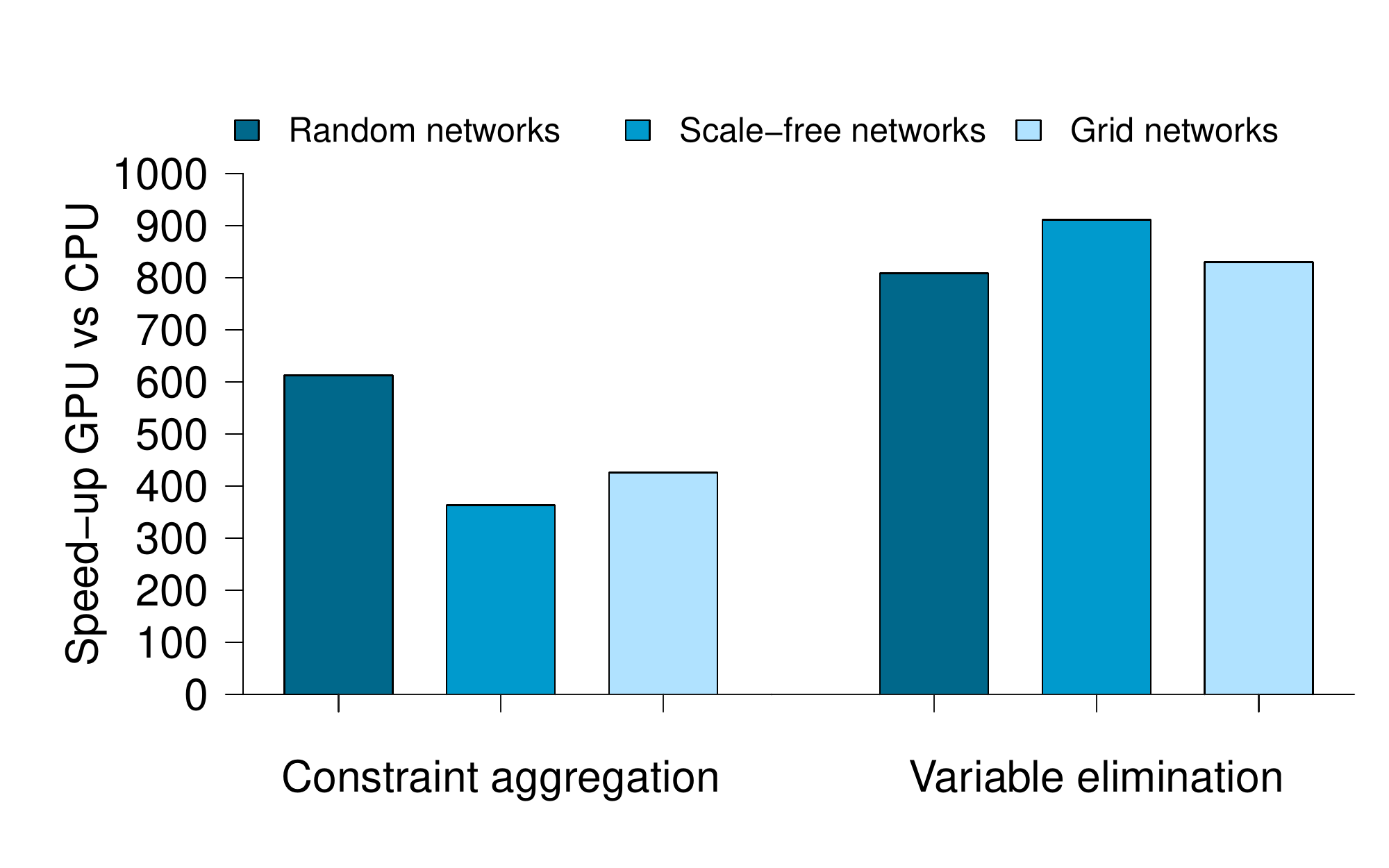}
  \caption{Analysis of the average speedup obtained by the GPU-based constraint aggregation, and GPU-based variable elimination w.r.t.~their CPU-based counterparts in the random, scale-free, and grid network instances. \label{fig:funcSpeed}}
\end{figure}

Next, we analyze the performance of the individual kernels that implement the constraint aggregation and the variable elimination processes described, respectively, in Sections~\ref{sec:aggregation} and~\ref{sec:elimination}. 
Figure~\ref{fig:funcSpeed} illustrates the average speedup obtained by the GPU-based constraint aggregation, and the GPU-based variable elimination with respect to their CPU-based counterparts when considering the largest bucket processed in each instance of the random, scale-free, and grid network instances.
The reported average speedup for the constraint aggregation operations range from 363x (in scale free networks) to 613x (in random networks). 
The variable elimination operations achieve an even higher speedup, ranging from 830x (for grid networks) to 911x (for scale free networks). This is due to the high locality of data exploited by the GPU-based variable elimination kernel, which encourages coalesced data accesses, and through memory reuse, where we overwrite the input bucket-table of the variable elimination process with the resulting bucket-table from the same process.

\begin{figure}[!t]
  \centering
  \includegraphics[width=0.47\textwidth]{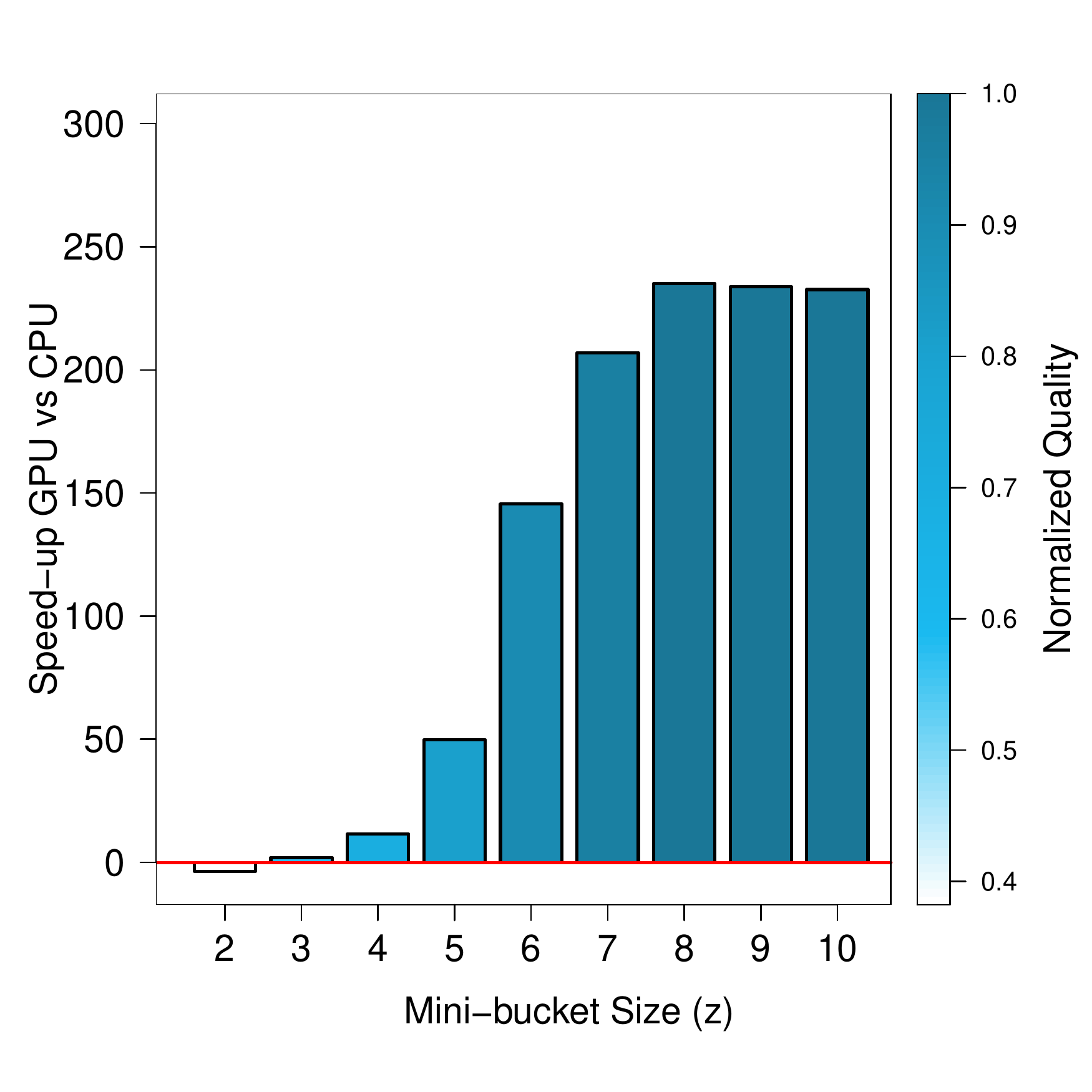}
  \includegraphics[width=0.47\textwidth]{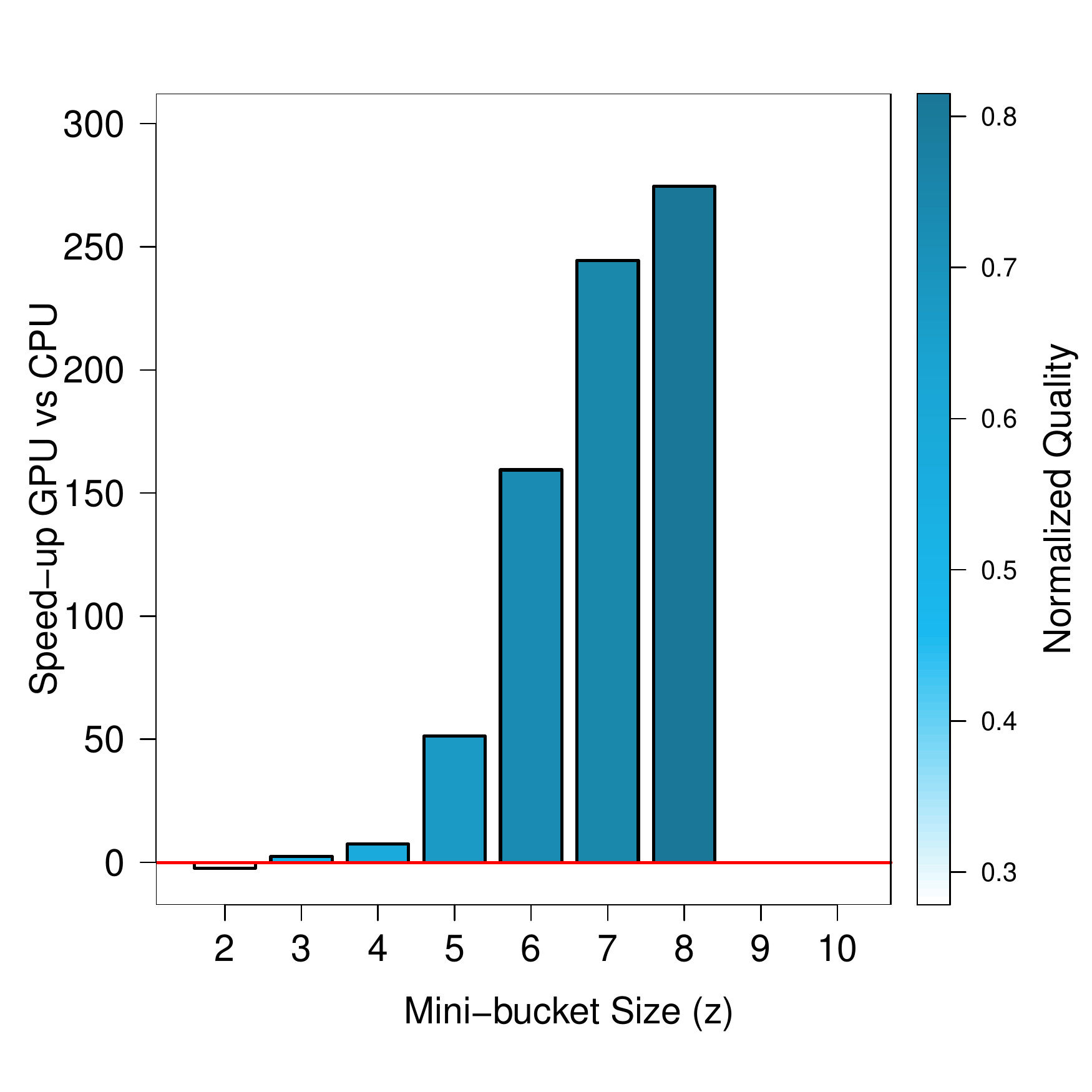}\\
  {\small (a) \hspace{150pt} (b)}\\
  \includegraphics[width=0.47\textwidth]{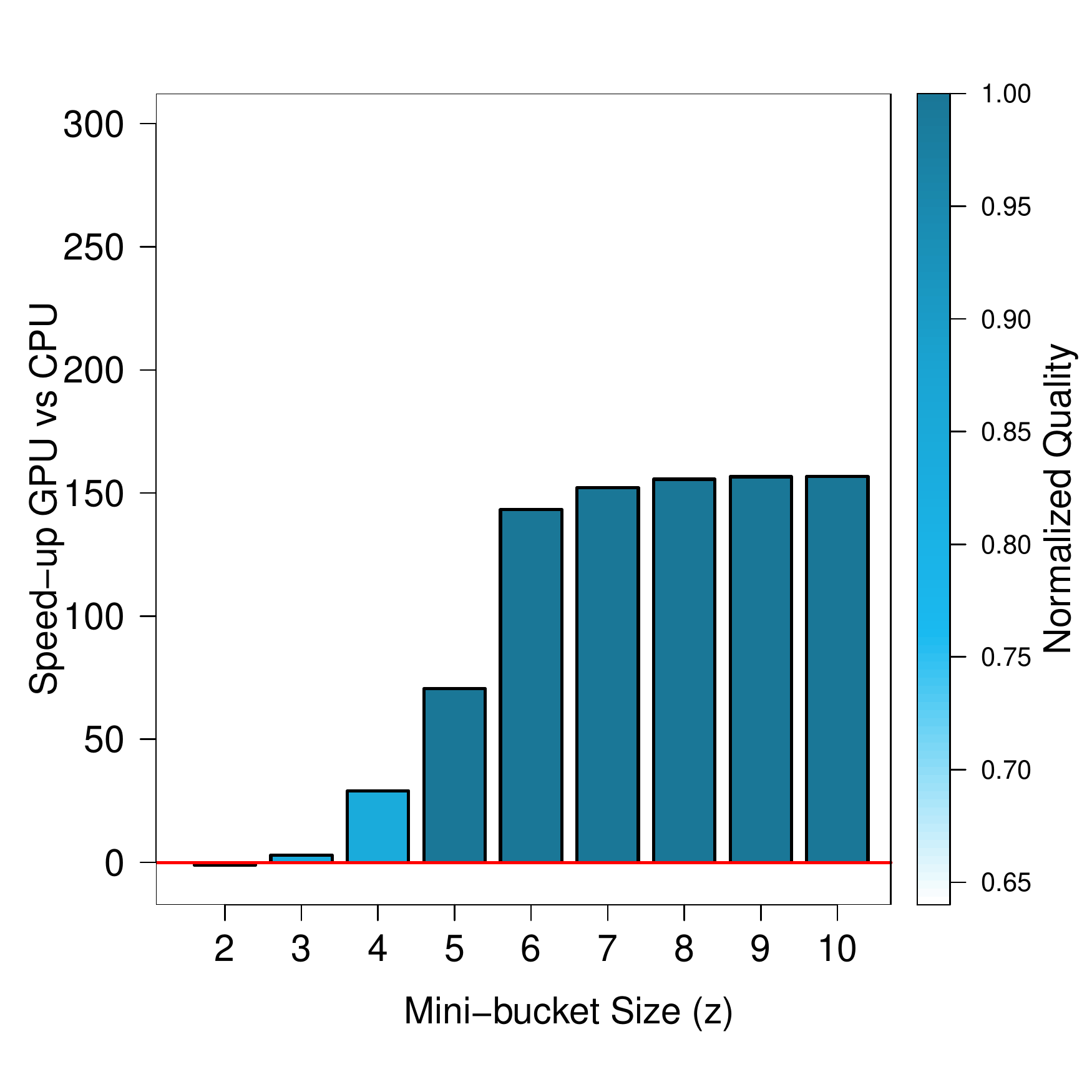}\hspace{18pt}
  \includegraphics[width=0.44\textwidth]{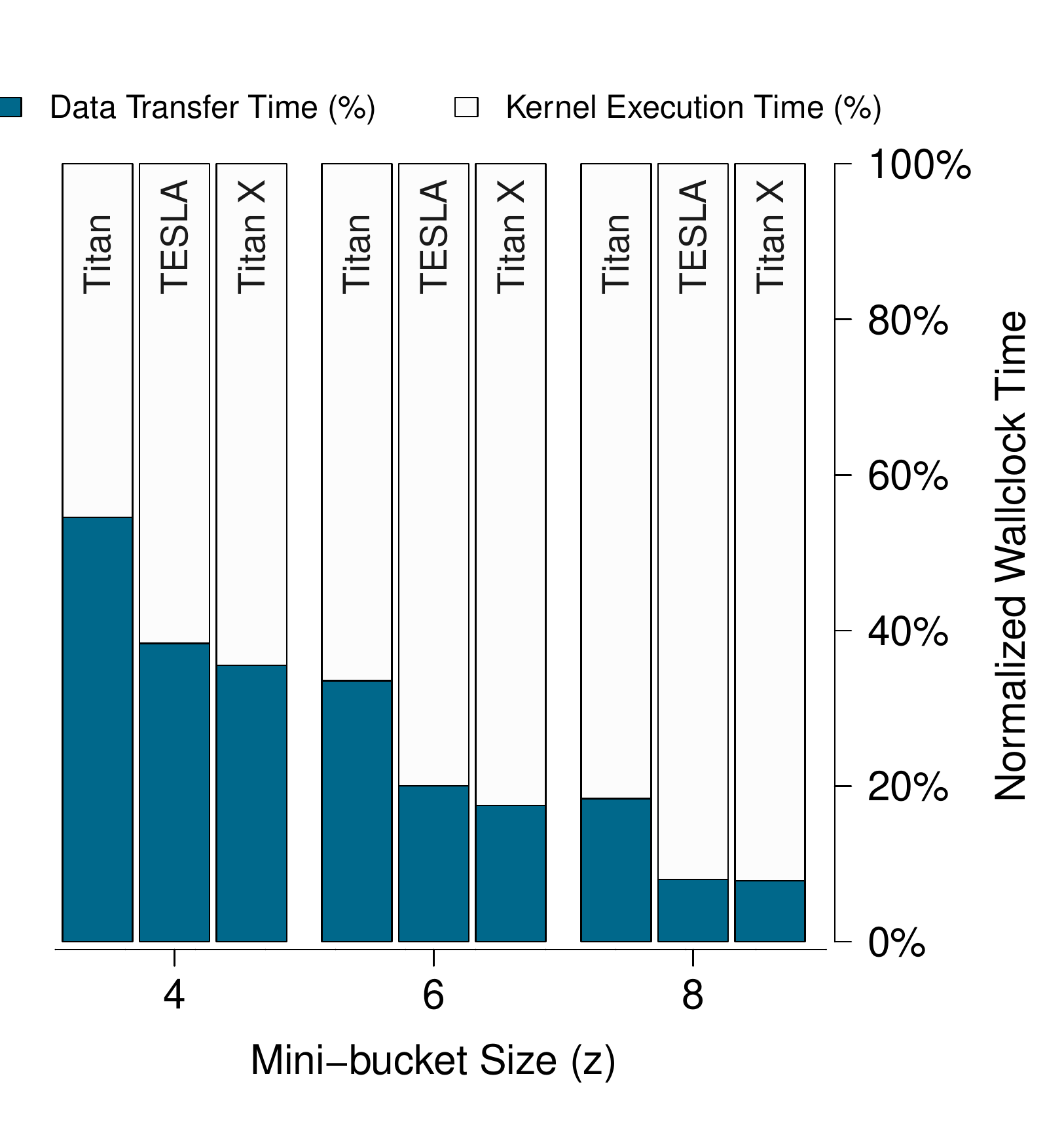}\\
  \caption{ 
  (a)~MBE Results on Random Networks: $n = 20, d  = 25, p_1 =0.3$; 
  (b)~MBE results on Scale-free Networks: $n = 20, d  = 25$;
  (c)~MBE results on Grid Networks: $\sqrt{n} = 10, d = 25$;
  (d)~Normalized data allocation and transfer times (blue) vs.~kernel times (white) on different GPUs.
  \label{fig:res_mbe1}}  
\end{figure}

Next, we compare our centralized and distributed versions of GpuBE with MBE~\cite{Dechter:99} and ADPOP~\cite{petcu:05} at varying of the mini-bucket size $z \in \{2,\ldots, 10\}$,  on binary constraint networks with \emph{random}, \emph{scale-free}, and \emph{grid} topologies, using the same settings described in the previous section. The instances for each topology are generated as described above.
Fig.~\ref{fig:res_mbe1}(a--c) illustrate the speedup of the CPU and GPU versions of MBE, respectively, on random networks with $n = 20$, $d  = 25$, $p_1 =0.3$, on scale-free networks with $n = 20$, $d  = 25$, and on grid networks with $\sqrt{n} = 10, d  = 25$.
The intensity of the color illustrates the solution quality of the bound returned (darker color denotes better solution quality).
We make the following observations:
\bitemize
\item
The speedup obtained by the GPU vs.~CPU solvers increases as the size of the mini-buckets increases. This observation is consistent with the previous observation that the speedup increases with increasing induced widths. 
\item 
The speedup saturates when $z=7$ in all benchmarks, reporting maximal speedups of 235\textsl{x},  274\textsl{x}, and 156\textsl{x}, for random, scale-free, and grid networks, respectively. 
This phenomena occurs when the maximum concurrent number of GPU threads are scheduled and executed simultaneously by all the GPU SMs---i.e., when there is enough work to saturate the GPU maximal occupancy.
\item
As for the previous experiment, the speedup trends of the distributed algorithms are similar to those of the centralized algorithms. The correlation\footnote{We use the \emph{Pearson product-moment correlation} coefficient.} of the CPU vs.~GPU speedup between the centralized and the distributed solutions are 0.93, 0.95, and 0.99, respectively for the grid, random, and scale-free network topologies.
\eitemize

\begin{table}
  \centering
	\resizebox{\textwidth}{!} {
       \begin{tabular}{ l l l l }
       \hline
       & \bemph{TESLA M2075} & \bemph{GeForce GTX Titan} & \bemph{GeForce GTX Titan X} \\[2pt]
	\hline
       CUDA Capability 		& 2.0 & 3.5 & 5.2 \\ 
	Global Memory Size 		& 5375 MB & 6137 MB & 12286 MB \\
	Number of SMs 			& 14 & 14 & 24 \\ 
	Cores per SM 			& 32 & 192 & 128 \\
	GPU Max Clock Rate 	& 1.15 GHz & 0.88 GHz & 1.08 GHz \\
	Memory Clock Rate 		& 1566 Mhz & 3004 Mhz & 3505 Mhz \\
	L2 Cache Size 			& 786 KB & 1572 KB & 3145 KB \\
	Max Number of Threads per SM  		 & 1536 & 2048 & 2048 \\
	Concurrent copy and execution & yes & yes & yes \\
       \hline
	\end{tabular}
	}
	\caption{GPU device specifics. \label{tab:gpus}}
\end{table}
\begin{table}
  \centering
	\resizebox{\textwidth}{!} {
       \begin{tabular}{ c | *{3}{r} | *{3}{r} | *{3}{r}}
       \hline
       & \multicolumn{3}{c |}{Grid} & \multicolumn{3}{c |}{Random} & \multicolumn{3}{c}{Scale-Free}\\
   z  &  TESLA  &  Titan  &  Titan X  &  TESLA  &  Titan  &  Titan X  &    TESLA  &   Titan  &    Titan X\\
   \hline
  2  &  0.22  &  0.85  &  0.49  &  0.21  &  0.28  &  0.60  &  0.27  &  0.43  &  0.57\\
  3  &  1.54  &  3.05  &  2.80  &  1.42  &  1.99  &  3.25  &  1.48  &  2.42  &  1.61\\
  4  &  21.9  &  29.1  &  33.0  &  10.7  &  11.6  &  16.1  &  12.2  &  7.46 &12.0\\
  5 & 117 &  150 	&  232 &   60.6 &   49.8 &   66.4 &   59.3 &   51.3 &   66.8\\
  6 &  117 &  143 	&  237  & 144 &  145  & 223 &  163 	 & 159 &  285\\
  7 &  117 &  152 	&  235  & 198  & 207  & 392  & 208  	& 244  & 435\\
  8 &  118 &  153  	& 241  & 198  & 235  & 645  & 211 		 & 274  & 627\\
  9 &  115 &  156  	& 238  & 197  & 234  & 620  &  oom  & oom  & oom\\
 10 & 117 &  155  	& 235 &  199  & 233 &  628  & oom   & oom & oom\\
       \hline
	\end{tabular}
	}
	\caption{CPU vs.~GPU speedup on different GPU devices. \label{tab:gpusPerformance}}
\end{table}

\smallskip
Table \ref{tab:gpusPerformance} illustrates a comparison of the speedups obtained with three different GPU hardware configurations: \emph{TESLA M2075}, \emph{GeForce GTX Titan} and \emph{GeForce GTX Titan X}, whose specifics are summarized in Table~\ref{tab:gpus}.\footnote{In all other experiments we used the GeForce GTX Titan, as this is the best, most affordable card at our disposal.} 

Among the three GPU devices, the TESLA M2075 achieve the lowest maximal speedups, which range from 117.8\textsl{x} to 213,7\textsl{x}. 
Additionally, the speedup saturates when $z=5$ for grid networks and $z=7$ for random and scale-free networks.
This is due to the fact that this card can schedule the smallest number of cores per each SMs (32). Since each core can run concurrently a wrap (32 threads), its maximal level of concurrency is $14 \times 32 \times 32 = 14,336$ threads (and is thus the maximum number of parallel aggregation operations).
In contrast, the speedup obtained by our GpuBE is the highest on the GeForce GTX Titan X---obtaining a maximal speedup of 646.9\textsl{x}---and saturates when $z=8$ in all networks.
The maximum number of threads that can run concurrently on this card is $24 \times 128 \times 32 = 98,304$.
The speedups obtained by our solver on the GeForce GTX Titan, used in the rest for the experiments in this paper, are larger than those obtained on the TESLA but smaller than those obtained on the GeForce GTX Titan X. This card can run up to $86,016$ threads. In addition to the number of threads than can run concurrently, the GPU clock rate and L2 cache size play a substantial role in the GPU performance.


Finally, Fig.~\ref{fig:res_mbe1}(d) illustrates the time spent by the GPU devices while executing the kernel functions (in white) in contrast to the time used for memory transfers and allocations (in blue), 
at varying mini-bucket size $z=\{4,6,8\}$. 
These times are averaged among all instances for the three network topologies examined and are normalized with the respect to the wallclock runtime.
The results show that the time spent by the device in performing actual computations increases, with the respect to the memory transfer time, as the mini-bucket size increase. Allocations and memory transfers on the Titan device are slower than on the TESLA and the GTX Titan X. Finally, these times account for the 36\% to 55\%, 18\% to 34\%, and 8\% to 18\% of the total time, respectively for the mini-bucket sizes $4, 6$, and $8$.


\subsection{WCSPs Benchmarks}
We now report the evaluation of our GpuBE on the following standard WCSPs benchmarks:
\bitemize
\item \emph{Celar}: Radio link frequency assignment problems.
\item \emph{Coloring}: Graph coloring instances cast into minimum coloring instances.
\item \emph{Iscas89}: WCSPs derived from digital circuits.
\item \emph{Pedigree}: Instances from the genetic linkage analysis domain that is associated with the task of haplotyping. 
\item \emph{Spot}: Instances of the daily photograph scheduling problem of Earth observation satellites.
\eitemize

Tables \ref{tab:celar}--\ref{tab:pedigree},  
tabulate the results for the above benchmarks.
In each table and for each instance, we report, in order, the instance name---as appearing in the original benchmark---the number of variables $n$ of the problem, the maximum size of their domains $d$, the number of constraints $c$, the graph density $p_1$, and the induced width $w^*$ of the underlying \rev{primal graph}.
In each table, the top row shows the runtimes in seconds of GpuBE(z) at varying bucket size $z$ and GpuBE. 
The bottom row shows the returned solutions' qualities, where for GpuBE(z), we report the lower bound it returned.
When GpuBE failed to report a solution (due to memory limits), we report the solution quality found by toulbar2 (shown in parenthesis) or a dash symbol, if toulbar2 did not terminate within the time limit.
The speedup of GpuBE(z) and GpuBE w.r.t. their CPU counterparts are shown in parentheses. 
For each instance, we vary the bucket size $z$ from $2$ to $20$, and report the minimum bucket size $z_{\min}$, which is the largest constraint arity of the instance, the maximum bucket size $z_{\max} = \min\{w^z, 20\}$, where
$w^z$ is defined as the maximal bucket size that can be processed within the hardware memory limits, and the intermediate bucket sizes $z_2 = z_{\min} + \frac{1}{3} (z_{\max} - z_{\min})$ and $z_3 = z_{\min} + \frac{2}{3} (z_{\max} - z_{\min})$.

Consistent with our previous observations, the algorithms' speedups and solution qualities increase as the bucket size increases. 
Additionally, for several large problems instances (e.g., \textit{scen06-24reduc}---\textit{scen06reduc} in the Celar benchmark), our GPU implementation of MBE can report good lower bounds quickly (within a few seconds), whereas solving the entire problem with the most competitive soft consistency technique in toulbar2 requires from 6 to 48 minutes. 
For other large instances (e.g., in the Spot benchmark), we observe that toulbar2 ran out of time for the majority of the instances, while our GpuBE(z) can quickly find lower bounds, which could be used in a AND-and-OR search type as proposed by Marinescu and Dechter~\cite{marinescu:09}.

\begin{table}
  \begin{center}
	\resizebox{0.9\textwidth}{!}
{
        \begin{tabular}{| l | p{6pt} p{2pt} p{12pt} p{6pt} p{6pt} || *{5}{r} | }
    \hline
	\prob{Problem}{$n$}{$d$}{$c$}{$p_1$}{$w^*$} & \multicolumn{4}{|c}{GpuBE($z$)} & GpuBE \\
	\prob{}{}{}{}{}{} &  $z_{\min}$ & $z_2$ & $z_3$ &  $z_{\max}$ & $w^*$\\
    \hline\hline
        \prob{ CELAR6-SUB0 }{16}{44}{207}{0.47}{7}  & 0.116 (1.26x)  & 0.116 (13.5x)  & 0.31 (182x)  & 0.31 (182x)  & oom \\
         &  \multicolumn{5}{|c||}{}  & 0  & 13  & 13  & 13   &  (159) \\
        \hline
        \prob{ CELAR7-SUB1-20 }{14}{20}{300}{0.98}{9}  & 0.111 (3.37x)  & 0.129 (42.8x)  & 0.451 (215x)  & 5.1 (313x)  & oom \\
         &  \multicolumn{5}{|c||}{}  & 20102  & 40931  & 71023  & 81433  & (132538) \\
        \hline
        \prob{ CELAR6-SUB1-24 }{14}{24}{300}{0.82}{9}  & 0.188 (0.59x)  & 0.122 (4.81x)  & 0.161 (71.8x)  & 0.837 (279x)  & oom \\
         &  \multicolumn{5}{|c||}{}  & 0  & 280  & 598  & 772  &  (2656) \\
        \hline
        \prob{ CELAR7-SUB0 }{16}{44}{188	}{0.66}{9}  & 0.173 (0.63x)  & 0.157 (10.1x)  & 0.317 (187x)  & 0.317 (187x)  & oom \\
         &  \multicolumn{5}{|c||}{}  & 0  & 104  & 10001  & 10001  &  (10310)\\
        \hline
        \prob{ CELAR6-SUB1 }{14}{44}{300	}{0.82}{9} & 0.201 (1.25x)  & 0.317 (10.6x)  & 0.723 (171x)  & 0.723 (171x)  & oom \\
         &  \multicolumn{5}{|c||}{}  & 0  & 308  & 626  &  626 &  (2669) \\
        \hline
        \prob{ CELAR7-SUB1 }{14}{44}{300}{0.82}{9}  & 0.295 (1.03x)  & 0.316 (10.5x)  & 0.593 (210x)  & 0.593 (210x)   & oom \\
         &  \multicolumn{5}{|c||}{}  & 0  & 21523  & 51123  &  51123 &  (142640)\\
        \hline
        \prob{ CELAR6-SUB2 }{16}{44}{353}{0.77}{10}  & 0.254 (1.3x)  & 0.328 (12.6x)  & 0.82 (188x)  & 0.82 (188x)  & oom  \\
         &  \multicolumn{5}{|c||}{}  & 0  & 231  & 387  & 387  &  (2746)\\
        \hline
        \prob{ CELAR7-SUB2 }{16}{44}{353}{0.77}{10} & 0.407 (0.54x)  & 0.254 (16.4x)  & 0.95 (164x)  & 0.95 (164x)  & oom  \\
         &  \multicolumn{5}{|c||}{}  & 0  & 20420  & 40931  & 40931  & (173252) \\
        \hline
        \prob{ CELAR6-SUB3 }{18}{44}{421}{0.71}{10}  & 0.375 (1.08x)  & 0.304 (16.1x)  & 0.975 (188x)  & 0.975 (188x)   & oom \\
         &  \multicolumn{5}{|c||}{}  & 0  & 265  & 452  &  452 &  (3079)\\
        \hline
        \prob{ CELAR7-SUB3 }{18}{44}{421}{0.71}{10}  & 0.449 (0.97x)  & 0.46 (10.8x)  & 1.002 (178x)  & 1.002 (178x)  & oom \\
         &  \multicolumn{5}{|c||}{}  & 0  & 20615  & 51434  & 51434  & (203460) \\
        \hline
        \prob{ scen06-30reduc }{81}{14}{399}{0.11}{10} & 0.071 (2.65x)  & 0.068 (10.3x)  & 0.325 (155x)  & 1.738 (278x)  & oom \\
         &  \multicolumn{5}{|c||}{}  & 285  & 690  & 975  & 1447  &  (2080)\\
        \hline
	\prob{ scen06-30 }{99}{14}{1178}{0.09}{10}  & 0.19 (1.5x)  & 0.189 (14.2x)  & 1.313 (238x)  & 13.3 (345x)  & oom  \\
	 &  \multicolumn{5}{|c||}{}  & 450  & 411  & 1201  & 1100  & (2080) \\
        \hline
		\prob{ CELAR6-SUB4-20 }{22}{20}{477}{0.82}{11}  & 0.197 (4.09x)  & 0.232 (49.6x)  & 1.021 (225x)  & 11.29 (344x)  & oom \\
		 &  \multicolumn{5}{|c||}{}  & 494  & 598  & 732  & 1359  & (2716) \\
        \hline
        \prob{ CELAR7-SUB4-22 }{22}{22}{473}{0.67}{11}  & 0.221 (0.69x)  & 0.211 (3.54x)  & 0.158 (82.4x)  & 0.922 (286x)  & oom \\
         &  \multicolumn{5}{|c||}{}  & 0  & 40104  & 60214  & 31530  & (202342) \\
        \hline
        \prob{ CELAR6-SUB4reduc }{20}{44}{149}{0.77}{11}  & 0.106 (2.01x)  & 0.213 (10.9x)  & 0.357 (241x)  & 0.357 (241x)   & oom  \\
         &  \multicolumn{5}{|c||}{}  & 0  & 44  & 283  & 283  & (202342) \\
        \hline
        \prob{ CELAR6-SUB4 }{22}{44}{477	}{0.65}{1} & 0.387 (0.73x)  & 0.343 (17.4x)  & 1.013 (229x)  & 1.013 (229x)  & oom \\
         &  \multicolumn{5}{|c||}{}  & 0  & 170  & 405  &405  &  (3230) \\
        \hline
        \prob{ CELAR7-SUB4 }{22}{44}	{477}{0.65}{1}  & 0.347 (0.82x)  & 0.344 (18.1x)  & 1.24 (188x)  & 1.24 (188x)  & oom \\
         &  \multicolumn{5}{|c||}{}  & 0  & 30118  & 31442  &  31442 & (242443) \\
        \hline
	\prob{ scen06-24reduc }{81}{20}{403}{0.12}{12}  & 0.099 (4.76x)  & 0.101 (57.9x)  & 0.375 (217x)  & 4.001 (303x)  & oom \\
	 &  \multicolumn{5}{|c||}{}  & 278  & 599  & 634  & 1411  &  (2857)\\
        \hline
        \prob{ scen06-22reduc }{81}{22}{404	}{0.12}{12}  & 0.164 (0.68x)  & 0.091 (5.87x)  & 0.122 (67.1x)  & 0.52 (243x)  & oom \\
         &  \multicolumn{5}{|c||}{}  & 0  & 453  & 717  & 793  &  (3159)\\
        \hline
        \prob{ scen06-20reduc }{82}{24}{409}{0.12}{12}  & 0.203 (0.68x)  & 0.095 (7.53x)  & 0.142 (86.1x)  & 0.838 (277x)  & oom \\
         &  \multicolumn{5}{|c||}{}  & 0  & 447  & 717  & 794  &  (3163)\\
        \hline
        \prob{ scen06-18reduc }{82}{26}{409}{0.12}{12}  & 0.221 (0.76x)  & 0.194 (4.7x)  & 0.303 (56x)  & 1.189 (292x)  & oom  \\
         &  \multicolumn{5}{|c||}{}  & 0  & 458  & 718  & 796  &  (3263)\\
        \hline
        \prob{ scen06-24 }{99}{20}{1203}{0.10}{12} & 0.25 (0.52x)  & 0.236 (3.97x)  & 0.278 (47.1x)  & 0.867 (233x)  & oom \\
         &  \multicolumn{5}{|c||}{}  & 0  & 437  & 319  & 900  &  \\
        \hline
        \prob{ scen06-16reduc }{82}{28}{409}{0.12}{12}  & 0.22 (0.45x)  & 0.113 (10.4x)  & 0.235 (101x)  & 1.695 (304x)  & oom \\
         &  \multicolumn{5}{|c||}{}  & 0  & 458  & 717  & 812  &  \\
        \hline
        \prob{ scen06-22 }{99}{22}{1210}{0.10}{12}  & 0.271 (0.58x)  & 0.26 (4.93x)  & 0.358 (56.3x)  & 1.415 (256x)  & oom \\
         &  \multicolumn{5}{|c||}{}  & 0  & 437  & 403  & 803  &  \\
        \hline
        \prob{ scen06-20 }{100}{24}{1215}{0.10}{12} & 0.306 (1.2x)  & 0.263 (6.37x)  & 0.371 (78.2x)  & 1.979 (291x)  & oom \\
         &  \multicolumn{5}{|c||}{}  & 0  & 437  & 352  & 804  &  \\
        \hline
        \prob{ scen06-18 }{100}{26}{1221}{0.10}{12}  & 0.352 (0.83x)  & 0.299 (7.54x)  & 0.457 (94.4x)  & 2.995 (303x)  & oom \\
         &  \multicolumn{5}{|c||}{}  & 0  & 437  & 327  & 813  &  \\
        \hline
        \prob{ scen06-16 }{100}{28}{1222}{0.1}{12} & 0.36 (1.34x)  & 0.389 (7.12x)  & 0.537 (122x)  & 4.382 (317x)  & oom  \\
         &  \multicolumn{5}{|c||}{}  & 0  & 437  & 328  & 813  &  \\
        \hline
        \prob{ scen06reduc }{82}{44}{409}{0.12}{14}  & 0.343 (0.68x)  & 0.306 (15.1x)  & 0.787 (204x)  & 0.787 (204x)  & oom  \\
         &  \multicolumn{5}{|c||}{}  & 0  & 137  & 318  & 318  &  \\
        \hline
        \hline
     \end{tabular}
     }
     {\small
    \caption{Celar Benchmark: Runtime (in seconds) of GpuBE, at varying of the bucket size $z$ and GpuBE($w^*$) (top), 
    and solution quality (bottom).  The speedup of GpuBE($z$) and GpuBE($w^*$) w.r.t. their CPU counterparts are shown in parenthesis.}
    \label{tab:celar}}
  \end{center}
\end{table}

\begin{table}
  \begin{center}
  \resizebox{0.9\textwidth}{!}
  {
        \begin{tabular}{| l | p{6pt} p{2pt} p{12pt} p{6pt} p{6pt} || *{5}{r} |}
    \hline
	\prob{Problem}{$n$}{$d$}{$c$}{$p_1$}{$w^*$} & \multicolumn{4}{|c}{GpuBE($z$)} & GpuBE \\
	\prob{}{}{}{}{}{} &  $z_{\min}$ & $z_2$ & $z_3$ &  $z_{\max}$ & $w^*$ \\
    \hline\hline
            \prob{ GEOM40-2 }{40}{2}{78}{0.12}{5}  & 0.004 (0.25x)  & 0.004 (0.25x)  & 0.004 (0.25x)  & 0.004 (0.25x)  & 0.004 (0.25x)\\
             &  \multicolumn{5}{|c||}{}  & 22  &  22 & 22  & 22  & 22 \\
            \hline
            \prob{ GEOM40-3 }{40}{3}{78}{0.12}{5}  & 0.009 (0.44x)  & 0.009 (0.44x)  & 0.009 (0.44x)  & 0.009 (0.44x)  & 0.009 (0.44x)\\
             &  \multicolumn{5}{|c||}{}  &  7 & 7  & 7  & 7  & 7 \\
            \hline
            \prob{ GEOM40-4 }{40}{4}{78}{0.12}{5} & 0.004 (4.25x)  & 0.004 (4.25x)  & 0.004 (4.25x)  & 0.004 (4.25x)  & 0.004 (4.25x) \\
             &  \multicolumn{5}{|c||}{}  &  3 & 3  & 3  & 3  & 3 \\
            \hline
            \prob{ GEOM40-5 }{40}{5}{78}{0.12}{5}  & 0.005 (10.8x)  & 0.005 (10.8x)  & 0.005 (10.8x)  & 0.005 (10.8x)  & 0.005 (10.8x)\\
             &  \multicolumn{5}{|c||}{}  &  1 &  1 &  1 & 1  & 1 \\
            \hline
            \prob{ GEOM40-6 }{40}{6}{78}{0.12}{5}  & 0.011 (5.73x)  & 0.011 (5.73x)  & 0.011 (5.73x)  & 0.011 (5.73x)  & 0.011 (5.73x) \\
             &  \multicolumn{5}{|c||}{}  & 0  & 0  &  0 & 0  & 0 \\
            \hline
            \prob{ GEOM30a-3 }{30}{3}{81}{0.2}{6}  & 0.014 (0.21x)  & 0.024 (0.5x)  & 0.011 (1.73x)  & 0.024 (0.79x)  & 0.003 (2.67x)\\
             &  \multicolumn{5}{|c||}{}  & 0  & 10  & 11  & 11  & 11 \\
            \hline
            \prob{ GEOM30a-4 }{30}{4}{81}{0.2}{6}  & 0.028 (0.14x)  & 0.024 (2.25x)  & 0.012 (10.3x)  & 0.012 (10.2x)  & 0.003 (14.7x) \\
             &  \multicolumn{5}{|c||}{}  & 0  & 4  & 4  & 4  & 4  \\
            \hline
            \prob{ GEOM30a-5 }{30}{5}{81}{0.2}{6} & 0.029 (0.21x)  & 0.012 (16x)  & 0.012 (21.3x)  & 0.012 (20.2x)  & 0.004 (21.8x)\\
             &  \multicolumn{5}{|c||}{}  & 0  & 1  & 1  & 1  & 1 \\
            \hline
            \prob{ GEOM30a-6 }{30}{6}{81}{0.2}{6} & 0.03 (0.17x)  & 0.013 (35.3x)  & 0.028 (31.6x)  & 0.015 (68.8x)  & 0.006 (50.3x) \\
             &  \multicolumn{5}{|c||}{}  & 0  & 0  & 0  & 0  & 0 \\
            \hline
            \prob{ queen5-5-3 }{25}{3}{160}{0.87}{18}  & 0.031 (0.19x)  & 0.034 (48.2x)  & 0.471 (203x)  & 2.899 (250x)  & oom \\
             &  \multicolumn{5}{|c||}{}  & 6  & 18  & 23  & 25  &  (29) \\
             \hline
            \prob{ queen5-5-4 }{25}{4}{160}{0.87}{18}  & 0.031 (1x)  & 0.038 (53.3x)  & 0.158 (169x)  & 1.355 (238x)  & oom \\
             &  \multicolumn{5}{|c||}{}  & 0  & 1  & 4  & 5  &  (12) \\
             \hline
            \prob{ queen5-5-5 }{25}{5}{160}{0.87}{18}  & 0.031 (2.19x)  & 0.121 (97x)  & 1.031 (247x)  & 4.659 (267x)  & oom  \\
             &  \multicolumn{5}{|c||}{}  & 0  & 0  & 0  & 0  &  (0) \\
             \hline
            \prob{ myciel5g-3 }{47}{3}{236}{0.44}{20}  & 0.033 (2.58x)  & 0.069 (27.2x)  & 1.999 (201x)  & 11.97 (308x)  & oom \\
             &  \multicolumn{5}{|c||}{}  & 0  & 3  & 9  & 12  &  (16) \\
             \hline
            \prob{ myciel5g-4 }{47}{4}{236}{0.44}{20}  & 0.031 (1.03x)  & 0.045 (52.6x)  & 0.23 (143x)  & 7.852 (293x)  & oom \\
             &  \multicolumn{5}{|c||}{}  & 0  & 0  & 0  & 0  &  (4) \\
             \hline
            \prob{ myciel5g-5 }{47}{5}{236}{0.44}{20}  & 0.072 (0.92x)  & 0.051 (62.7x)  & 0.41 (160x)  & 6.513 (278x)  & oom  \\
             &  \multicolumn{5}{|c||}{}  & 0  & 0  & 0  & 0  &  (1) \\
             \hline
            \prob{ myciel5g-6 }{47}{6}{236}{0.44}{20}  & 0.035 (3.57x)  & 0.069 (5.39x)  & 0.123 (91x)  & 1.636 (220x)  & oom \\
             &  \multicolumn{5}{|c||}{}  & 0  & 0  & 0  & 0  &  (0) \\
             \hline
            \prob{ DSJC125.1.4 }{125}{4}{736}{0.72}{72}  & 0.241 (0.62x)  & 0.168 (19.4x)  & 0.477 (96.5x)  & 3.551 (197x)  & oom  \\
             &  \multicolumn{5}{|c||}{}  & 0  & 0  & 0  & 0  &  -- \\
             \hline
            \prob{ DSJC125.1.5 }{125}{5}{736}{0.72}{72}  & 0.285 (1.1x)  & 0.17 (4.43x)  & 0.393 (38.8x)  & 1.801 (191x)  & oom \\
             &  \multicolumn{5}{|c||}{}  & 0  & 0  & 0  & 0  &  (0) \\
             \hline
            \prob{ le450-5a-2 }{450}{2}{5714}{0.81}{344} & 1.725 (0.18x)  & 1.611 (1.46x)  & 2.934 (40.7x)  & 10.21 (67.3x)  & oom \\
             &  \multicolumn{5}{|c||}{}  & 618  & 734  & 833  & 878  &  -- \\
             \hline
            \prob{ le450-5a-3 }{450}{3}{5714}{0.81}{344} & 1.728 (0.19x)  & 1.847 (2.39x)  & 2.042 (17.6x)  & 7.207 (44.4x)  & oom \\
             &  \multicolumn{5}{|c||}{}  & 42  & 55  & 57  & 58  & --\\
              \hline
		\prob{ le450-5a-4 }{450}{4}{5714}{0.81}{344}  & 1.422 (0.44x)  & 1.549 (1.29x)  & 2.088 (13.7x)  & 6.517 (68.6x)  & oom \\
		 &  \multicolumn{5}{|c||}{}  & 0  & 0  & 4  & 1  &  -- \\
             \hline
            \prob{ le450-5a-5 }{450}{5}{5714}{0.81}{344}  & 1.67 (0.98x)  & 1.844 (3.7x)  & 3.505 (39.3x)  & 10.23 (65.1x)  & oom \\
             &  \multicolumn{5}{|c||}{}  & 0  & 0  & 0  & 0  &  -- \\
             \hline
    \end{tabular}
     }
     {\small
    \caption{Coloring Benchmark: Runtime (in seconds) of GpuBE, at varying of the bucket size $z$ and GpuBE($w^*$) (top), 
    and solution quality (bottom).  The speedup of GpuBE($z$) and GpuBE($w^*$) w.r.t. their CPU counterparts are shown in parenthesis.}
    \label{tab:coloring}}	
  \end{center}
\end{table}

\begin{table}
  \begin{center}
  \resizebox{0.9\textwidth}{!}
  {
        \begin{tabular}{| l | p{6pt} p{2pt} p{12pt} p{6pt} p{6pt} || *{5}{r} | }
    \hline
	\prob{Problem}{$n$}{$d$}{$c$}{$p_1$}{$w^*$} & \multicolumn{4}{|c}{GpuBE($z$)} & GpuBE \\
	\prob{}{}{}{}{}{} &  $z_{\min}$ & $z_2$ & $z_3$ &  $z_{\max}$ & $w^*$ \\
    \hline\hline
            \prob{ s386 }{172}{2}{172}{0.04}{19}  & 0.054 (0.28x)  & 0.051 (2.55x)  & 0.053 (16.1x)  & 0.185 (71.9x)  & 0.129 (82.8x) \\
             &  \multicolumn{5}{|c||}{}  & 29  & 29  & 29  & 29  & 29 \\
            \hline
            \prob{ s1423 }{748}{2}{748}{{0.06}}{38}  & 0.184 (0.12x)  & 0.182 (1.08x)  & 0.189 (6.05x)  & 0.546 (57.8x)  & oom  \\
             &  \multicolumn{5}{|c||}{}  & 231  & 231  & 231  & 231  &   (231)\\
            \hline
            \prob{ c499 }{499}{2}{499}{0.01}{42}  & 0.133 (0.89x)  & 0.131 (1.08x)  & 0.141 (9.72x)  & 0.391 (78.1x)  & oom \\
             &  \multicolumn{5}{|c||}{}  & 111  & 111  & 111  & 111  &  (111)\\
            \hline
            \prob{ c432 }{432}{2}{432}{0.01}{54}  & 0.225 (0.33x)  & 0.237 (0.72x)  & 0.270 (7.93x)  & 0.622 (69.6x)  & oom  \\
             &  \multicolumn{5}{|c||}{}  & 101  & 101  & 101  & 101  &  --\\
            \hline
            \prob{ s1494 }{661}{2}{661}{0.01}{57}  & 0.238 (0.19x)  & 0.222 (1.14x)  & 0.249 (17.1x)  & 1.285 (101x)  & oom  \\
             &  \multicolumn{5}{|c||}{}  & 32  & 32  & 32  & 32  &  (32)\\
            \hline
            \prob{ s1488 }{667}{2}{667}{0.01}{62}  & 0.232 (0.13x)  & 0.219 (1.07x)  & 0.244 (16.7x)  & 1.268 (98.5x)  & oom  \\
             &  \multicolumn{5}{|c||}{}  & 32  & 32  & 32  & 32  &  (32)\\
            \hline
            \prob{ c880 }{880}{2}{880}{0.04}{68}  & 0.245 (0.11x)  & 0.241 (0.69x)  & 0.295 (8.86x)  & 1.1 (91.7x)  & oom  \\
             &  \multicolumn{5}{|c||}{}  & 162  & 162  & 162  & 162  &  --\\
            \hline
            \prob{ s1196 }{561}{2}{561}{0.01}{92}  & 0.19 (0.27x)  & 0.185 (0.93x)  & 0.305 (10.6x)  & 1.049 (101x)  & oom \\
             &  \multicolumn{5}{|c||}{}  & 95  & 95  & 95  & 95  &  (95)\\
            \hline
            \prob{ s953 }{440}{2}{440}{0.01}{93}  & 0.234 (0.13x)  & 0.139 (1.62x)  & 0.262 (11.2x)  & 0.781 (98.8x)  & oom \\
             &  \multicolumn{5}{|c||}{}  & 124  & 124  & 124  & 124  &  (124)\\
            \hline
            \prob{ s1238 }{540}{2}{540}{0.01}{95}  & 0.195 (0.28x)  & 0.184 (1.04x)  & 0.21 (16.8x)  & 0.964 (94.6x)  & oom \\
             &  \multicolumn{5}{|c||}{}  & 95  & 95  & 95  & 95  &  (95)\\
            \hline
    \end{tabular}
     }
	     {\small
    \caption{Iscas-89 Benchmark: Runtime (in seconds) of GpuBE, at varying of the bucket size $z$ and GpuBE($w^*$) (top), 
    and solution quality (bottom).  The speedup of GpuBE($z$) and GpuBE($w^*$) w.r.t. their CPU counterparts are shown in parenthesis.}
    \label{tab:iscas}}
  \end{center}
\end{table}

\begin{table}
  \begin{center}
  \resizebox{0.9\textwidth}{!}
  {
        \begin{tabular}{| l | p{6pt} p{2pt} p{12pt} p{6pt} p{8pt} || *{5}{r} |}
    \hline
	\prob{Problem}{$n$}{$d$}{$c$}{$p_1$}{$w^*$} & \multicolumn{4}{|c}{GpuBE($z$)} & GpuBE \\
	\prob{}{}{}{}{}{} &  $z_{\min}$ & $z_2$ & $z_3$ &  $z_{\max}$ & $w^*$ \\
    \hline\hline
                \prob{ eye }{36}{21}{53}{0.09}{2}  & 0.048 (3.83x)  & 0.045 (3.67x)  & 0.045 (2.2x)  & 0.045 (4.13x)  & 0.006 (12.7x)  \\
                 &  \multicolumn{5}{|c||}{}  & 1  & 1  & 1  & 1  & 1  \\
                \hline
                \prob{ wijsmanguo }{49}{36}{68}{0.06}{3} & 0.278 (2.15x)  & 0.279 (2.18x)  & 0.278 (2.21x)  & 0.278 (2.15x)  & 0.012 (18.8x) \\
                 &  \multicolumn{5}{|c||}{}  & 1  & 1  & 1  & 1  & 1  \\
                \hline
                \prob{ cancer }{49}{36}{68}{0.06}{3}  & 0.304 (2.34x)  & 0.28 (2.91x)  & 0.291 (2.44x)  & 0.278 (2.55x)  & 0.012 (18.8x) \\
                 &  \multicolumn{5}{|c||}{}  & 1  & 1  & 1  & 1  & 1  \\
                \hline
                \prob{ sobel }{7}{6}{8}{0.61}{3}  & 0.002 (0.5x)  & 0.002 (1x)  & 0.001 (2x)  & 0.001 (2x)  & 0.001 (1x) \\
                 &  \multicolumn{5}{|c||}{}  & 0  & 0  & 0  & 0  & 0  \\
                \hline
                \prob{ connell }{12}{6}{15}{0.39}{3}  & 0.004 (0.25x)  & 0.004 (1.5x)  & 0.002 (2.5x)  & 0.001 (6x)  & 0.001 (4x)  \\
                 &  \multicolumn{5}{|c||}{}  & 0  & 1  & 1  & 1  & 1  \\
                \hline
                \prob{ pedck60-L2 }{60}{10}{106}{0.08}{5}  & 0.017 (1.71x)  & 0.052 (78x)  & 0.052 (77x)  & 0.051 (82.6x)  & 0.02 (144x)  \\
                 &  \multicolumn{5}{|c||}{}  & 2  & 2  & 2  & 2  & 2  \\
                \hline
                \prob{ pedck60-L1 }{60}{10}{108}{0.08}{5}  & 0.016 (1.56x)  & 0.033 (125x)  & 0.044 (94.1x)  & 0.033 (124x)  & 0.02 (137x) \\
                 &  \multicolumn{5}{|c||}{}  & 2  & 2  & 2  & 2  & 2  \\
                \hline
                \prob{ pedck60-L12 }{60}{10}{108}{0.08}{5}  & 0.023 (1.26x)  & 0.033 (124x)  & 0.033 (127x)  & 0.033 (126x)  & 0.02 (143x) \\
                 &  \multicolumn{5}{|c||}{}  & 6  & 6  & 6  & 6  & 6  \\
                \hline
                \prob{ saudiarabia }{37}{15}{43}{0.16}{5}  & 0.015 (4.13x)  & 0.267 (231x)  & 0.286 (211x)  & 0.282 (210x)  & 0.187 (212x) \\
                 &  \multicolumn{5}{|c||}{}  & 0  & 0  & 0  & 0  & 0  \\
                \hline
                \prob{ parkinson }{37}{15}{43}{0.16}{5}  & 0.015 (3.4x)  & 0.292 (206x)  & 0.265 (223x)  & 0.274 (223x)  & 0.186 (237x)  \\
                 &  \multicolumn{5}{|c||}{}  & 0  & 0  & 0  & 0  & 0  \\
                \hline
                \prob{ pedck350l3 }{350}{10}{578}{0.03}{24}  & 0.092 (2.13x)  & 0.094 (7.23x)  & 0.109 (35.7x)  & 0.192 (112x)  & oom  \\
                 &  \multicolumn{5}{|c||}{}  & 0  & 0  & 0  & 0  &  (0)\\
                \hline
                \prob{ pedck350l2 }{350}{10}{578}{0.03}{24}  & 0.091 (1.02x)  & 0.131 (4.21x)  & 0.109 (32.9x)  & 0.235 (81.1x)  & oom \\
                 &  \multicolumn{5}{|c||}{}  & 0  & 0  & 1  & 1  & (1)\\
                \hline
                \prob{ pedck350 }{350}{10}{580}{0.03}{26}  & 0.099 (1.04x)  & 0.139 (3.89x)  & 0.117 (27.9x)  & 0.252 (114x)  & oom  \\
                 &  \multicolumn{5}{|c||}{}  & 0  & 0  & 2  & 2  & (2)\\
                \hline
                \prob{ sheep4r-4-3 }{2662}{10}{5021}{0.00}{38}  & 1.16 (0.98x)  & 1.173 (5.83x)  & 1.551 (29.9x)  & 3.139 (110x)  & oom \\
                 &  \multicolumn{5}{|c||}{}  & 0  & 0  & 0  & 1  &  (1)\\
                \hline
                \prob{ sheep4r-4-2 }{2172}{10}{4026}{0.00}{46}  & 0.874 (0.96x)  & 1.3 (4.11x)  & 1.145 (29.9x)  & 2.339 (110x)  & oom \\
                 &  \multicolumn{5}{|c||}{}  & 1  & 0  & 0  & 0  &  (1)\\
                \hline
                \prob{ sheep4r-4-1 }{641}{10}{1196}{0.05}{63}  & 0.273 (0.96x)  & 0.291 (5.92x)  & 0.36 (35.7x)  & 0.782 (120x)  & oom \\
                 &  \multicolumn{5}{|c||}{}  & 0  & 0  & 0  & 0  & (0)\\
                \hline
                \prob{ sheep4r-4-0 }{1541}{10}{2941}{0.03}{108}  & 0.76 (0.95x)  & 0.786 (5.93x)  & 1.066 (31.3x)  & 1.066 (31.3x)  & oom \\
                 &  \multicolumn{5}{|c||}{}  & 0  & 0  & 0  & 0  & (0)\\
                \hline
                \prob{ pedck1000 }{928}{6}{1736}{0.09}{126}  & 0.368 (0.34x)  & 0.393 (1.32x)  & 0.399 (4.46x)  & 0.447 (20.1x)  & oom  \\
                 &  \multicolumn{5}{|c||}{}  & 15  & 15  & 12  & 16  &  (19)\\
                 \hline
                \end{tabular}
     }
	     {\small
    \caption{Pedigree Benchmark: Runtime (in seconds) of GpuBE, at varying of the bucket size $z$ and GpuBE($w^*$) (top), 
    and solution quality (bottom).  The speedup of GpuBE($z$) and GpuBE($w^*$) w.r.t. their CPU counterparts are shown in parenthesis.}
    \label{tab:pedigree}}
  \end{center}
\end{table}

\begin{table}
  \begin{center}
  \resizebox{0.9\textwidth}{!}
  {
        \begin{tabular}{| l | p{6pt} p{2pt} p{12pt} p{6pt} p{8pt} || *{5}{r} |}
    \hline
	\prob{Problem}{$n$}{$d$}{$c$}{$p_1$}{$w^*$} & \multicolumn{4}{|c}{GpuBE($z$)} & GpuBE \\
	\prob{}{}{}{}{}{} &  $z_{\min}$ & $z_2$ & $z_3$ &  $z_{\max}$ & $w^*$ \\
    \hline\hline
    \prob{ 8 }{8}{4}{15}{0.28}{2}  & 0.003 (0.33x)  & 0.001 (1.0x)  & 0.002 (1.5x)  & 0.003 (1.33x)  & 0.003 (2.0x)\\
    &  \multicolumn{5}{|c||}{}  & 2  & 2  & 2  & 2  & 2  \\
    \hline
    \prob{ 1502 }{209}{4}{411}{0.01}{5}  & 0.082 (0.32x)  & 0.039 (1.18x)  & 0.08 (0.57x)  & 0.079 (0.61x)  & 0.029 (0.66x)\\
    &  \multicolumn{5}{|c||}{}  & 28042  & 28042  & 28042  & 28042  & 28042  \\
    \hline
    \prob{ 54 }{67}{4}{271}{0.15}{11}  & 0.037 (0.92x)  & 0.036 (1.86x)  & 0.036 (7.31x)  & 0.045 (53.7x)  & 0.027 (46.1x) \\
    &  \multicolumn{5}{|c||}{}  & 31  & 32  & 35  & 37  & 37  \\
    \hline
    \prob{ 503 }{143}{4}{635}{0.07}{12}  & 0.088 (0.74x)  & 0.085 (2.88x)  & 0.088 (18.3x)  & 0.203 (31.9x)  & 0.044 (59.9x) \\
    &  \multicolumn{5}{|c||}{}  & 7093  & 9106  & 11111  & 11113  & 11113  \\
    \hline
    \prob{ 29 }{82}{4}{462}{0.17}{14}  & 0.085 (0.11x)  & 0.07 (3.71x)  & 0.147 (97.7x)  & 7.281 (236x)  & 3.327 (141x) \\
    &  \multicolumn{5}{|c||}{}  & 7035  & 8048  & 8055  & 8059  & 8059 \\
    \hline
    \prob{ 404 }{100}{4}{710}{0.16}{19} & 0.112 (0.6x)  & 0.207 (5.95x)  & 0.188 (113x)  & 0.701 (224x)  & 0.301 (124x) \\
    &  \multicolumn{5}{|c||}{}  & 76  & 106  & 112  & 114  & 114  \\
    \hline
    \prob{ 42b }{190}{4}{1140}{0.09}{19}  & 0.455 (0.13x)  & 0.358 (0.65x)  & 0.276 (43.6x)  & 12.98 (239x)  & oom  \\
    &  \multicolumn{5}{|c||}{}  & 58049  & 95050  & 135050  & 155050  & --\\
    \hline
    \prob{ 505b }{240}{4}{1716}{0.06}{23}  & 0.392 (0.12x)  & 0.484 (0.61x)  & 0.361 (48.6x)  & 7.531 (238x)  & oom \\
    &  \multicolumn{5}{|c||}{}  & 6106  & 11159  & 15204  & 19244  &  --\\
    \hline
    \prob{ 1504 }{605}{4}{4187}{0.03}{25}  & 1.047 (0.18x)  & 0.884 (2.41x)  & 1.221 (55.3x)  & 12.09 (263x)  & oom  \\
    &  \multicolumn{5}{|c||}{}  & 80104  & 107175  & 131227  & 141251  &  --\\
    \hline
    \prob{ 408b }{200}{4}{1843}{0.10}{29}  & 0.552 (0.14x)  & 0.309 (2.77x)  & 0.862 (132x)  & 5.774 (218x)  & oom \\
    &  \multicolumn{5}{|c||}{}  & 2104  & 5162  & 6206  & 6215  &  --\\
    \hline
    \prob{ 42 }{190}{4}{1394}{0.11}{30}  & 0.295 (0.11x)  & 0.251 (7.53x)  & 0.59 (56.9x)  & 9.419 (241x)  & oom  \\
    &  \multicolumn{5}{|c||}{}  & 60049  & 97050  & 105050  & 133050  &  --\\
    \hline
    \prob{ 505 }{240}{4}{2242}{0.11}{30}  & 0.61 (0.09x)  & 0.547 (1.05x)  & 0.712 (83.3x)  & 8.757 (220x)  & oom \\
    &  \multicolumn{5}{|c||}{}  & 5103  & 7149  & 10178  & 15211  &  --\\
    \hline
    \prob{ 408 }{200}{4}{2232}{0.17}{40}  & 0.637 (0.12x)  & 0.506 (2.14x)  & 0.816 (96x)  & 10.2 (265x)  & oom \\
    &  \multicolumn{5}{|c||}{}  & 2100  & 3123  & 4169  & 5185  &  --\\
        \hline
        \prob{ 5 }{309}{4}{5621}{0.19}{44}  & 1.45 (0.14x)  & 1.18 (1.09x)  & 1.879 (37x)  & 6.573 (158x)  & oom  \\
        &  \multicolumn{5}{|c||}{}  & 42  & 53  & 86  & 106  &  --\\
        \hline
        \prob{ 412 }{300}{4}{4348}{0.16}{48}  & 1.366 (0.08x)  & 1.097 (2.05x)  & 1.228 (41.3x)  & 7.048 (183x) & oom \\
        &  \multicolumn{5}{|c||}{}  & 2106  & 2131  & 8176  & 11258  &  --\\
        \hline
        \prob{ 507 }{311}{4}{5732}{0.18}{68}  & 1.788 (0.13x)  & 1.379 (1.59x)  & 1.744 (38.2x)  & 5.343 (140x) & oom \\
        &  \multicolumn{5}{|c||}{}  & 6114  & 5140  & 7194  & 10226  &  --\\
        \hline
        \prob{ 1506 }{940}{4}{15240}{0.05}{77}  & 4.82 (0.13x)  & 3.805 (0.82x)  & 3.929 (8.52x)  & 16.93 (128x) & oom \\
        &  \multicolumn{5}{|c||}{}  & 50122  & 68152  & 86172  & 88235  &  --\\
        \hline
        \prob{ 28 }{230}{4}{5226}{0.42}{87}  & 1.397 (0.12x)  & 1.173 (1.89x)  & 1.946 (50x)  & 8.584 (226x)  & oom  \\
        &  \multicolumn{5}{|c||}{}  & 43075  & 52105  & 76105  & 88105  &  --\\
        \hline
        \prob{ 509 }{348}{4}{8624}{0.22}{92}  & 2.82 (0.16x)  & 2.204 (5.56x)  & 5.664 (75.6x)  & 10.61 (278x)  & oom \\
        &  \multicolumn{5}{|c||}{}  & 3105  & 3217  & 5191  & 6157  &  --\\
        \hline
    \prob{ 414 }{364}{4}{10108}{0.24}{104}  & 3.61 (0.16x)  & 2.673 (1.95x)  & 3.091 (16.7x)  & 7.608 (80.7x)  & oom  \\
    &  \multicolumn{5}{|c||}{}  & 3113  & 3120  & 4195 & 6141  &  --\\
    \hline
    \prob{ 1401 }{488}{4}{10963}{0.17}{105}  & 3.283 (0.16x)  & 3.047 (1.62x)  & 3.543 (13x)  & 13.65 (157x)  & oom  \\
    &  \multicolumn{5}{|c||}{}  & 64057  & 61066  & 66071  & 87071  & --\\
    \hline
    \prob{ 1403 }{665}{4}{13616}{0.11}{105}  & 4.52 (0.14x)  & 3.829 (1.63x)  & 4.246 (15x)  & 10.11 (69.6x)  & oom \\
    &  \multicolumn{5}{|c||}{}  & 58099  & 71129  & 74118  & 82120  &  --\\
    \hline
    \prob{ 1405 }{855}{4}{18258}{0.09}{105}  & 5.708 (0.14x)  & 4.791 (2.08x)  & 6.787 (17.3x)  & 19.98 (74.5x)  & oom  \\
    &  \multicolumn{5}{|c||}{}  & 58099  & 71129  & 74118  & 84177  &  --\\
    \hline
    \prob{ 1407 }{1057}{4}{21786}{0.07}{105}  & 7.018 (0.15x)  & 6.283 (2.21x)  & 8.931 (19x)  & 18.15 (53.9x)  & oom  \\
    &  \multicolumn{5}{|c||}{}  & 58127  & 74180  &  75164 & 84202  & --\\
    \hline
	\end{tabular}
     }
     \small{
         \caption{Spot Benchmark: Runtime (in seconds) of GpuBE, at varying of the bucket size $z$ and GpuBE($w^*$) (top), 
    and solution quality (bottom).  The speedup of GpuBE($z$) and GpuBE($w^*$) w.r.t. their CPU counterparts are shown in parenthesis.}
    \label{tab:spot}}
  \end{center}
\end{table}

\section{Related Work}
\label{sec:relatedwork}
The use of GPUs to solve difficult combinatorial problems has been explored by several proposals in different areas of constraint optimization.
For instance, Meyer \emph{et al.}~\cite{lalami:11} proposed a multi-GPU implementation of the \emph{simplex tableau} algorithm that relies on a vertical problem decomposition to reduce communication between GPUs.
In constraint programming, Arbelaez and Codognet~\cite{arbelaez:14} proposed a GPU-based version of the \emph{Adaptive Search} algorithm, which explores several \emph{large neighborhoods} in parallel, resulting in a speedup factor of $17$. Campeotto \emph{et al.}~\cite{campeotto:14} proposed a GPU-based framework that exploits both parallel propagation and parallel exploration of several large neighborhoods using local search techniques, leading to a speedup factor of up to $38$.
The combination of GPUs with dynamic programming has also been explored to solve different combinatorial optimization problems. For instance, Boyer \emph{et al.}~\cite{boyer:12} proposed the use of GPUs to compute the classical DP recursion step for the knapsack problem, which led to a speedup factor of $26$.
Paw{\l}owski \emph{et al.}~\cite{pawlowski:14} presented a DP-based solution for the \emph{coalition structure formation problem} on GPUs, reporting up to two orders of magnitude of speedup.
In a recent work, Bistaffa \emph{et al.}~\cite{bistaffa:16} study the parallelization of an inference-based algorithm to solve COPs using GPUs, albeit exclusively in the centralized case. 
Silberstein \emph{et al.}~\cite{Silberstein:08} study a GPU-based kernel for the sum-product operations that arise in \emph{marginalize a product of functions} (MPF) problems. The authors report an average speedup factor of $15$ for random benchmarks and Bayesian networks and higher average speedups (up to two orders of magnitude) for log domains due to the difference in performance of the \emph{log2f} and \emph{exp2f} functions on the CPU and GPU. 

In the distributed constraint optimization context, GPU parallelism has been applied to speed up several DCOP solving techniques. Fioretto \emph{et al.}~\cite{fioretto:16} proposed a multi-variable agent decomposition strategy to solve \emph{general} DCOPs with complex local subproblems, which makes use of GPUs to implement a search-based and a sampling-based algorithm to speed up the agents' local subproblems resolution.
Le \emph{et al.}~\cite{le:16} studied a GPU accelerated algorithm in the context of stochastic DCOPs---DCOPs where the values of the cost tables are stochastic. The authors used SIMT-style parallelism on a DP-based approach, which resulted in a speedup of up to two orders of magnitude. 
Recently, a combination of GPUs with \emph{Markov Chain Monte Carlo} (MCMC) sampling algorithms has been proposed in the context of solving DCOPs~\cite{fioretto:16c}, where the authors adopted GPUs to accelerate the computation of the normalization constants used in the MCMC sampling process as well as to compute several samples in parallel, resulting in a speedup of up to one order of magnitude. 

Finally, a parallelization of the AND/OR Branch-and-Bound search with mini-bucket heuristic has been presented in \cite{lars:17}. 
%

Differently from other proposals, our approach aims at using GPUs to exploit SIMT-style parallelism from DP-based methods to solve general, exact and approximated, WCSPs and DCOPs. 

%

\section{Conclusions and Discussions}
\label{sec:conclusions}

Inference-based algorithms are powerful tools for solving discrete optimization problems. However, their applicability is limited by their high time and space requirements. Motivated by the increasing availability of GPUs, in this paper, we proposed a scheme to speed up the resolution of inference-based methods for centralized and distributed constraint optimization by exploiting SIMT-style parallelism. 
We introduced an exact algorithm and an approximated algorithm that are inspired by BE and MBE for WCSPs \rev{and tasks over belief networks (e.g., MPE)}, and by DPOP and ADPOP for DCOPs. These procedures make use of multiple threads in the GPU cards to parallelize the aggregation and elimination procedures, which are responsible for the high complexity in the inference-based approaches. 
Additionally, we detailed the design of the data structures adopted to process cost functions with GPUs, and of the mapping adopted to associate GPU threads to cost functions' entries, which allows us to efficiently exploit the data parallelism (SIMT) supported by GPUs.

Finally, we reported an extensive experimental evaluation of our inference-based GPU implementations on both centralized and distributed benchmarks. 
We showed that the use of GPUs provides significant advantages in terms of runtime and scalability, achieving speedups of up to two order of magnitude, showing a considerable reduction in runtime (up to 345 times faster) with respect to the serialized version, and that the speedups increase with the induced width of the problem and with the size of the domain of the problem's variables.

The proposed results are significant---the wide availability of GPUs provides access to parallel computing solutions that can be used to improve efficiency of WCSPs and DCOP solvers. 
Furthermore, GPUs are renowned for their complex architectures (multiple memory levels with very different size and speed characteristics; relatively slow cores), which often create challenges to the effective exploitation of parallelism from irregular applications. The strong experimental results indicate that the proposed algorithms are well-suited to GPU architectures.

These results hint that our approach could be exploited in the the context of a dynamic search that makes use of the mini-bucket elimination method as an heuristic to infer bounds on the solution quality, potentially allowing dynamic variable orderings. 
Indeed, the main drawback of this type of search (and various look-ahead methods) is that, since the heuristic is (re)computed in various nodes during search process, the time invested in the heuristic computation may not be cost-effective. 
Leveraging the use of GPUs to infer bounds faster during the dynamic search, 
may therefore produce dramatic speedup to the whole search process. 

%

While this paper describes the applicability of our approach to (M)BE and (A)DPOP,  analogous techniques can be derived and applied to other inference-based approaches to solve discrete optimization problems (e.g.,~to implement the logic of inference-based propagators) and optimization on probabilistic graphical models---(e.g.,~in finding maximum probability explanation (MPE) in belief networks). 
\rev{
Additionally, our work can be extended to solve \emph{sum-product} problems, also known as \emph{weighted counting}, \emph{partition function}, or \emph{probability of evidence}. Due to the difficulty of these problems the value of accelerated versions of the bucket elimination algorithms is especially important. We plan to work in this direction as future work. 
}

We also envision this technology could open the door to efficiently enforcing higher forms of consistencies than domain consistency (e.g.,~\emph{path consistency}~\cite{montanari:74}, \emph{adaptive consistency}~\cite{dechter:88}, or the more recently proposed \emph{branch consistency} for DCOPs \cite{fioretto:14b}), especially when the cost functions need to be represented explicitly. 

\begin{acknowledgements}
This research is partially supported by the National Science Foundation under grants 1345232,  1550662, 1458595, and 1401639. The views and conclusions contained in this document are those of the authors and should not be interpreted as representing the official policies, either expressed or implied, of the sponsoring organizations, agencies, or the U.S. government.
\end{acknowledgements}

\bibliographystyle{spmpsci}      
\bibliography{cp15b}   

%
%

\end{document}